\documentclass[twoside]{article}

\usepackage[preprint]{aistats2027}
\usepackage{amsthm}
\usepackage{enumitem}
\usepackage{amssymb}
\usepackage{bbm}
\usepackage{xspace}
\usepackage{algorithm}
\usepackage{algpseudocode}
\usepackage{xcolor}
\usepackage{hyperref}
\usepackage{aliascnt}
\usepackage{booktabs}
\usepackage{causalmodels}

\definecolor{darkred}{RGB}{203,27,79} 
\definecolor{darkteal}{RGB}{53,159,171} 
\hypersetup{
  colorlinks=true,
  citecolor=darkteal,
  linkcolor=darkred,
  urlcolor=blue
}

\newtheoremstyle{plainbf}
{\topsep}{\topsep}
{\itshape}
{}
{\bfseries}
{.}
{ }
{\thmname{#1}\thmnumber{ #2}\thmnote{ (\normalfont #3)}}

\newtheoremstyle{defbf}
{\topsep}{\topsep}
{\normalfont}
{}
{\bfseries}
{.}
{ }
{\thmname{#1}\thmnumber{ #2}\thmnote{ (\normalfont #3)}}

\newtheoremstyle{remarkit}
{\topsep}{\topsep}
{\normalfont}
{}
{\bfseries}
{.}
{ }
{\thmname{#1}\thmnumber{ #2}\thmnote{\normalfont\space(\textit{#3})}}

\theoremstyle{plainbf}
\newtheorem{theorem}{Theorem}

\newaliascnt{proposition}{theorem}
\newtheorem{proposition}[proposition]{Proposition}
\aliascntresetthe{proposition}

\newaliascnt{lemma}{theorem}
\newtheorem{lemma}[lemma]{Lemma}
\aliascntresetthe{lemma}

\newaliascnt{corollary}{theorem}
\newtheorem{corollary}[corollary]{Corollary}
\aliascntresetthe{corollary}

\newaliascnt{observation}{theorem}

\aliascntresetthe{observation}

\theoremstyle{defbf}
\newaliascnt{definition}{theorem}
\newtheorem{definition}[definition]{Definition}
\aliascntresetthe{definition}

\newaliascnt{assumption}{theorem}
\newtheorem{assumption}[assumption]{Assumption}
\aliascntresetthe{assumption}

\newaliascnt{example}{theorem}
\newtheorem{example}[example]{Example}
\aliascntresetthe{example}

\theoremstyle{remarkit}
\newaliascnt{remark}{theorem}

\aliascntresetthe{remark}

\makeatletter

\def\Hy@str@sub{subsection}
\def\Hy@str@subsub{subsubsection}
\newcommand{\appendixautoref}{
  \xdef\Hy@chapapp{appendix}
  \let\HyOrg@makecurrent\hyper@makecurrent
  \def\hyper@makecurrent##1{
    \edef\Hy@tmp{##1}
    \ifx\Hy@tmp\Hy@str@sub \HyOrg@makecurrent{appsubsection}
    \else\ifx\Hy@tmp\Hy@str@subsub \HyOrg@makecurrent{appsubsubsection}
    \else \HyOrg@makecurrent{##1}\fi\fi}
}
\makeatother

\newcommand{\Pa}{\mathrm{Pa}}
\newcommand{\sk}{\mathrm{sk}}
\newcommand{\mlt}{\mathrm{mlt}}
\newcommand{\supp}{\mathrm{supp}}

\newcommand{\tr}{\mathrm{tr}}
\newcommand{\diag}{\mathrm{diag}}
\newcommand{\BIC}{\mathrm{BIC}}

\newcommand{\PP}{\mathbb{P}}
\newcommand{\R}{\mathbb{R}}

\newcommand{\1}{\mathbbm{1}}

\newcommand{\N}{\mathcal{N}}
\newcommand{\cE}{\mathcal{E}}

\mathcode`G="8000
\mathcode`H="8000

\begingroup
\catcode`G=\active
\catcode`H=\active
\gdef G{\mathcal{\mathchar"7047}}
\gdef H{\mathcal{\mathchar"7048}}
\endgroup
\newcommand{\cO}{\mathcal{O}}

\newcommand{\cI}{\mathcal{I}}

\newcommand{\cD}{\mathcal{D}}
\newcommand{\cL}{\mathcal{L}}
\newcommand{\B}{\mathcal{B}}
\newcommand{\BL}{\mathcal{BL}}

\newcommand{\Tau}{\mathcal{T}}

\newcommand{\bX}{\mathbf{X}}

\newgraphop{\pastar}{pa^{\star}}
\newgraphop{\hatpastar}{\widehat{pa}^{\star}}
\makeatletter
\newcommand{\indepsym}{\mathrel{\causalmodels@indepglyph}}
\makeatother
\newcommand{\coarse}{\textsc{coarse}\xspace}
\newcommand{\repourl}{\url{https://github.com/tedescoG/coarse}}

\usepackage[numbers,sort&compress]{natbib}

\begin{document}

\twocolumn[

\aistatstitle{Score-Based Learning of Cluster DAGs from Interventions}

\aistatsauthor{Gaetano Tedesco \And Alex Markham}

\aistatsaddress{Informatics Institute,\\University of Amsterdam \And  Department of Mathematical Sciences,\\University of Copenhagen } ]

\begin{abstract}
  Graphical approaches to causal abstraction transform a low-level causal directed acyclic graph (DAG) over many measured variables into a smaller, high-level DAG whose nodes cluster the original variables and whose edges summarize the causal relations between clusters.
  Such cluster DAGs are easier to interpret, but learning them requires finding the clusters and recovering the edges between them.
  \citet{madaleno_coarsening_2026-1}
  learn the interventional coarsening---the cluster DAG that merges variables the interventions cannot distinguish---in two constraint-based phases: first the clusters, then the edges.
  We introduce \coarse, the first score-based method for this task: it keeps the two-phase structure but, under linear Gaussian assumptions, swaps the constraint-based edge phase for a score-based one.
  We show that the interventions themselves identify a causal order over the clusters, and learning the edges reduces to a single local search per cluster under a cluster-level BIC score. We prove that the procedure runs in polynomial time and, provided the variables affected by each intervention are correctly identified, that it is consistent. On synthetic and real-world interventional data, \coarse matches state-of-the-art edge recovery given enough samples, with an edge phase up to two orders of magnitude faster, including on dense graphs with hundreds of nodes.
\end{abstract}

\section{Introduction}
\label{sec:introduction}

Recovering the causal structure of a system is a central goal across the empirical sciences, and the large body of work on causal discovery shows how to learn such structure from data at the granularity of the recorded variables \citep{spirtes_algorithm_1991,chickering_optimal_2003,spirtes_introduction_2010}.
That resolution is not always the most useful one.
Fine-grained features such as pixels, neurons, or individual transcripts are not the units in which one wants to reason, intervene, or communicate \citep{chalupka_visual_2014,grosse-wentrup_neuro-cognitive_2024,squires_causal_2022}, and searching over DAGs with thousands of nodes is computationally daunting: the number of DAGs grows super-exponentially in the number of variables, and large-sample learning of Bayesian networks is NP-hard \citep{chickering_large-sample_2012}.
Both concerns call for coarser causal descriptions that remain causally valid while being more interpretable and statistically better behaved.
Causal abstraction formalizes this idea, relating two models of the same system at different granularities so that the coarse model preserves the effects of interventions predicted by the fine one \citep{rubenstein_causal_2017,beckers_abstracting_2019}.
We work with the graphical version of this idea, in which the coarse model is a \emph{cluster DAG}: a DAG whose nodes are blocks of the recorded variables and whose edges summarize the causal relations between blocks.

The literature on causal abstraction has developed along two complementary lines.
The \emph{functional} view encodes an abstraction as a pair of maps, one on variable values and one on interventions, such that abstracting and intervening commute.
\citet{rubenstein_causal_2017} introduced this as exact transformations, \citet{beckers_abstracting_2019} refined it into a hierarchy of abstractions, and later work extended it to the approximate \citep{beckers_approximate_2019} and soft-interventional \citep{massidda_causal_2022} cases.
The \emph{graphical} view operates directly on the causal DAG: cluster DAGs (C-DAGs) \citep{anand_causal_2023} carry the standard graphical calculus over an assumed clustering of the variables, transit clusters \citep{tikka_clustering_2023} collapse nodes while preserving the identifiability of a chosen causal query, and \citet{wahl_foundations_2024} study when Markov properties and faithfulness transfer under such grouping.
Finally, \citet{schooltink_aligning_2024} show that the C-DAGs of the graphical view are exactly the constructive abstractions of the functional hierarchy, so the two views describe the same objects.
Most of this work is theoretical in spirit, and learning an abstraction from data remains largely open: existing methods either take the partition as input \citep{entner_estimating_2012,anand_causal_2023,wahl_vector_2023}, fix it to preserve a chosen estimand \citep{tikka_clustering_2023}, or learn abstraction maps under strong parametric assumptions \citep{niu_learning_2022,massidda_learning_2024,xia_neural_2024,zennaro_jointly_2023,felice_causally_2025,dacunto_causal_2025,bing_structural_2026,felekis_distributionally_2025}.

Closest to us, the coarsening framework of \citet{madaleno_coarsening_2026-1} gives a lattice-theoretic account of all valid partition-induced DAGs.
The framework has since been extended to cyclic models, where the condensation of the fine graph arises as an identifiable coarse summary \citep{madaleno_coarsening_cycles_2026}.
Among them, it singles out the \emph{interventional coarsening}, which merges exactly those variables that the available interventions cannot distinguish, and learns it from interventional data with unknown targets using RePaRe, a constraint-based procedure in two phases: first the clusters, then the edges, decided by conditional independence (CI) tests.
\autoref{app:related-work} gives a fuller account of the related literature.

Constraint-based learning of this kind inherits the familiar fragilities of CI testing.
Each edge decision rests on many hypothesis tests whose reliability degrades as conditioning sets grow and whose errors accumulate across decisions \citep{ramsey_adjacency-faithfulness_2006,tsamardinos_max-min_2006}.
Score-based methods avoid testing altogether, replacing each CI query with a comparison of a decomposable score; this is the route behind the most successful practical structure searches, from greedy equivalence search to recent order-based and discrete searches, most recently from interventional data with known targets \citep{chickering_optimal_2003,teyssier_ordering-based_2012,andrews_fast_2023,wienobst_embracing_2026,chen_i-flop_2026}.
For cluster DAGs, no such method exists.
Only the constraint-based RePaRe learns the interventional coarsening; classical score-based searches do not directly apply, since they target a single DAG over the fine variables, not a coarsening thereof.

We introduce \coarse, the first score-based method for learning the interventional coarsening from interventional---more generally, heterogeneous---data.
The setting is linear Gaussian: each non-observational environment is a soft intervention on an unknown, possibly multi-target set of variables, so all edges are preserved and every environment is Markov to the fine DAG.
\coarse keeps the two-phase template---clusters first, edges second---and replaces the CI-tested edge phase with score comparisons.
Two observations make this work.
First, the interventions themselves identify a causal order over the blocks.
Each variable carries an intervention signature, the set of environments that shift its marginal; signatures only grow along directed paths, so comparing them yields a topological order of the cluster DAG and pins down each block's possible parents.
No search over orderings is needed.
Second, the fine structural equation model induces a Gaussian block-level model: a curved exponential family whose likelihood factorizes over blocks.
The Bayesian information criterion (BIC) \citep{schwarz_estimating_1978}, summed over environments, therefore decomposes per block and is locally consistent, and classical local searches recover the true parents at the block level.
With the order and each block's possible parents settled by the signatures, edge recovery reduces to a single score-based local search per block, in the spirit of grow-shrink \citep{margaritis_bayesian_1999,lam_greedy_2022,wienobst_embracing_2026}.
Assuming error-free signatures, we prove the resulting procedure consistent, with worst-case polynomial running time.

In sum, our contributions are:
(i) we extend the interventional coarsening of \citet{madaleno_coarsening_2026-1} to unknown, multi-target interventions---the resulting coarsening merges variables with equal intervention signatures, coincides with theirs under single-target interventions, and remains identifiable from marginal shifts in the multi-target case, where theirs is not;
(ii) we show that the same shift information identifies a causal order over the blocks, removing the search over orderings and confining each block's candidate parents to a set with all of its parents and none of its descendants;
(iii) we score coarse parent sets with the BIC pooled over environments, prove the score decomposable and locally consistent, and prove \coarse consistent and with worst-case polynomial running time;
and (iv) on synthetic and real-world interventional data, \coarse matches state-of-the-art edge recovery given enough samples, while its edge phase runs up to two orders of magnitude faster.
Code is available at \repourl.

\section{Preliminaries}
\label{sec:prelim}

We review some definitions and results about graphs, probability distributions, our data-generating model, coarsenings, and local structure search that will be useful for describing our new results in \autoref{sec:method}.

\subsection{Graphs and orders}
\label{sec:graphs}

Let $G = (V,E)$ be a directed acyclic graph (DAG) with node set $V = [d]$, indexing a random vector $X = (X_v)_{v \in V} \in \R^d$ with joint distribution $P_X$, and edge set $E$.
Two nodes are \emph{adjacent} if an edge joins them, and the \emph{skeleton} $\sk(G)$ is the undirected graph obtained by ignoring edge directions.
We write $\pa[G]{v}$ and $\ch[G]{v}$ for the \emph{parents} and \emph{children} of $v$.
The \emph{ancestors} $\an[G]{v}$ (\emph{descendants} $\de[G]{v}$) are the nodes $w$ with a directed path from $w$ to $v$ (from $v$ to $w$), with the convention $v \in \an[G]{v} \cap \de[G]{v}$.
The \emph{non-descendants} are $\nd[G]{v} := V \setminus \de[G]{v}$, and all notations extend to sets $A \subseteq V$ by union, e.g.\ $\de[G]{A} := \bigcup_{a \in A}\de[G]{a}$.

\begin{definition}[Ancestral partial order]
  \label{def:ancestral-order}
  For $u,v \in V$, write $u \preceq_G v$ iff $u \in \an[G]{v}$, and $u \prec_G v$ iff $u \preceq_G v$ and $u \neq v$.
\end{definition}

\begin{definition}[Topological ordering]
  \label{def:topological-order}
  A permutation $\tau:V \to V$ is a \emph{topological} (or \emph{causal}) \emph{ordering} of $G$ if $\tau(u) < \tau(v)$ for every $u \in \an[G]{v} \setminus \{v\}$.
\end{definition}

We write $\mathcal{T}( \le )$ for the set of all linear extensions of a partial order $\le$, and $\tau_{\prec v} := \{u \mid \tau(u) < \tau(v)\}$ for the nodes before $v$ in $\tau$.
The topological orders of $G$ are exactly the elements of $\mathcal{T}( \preceq_G )$, and each satisfies $\pa[G]{v} \subseteq \tau_{\prec v} \subseteq \nd[G]{v}$.

\subsection{Conditional independence and Markov properties}
\label{sec:prob-distr}

We write $X_A \indepsym_P X_B \mid X_C$ when $X_A$ and $X_B$ are conditionally independent given $X_C$ under $P_X$, and $A \indepsym_G B \mid C$ when $A$ and $B$ are \textit{d-separated} by $C$ in $G$.

\begin{definition}[Markov property and faithfulness]
  Let $G = (V,E)$ be a DAG and let $P_X$ be a joint probability distribution over $X$ admitting a strictly positive density $p_X$.
  We say that $P_X$ satisfies the \emph{Markov factorization property} with respect to $G$ if
  \begin{equation}
    \label{eq:markov-factorization}
    p_X(x) = \prod_{v \in V} p\big(x_v \mid x_{\pa[G]{v}}\big).
  \end{equation}
  This is equivalent to the \emph{local} Markov property, $X_v \indepsym_P X_{\nd[G]{v} \setminus \pa[G]{v}} \mid X_{\pa[G]{v}}$ for every $v$, and to the \emph{global} Markov property, $A \indepsym_G B \mid S \Rightarrow X_A \indepsym_P X_B \mid X_S$ \citep{lauritzen_graphical_1996} for all pairwise disjoint $A,B,S \subseteq V$.
  Moreover we say that $P_X$ is \emph{faithful} to $G$ when every conditional independence in $P_X$ is implied by \eqref{eq:markov-factorization}.
\end{definition}

Two DAGs $G$ and $G'$ are \emph{Markov equivalent}, written $G \sim G'$, if they entail the same conditional independence statements, which holds if and only if they have the same skeleton and the same v-structures \citep{verma_equivalence_2013}.
We write $[G]$ for the Markov equivalence class of $G$.

\subsection{Local structure search}
\label{sec:local-struct-search}

Global methods search over whole graphs on $V$.
Local methods instead fix a single target $v \in V$ and ask which subset of the remaining variables shields $X_v$ from everything else.
We call a set $\BL(v) \subseteq V \setminus \{v\}$ a \emph{Markov blanket} of $X_v$ if $X_v \indepsym_P X_{V \setminus (\BL(v) \cup \{v\})} \mid X_{\BL(v)}$.
An inclusion-minimal Markov blanket is called a \emph{Markov boundary} and we denote it by $\B(v)$.
If $\indepsym_P$ is a graphoid, which holds under strict positivity of $p_X$, then $\B(v)$ exists and is unique;
under the Markov property and faithfulness, it consists of the parents, children, and parents of children of $v$ \citep{pearl_probabilistic_1988}.
Graphoids are defined in App.~\ref{app:graphs}.
In practice we can restrict the search to a candidate pool fixed in advance, typically by a topological order.

\begin{definition}[Restricted Markov boundary; \citealp{lam_greedy_2022}]
  \label{def:rmb}
  Let $v \in V$ and $Z \subseteq V \setminus \{v\}$.
  A set $\B(v,Z) \subseteq Z$ is a \emph{Markov boundary of $X_v$ relative to $Z$} if
  \begin{enumerate}[label=(\roman*), nosep]
    \item $X_v \indepsym_P X_{Z \setminus \B(v,Z)} \mid X_{\B(v,Z)}$;
    \item no proper subset $\B' \subsetneq \B(v,Z)$ satisfies (i).
  \end{enumerate}
\end{definition}

The case $Z = V \setminus \{v\}$ recovers $\B(v)$, and existence and uniqueness carry over unchanged.
If the pool contains all parents and no descendants, as the prefix $\tau_{\prec v}$ of any $\tau \in \mathcal{T}( \preceq_G )$ does, the restricted boundary collapses to the parents set.

\begin{proposition}[Collapse to parents; \citealp{lam_greedy_2022}]
  \label{prop:collapse}
  Let $P_X$ be Markov and faithful to $G$.
  If $\pa[G]{v} \subseteq Z \subseteq \nd[G]{v}$, then $\B(v,Z) = \pa[G]{v}$.
\end{proposition}

Once a topological order is known, recovering $G$ thus reduces to one restricted boundary search per node.

The grow-shrink algorithm \citep{margaritis_bayesian_1999} computes $\B(v,Z)$ in two phases.
The \emph{grow} phase starts from $S = \emptyset$ and adds each $w \in Z$ with $X_v \not\indepsym_P X_w \mid X_S$, ending as a Markov blanket relative to $Z$;
the \emph{shrink} phase removes each $w \in S$ with $X_v \indepsym_P X_w \mid X_{S \setminus \{w\}}$, discarding redundant early inclusions.
If $\indepsym_P$ is a compositional graphoid, the output is $\B(v,Z)$ after $O(|Z|^2)$ CI queries.

Following \citet{lam_greedy_2022} and the FLOP method of \citet{wienobst_embracing_2026}, we replace each CI query by a comparison of a decomposable local score $s(v,S;\cD)$, where $\cD$ consists of $n$ i.i.d.\ samples.
A candidate $w$ is added if $s(v,S \cup \{w\};\cD) > s(v,S;\cD)$ and removed if $s(v,S \setminus \{w\};\cD) > s(v,S;\cD)$.
Each phase repeats its sweep until no change occurs, so the search is non-greedy and not tied to a fixed visiting order (App.~\ref{app:gs}).
If $s$ is \emph{locally consistent}, so that as $n \to \infty$ the increment is positive exactly when $X_v \not\indepsym_P X_w \mid X_S$, every decision agrees in the limit with its CI counterpart, and the score-based version returns $\B(v,Z)$.

As the local score we use the Bayesian information criterion \citep{schwarz_estimating_1978},
\begin{equation}
  \label{eq:bic}
  s(v,S;\cD) = 2\,\hat\ell(v,S) - \lambda \log(n)\, \dim(v,S),
\end{equation}
where $\hat\ell(v,S)$ is the maximized log-likelihood of $X_v$ given $X_S$, $\dim(v,S)$ its number of free parameters, and $\lambda > 0$ a penalty multiplier ($\lambda = 1$ gives the standard BIC).
For curved exponential families, the BIC is decomposable, score equivalent, consistent, and locally consistent \citep{haughton_choice_1988,chickering_optimal_2003}.

\subsection{Linear Gaussian model and interventions}
\label{sec:data-gener-model}

Score-based methods compare graphs through the likelihood of the data, so they must assume a parametric family for $P_X$.
The most common choice is the linear Gaussian model, used for instance by GES, GIES, GnIES and FLOP
\citep{chickering_optimal_2003,hauser_characterization_2011,gamella_characterization_2025,wienobst_embracing_2026}.
Let $\cE = \{0,1,\dots,m\}$ be a finite set of environments, with $e = 0$ the observational one, and let $X^{(e)} \in \R^d$ be the random vector observed in environment $e$, with distribution $P^{(e)}$.
We assume each $X^{(e)}_v$ given $X^{(e)}_{\pa[G]{v}}$ Gaussian with mean linear in the parents and constant variance $(\omega^{(e)}_v)^2$.
In matrix form,
\begin{equation}
  \label{eq:lin-gauss}
  X^{(e)} = A^{(e)} X^{(e)} + \eta^{(e)},\qquad \eta^{(e)}
          \sim \N\big(\nu^{(e)},\Omega^{(e)}\big),
\end{equation}
where $A^{(e)}_{vu} \neq 0$ only if $u \in \pa[G]{v}$ and $\Omega^{(e)} = \diag\big((\omega^{(e)}_1)^2,\dots,(\omega^{(e)}_d)^2\big)$.
Hence $P^{(e)}$ is Gaussian, and the BIC \eqref{eq:bic} of a parent set is a least-squares regression.

Each environment $e \geq 1$ arises from a \emph{soft intervention} on an unknown target set $I_e \subseteq V$, possibly $|I_e| > 1$, changing only the conditionals of the targets, so that \eqref{eq:markov-factorization} becomes
\begin{equation}
  \label{eq:int-factorization}
  p^{(e)}(x) = \prod_{t \in I_e} p^{(e)}\big(x_t \mid x_{\pa[G]{t}}\big)\prod_{v \notin I_e} p^{(0)}\big(x_v \mid x_{\pa[G]{v}}\big).
\end{equation}
In terms of \eqref{eq:lin-gauss}, soft interventions change only the rows $A^{(e)}_{t, \cdot}$ and the pairs $(\nu^{(e)}_t,\omega^{(e)}_t)$ with $t \in I_e$, so all environments share the same DAG $G$.
We write $f^{(e)}(X_v)$ for the marginal of $X_v$ in environment $e$.
For each $e$ we observe a dataset $\cD^{(e)}$ of $n_e$ i.i.d.\ samples, independent across environments, with $n := \sum_e n_e$.
The intervention targets $I_e$ are never observed.

\subsection{Coarsenings}
\label{sec:coarsenings}

We build on the framework of \citet{madaleno_coarsening_2026-1}, making precise the \emph{resolution} identifiable under unknown, multi-target interventions.

\begin{definition}[Coarsening; \citealp{madaleno_coarsening_2026-1}]
  \label{def:coarsening}
  A coarsening of a DAG $G = (V,E)$ is a DAG $G' = (V',E')$ with a surjection $\chi:V \to V'$ such that
  \[
    E' = \{\chi(v) \to \chi(w) : v \to w \in E,\ \chi(v) \neq \chi(w)\}.
  \]
\end{definition}

We identify $V'$ with the partition $\Pi = \{\pi_1,\dots,\pi_q\}$ of $V$ whose parts group the nodes sharing a common image under $\chi$ and refer to each part as a \emph{block}.
We call $G'$ a \emph{valid} coarsening of $G$ whenever its induced edge set is acyclic.
\citet{madaleno_coarsening_2026-1} show that valid coarsenings form a lattice under refinement
($\Pi$ refines $\Pi'$, written $\Pi \preceq \Pi'$, if every part of $\Pi$ lies inside some part of $\Pi'$), and that every distribution Markov to $G$ is also Markov to each of its coarsenings.
Further, they show that a valid coarsening can be learned by RePaRe, which walks down this lattice from $\{V\}$: a \emph{Refine} step splits blocks and an \emph{IsEdge} step decides edges with conditional independence (CI) tests.

For interventional data, the authors target the \emph{interventional coarsening} $G^\cI$ (partition $\Pi^\cI$), merging nodes sharing the intervened ancestors $\an[G]{v} \cap T$, $T := \bigcup_{e \geq 1} I_e$;
with single-target interventions (and suitable assumptions), $X_v$ shifts in environment $e$ iff $v \in \de[G]{I_e}$, so RePaRe recovers $G^\cI$ by grouping equal shift patterns, without knowing the targets.
\autoref{sec:target} extends this to multi-target interventions.

\section{Method: \coarse}
\label{sec:method}

\coarse learns the interventional coarsening in two phases.
First, it reads the blocks and a causal order over them off marginal shifts alone (\autoref{sec:target} and \autoref{sec:order}).
Second, with the order fixed, it runs a single score-based local search per block (\autoref{sec:score} and \autoref{sec:guarantees}).
All proofs are deferred to \autoref{app:theory}.

\subsection{The interventional coarsening}
\label{sec:target}

We first characterize the coarsening identifiable from interventional data under multi-target interventions.

\begin{definition}[Intervention signature]
  \label{def:signature}
  The signature of a node $v \in V$ is the set of environments that shift its marginal,
  \[
    \sigma(v) := \{e \in \cE \setminus \{0\} : f^{(e)}(X_v) \neq f^{(0)}(X_v)\}.
  \]
\end{definition}

\begin{definition}[Interventional coarsening]
  \label{def:icoarsening}
  The interventional coarsening $G^{\cE}$ is the coarsening of $G$ whose surjection satisfies $\chi(v) = \chi(w) \iff \sigma(v) = \sigma(w)$.
  We denote its partition by $\Pi^{\cE}$, and call a block's common signature its \emph{support}, $\supp(\pi)$.
\end{definition}

In words, $G^{\cE}$ merges the variables the available interventions cannot tell apart; $\Pi^{\cE}$ is identifiable by construction, since $\sigma$ depends only on marginals.
\autoref{def:icoarsening} extends the interventional coarsening $G^{\cI}$ of \citet{madaleno_coarsening_2026-1}, and the two coincide under single-target interventions.
With multi-target interventions, a shift only signals that \emph{some} target lies upstream of $v$, so $\Pi^{\cI}$ is no longer identifiable from marginal shifts.
In general, $\Pi^{\cI} \preceq \Pi^{\cE}$ (App.~\ref{app:multi-target}).

Our results rely on the following assumptions.

\begin{assumption}[Identifiability assumptions; \citealp{madaleno_coarsening_2026-1}]
  \label{ass:ident}
  For the fine-grained DAG $G = (V, E)$ over the random variables indexed by $V$, and the environments $\cE$ with targets $\{I_e\}_{e \in \cE}$, $I_0 := \emptyset$, the distributions $\{f^{(e)}\}_{e \in \cE}$ satisfy:
  \begin{enumerate}[label=(\roman*), nosep]
    \item \emph{Coarse Markov}: for every $e \in \cE$ the joint law factorizes according to $G^{\cE}$,
          \begin{equation}
            \begin{aligned}
              f^{(e)}(X_V) = & \prod_{u \, \in \, \chi(V) \setminus \chi(I_e)} f^{(0)}\big(X_u \mid X_{\pa{u}}\big) \\
                             & \prod_{u \, \in \, \chi(I_e)} f^{(e)}\big(X_u \mid X_{\pa{u}}\big),
            \end{aligned}
          \end{equation}
          where $X_u := X_{\chi^{-1}(u)}$, $X_{\pa{u}} := X_{\chi^{-1}(\pa[G^{\cE}]{u})}$;
    \item \emph{Coarse faithfulness}: the only conditional independences among coarse nodes are those implied by the Markov factorization above;
    \item \emph{Interventional soundness}: every intervention induces a shift in the marginal distribution of its targets and each downstream node with respect to the observational one.
          That is, $v \in \de[G]{I_e} \Rightarrow f^{(e)}(X_v) \neq f^{(0)}(X_v)$ for all $e \in \cE \setminus \{0\}$ and $v \in V$.
  \end{enumerate}
\end{assumption}

Every (and only) descendant of an intervention target records a shift, so signatures only grow along directed edges:
\begin{lemma}[Signature monotonicity]
  \label{lem:monotone}
  Under \autoref{ass:ident}, if $u \to w$ in $G$, then $\sigma(u) \subseteq \sigma(w)$.
\end{lemma}

\begin{proposition}[Validity]
  \label{prop:validity}
  Under \autoref{ass:ident} the graph $G^{\cE}$, induced by $\chi$, is a valid coarsening of $G$.
\end{proposition}

\subsection{A causal order from the interventions}
\label{sec:order}

By \autoref{lem:monotone}, comparing supports therefore orders the blocks consistently with causation.

\begin{definition}[Support order and parent pool]
  \label{def:support}
  The \emph{support order} is the strict partial order on $\Pi^{\cE}$ given by $\pi_a <_{\supp} \pi_b \iff \supp(\pi_a) \subsetneq \supp(\pi_b)$,
  and $\Tau( <_{\supp} )$ denotes the set of its linear extensions.
  The \emph{parent pool} of a block $\pi$ is $\pastar{\pi} := \{\pi' \in \Pi^{\cE} : \pi' <_{\supp} \pi\}$;
  distinct blocks have distinct supports, so two blocks are either $<_{\supp}$-ordered or incomparable.
\end{definition}

\begin{theorem}[Supports encode the causal order]
  \label{thm:order}
  Under \autoref{ass:ident}, for distinct blocks $\pi_a, \pi_b \in \Pi^{\cE}$:
  \begin{enumerate}[label=(\roman*), nosep]
    \item \emph{Ancestry}: if $\pi_a \prec_{G^{\cE}} \pi_b$, then $\pi_a <_{\supp} \pi_b$;
    \item \emph{Parent pool}: $\pa[G^{\cE}]{\pi_a} \subseteq \pastar{\pi_a} \subseteq \nd[G^{\cE}]{\pi_a}$; in particular, blocks with incomparable supports are non-adjacent and neither is an ancestor of the other;
    \item \emph{Orientation}: if $\pi_a$ and $\pi_b$ are adjacent in $G^{\cE}$, then $\pi_a \to \pi_b \in G^{\cE}$ if and only if $\pi_a <_{\supp} \pi_b$.
  \end{enumerate}
  Hence $G^{\cE}$ is determined by its skeleton $\sk(G^{\cE})$ and the supports.
\end{theorem}

The converse of (i) fails: causally unrelated blocks can have nested supports (\autoref{ex:nested}).
Supports cannot decide adjacency, which is left to the score.

\begin{corollary}[No search over orders]
  \label{cor:noorder}
  Under \autoref{ass:ident}:
  \begin{enumerate}[label=(\roman*), nosep]
    \item every $\tau \in \Tau( <_{\supp} )$ is a topological order of $G^{\cE}$, and $\pastar{\pi} \subseteq \tau_{\prec \pi}$ for every block $\pi$;
    \item for every block $\pi$ and environment $e$, the Markov boundary of $X_{\pi}$ relative to $\pastar{\pi}$ under $P^{(e)}$ is $\pa[G^{\cE}]{\pi}$.
  \end{enumerate}
\end{corollary}

Order-based methods search over causal orders and, for each order, over parent sets.
Here the interventions fix the order once and for all, and learning $G^{\cE}$ reduces to one restricted Markov boundary search per block over a pool containing all parents and no descendants;
any parent sets chosen within the pools form a DAG.

\subsection{Scoring coarse parents}
\label{sec:score}

Without loss of generality, we assume centered data.
For each block write $X_j$ to denote $X_{\pi_j} := (X_v)_{v \in \pi_j} \in \R^{r_j}$ with $r_j := |\pi_j|$, and let $X_{\Pa_j} \in \R^{s_j}$ stack its parent blocks, with $s_j := \sum_{\pi_k \in \Pa_j} r_k$.

\begin{definition}[Coarse block-Gaussian model]
  \label{def:blockmodel}
  In environment $e$,
  \[
    X^{(e)}_j = B^{(e)}_jX^{(e)}_{\Pa_j}
                + \varepsilon^{(e)}_j,\qquad \varepsilon^{(e)}_j
              \sim \N\big(0,\Sigma^{(e)}_j\big),
  \]
  for $j = 1,\dots,q$, with unrestricted block-regression matrix $B^{(e)}_j \in \R^{r_j \times s_j}$, unrestricted $\Sigma^{(e)}_j \in \R^{r_j \times r_j} \succ 0$, and mutually independent block residuals (i.e. we assume \textit{causal sufficiency} at the coarse level).
\end{definition}

\begin{theorem}[Curved exponential family]
  \label{thm:curved}
  For each $e \in \cE$ and each DAG $H$ on $\Pi^{\cE}$, the model of \autoref{def:blockmodel} on $H$ is a curved exponential family of dimension
  \[
    \dim(H) = \sum_{j=1}^q d_j,\qquad d_j := r_js_j + \tfrac{r_j(r_j+1)}{2}.
  \]
\end{theorem}

The likelihood factorizes over blocks, so the MLE is one least-squares regression per block.
We can therefore define the score as

\begin{definition}[Pooled BIC]
  \label{def:bic}
  Let $H$ be a coarse DAG on $\Pi^{\cE}$ with parent sets $\Pa_j$, and $\cD = \{\cD^{(e)}\}_{e \in \cE}$.
  For each environment, the block score is the BIC \eqref{eq:bic} of the least-squares regression of $X_j$ on $X_{\Pa_j}$ fitted on $\cD^{(e)}$,
  \[
    \BIC^{(e)}_j(\pi_j,\Pa_j) := -n_e \log\big|\widehat\Sigma^{(e)}_j\big|
                                 - \lambda \log(n_e)\, d_j,
  \]
  where $\widehat\Sigma^{(e)}_j$ is the residual sample covariance and $d_j$ is as in \autoref{thm:curved}.
  The pooled BIC is

  \begin{equation}
    \begin{aligned}
      & \BIC(H,\cD) := \sum_{j=1}^{q} \BIC_j(\pi_j,\Pa_j),                 \\
      & \BIC_j(\pi_j,\Pa_j) := \sum_{e \in \cE} \BIC^{(e)}_j(\pi_j,\Pa_j).
    \end{aligned}
  \end{equation}
\end{definition}

Up to a constant, $\BIC(H,\cD)$ equals $\sum_{e} \big(2\hat\ell^{(e)}(H) - \lambda\log(n_e)\dim(H)\big)$, the sum of the per-environment BICs.
It is score equivalent and, by construction, decomposable, with $\BIC_j$ depending on $H$ only through $(\pi_j,\Pa_j)$; for the derivations see App.~\ref{app:score}.
We write $\indep[\cE]{\pi_j}{\pi_k}{S}$ when $\indep{X_{\pi_j}}{X_{\pi_k}}{X_S}$ holds in every $P^{(e)}$ and $X_S$ collects the variables of the blocks in $S$.
Asymptotics are taken with $n_e = \gamma_e n$, $\gamma_e > 0$ fixed, and $n \to \infty$.

\begin{proposition}[Local consistency]
  \label{prop:localcons}
  For a block $\pi_j$, a set of blocks $S$ and a block $\pi_k \notin S \cup \{\pi_j\}$,
  \[
    \BIC_j(S \cup \{\pi_k\})
    - \BIC_j(S)\xrightarrow{p} \begin{cases}
                                 +\infty & \text{if } \nindep[\cE]{\pi_j}{\pi_k}{S}, \\
                                 -\infty & \text{if } \indep[\cE]{\pi_j}{\pi_k}{S}.
                               \end{cases}
  \]
\end{proposition}

Intuitively, a single dependent environment contributes an order-$n$ gain that beats every order-$\log n$ penalty; with no dependent environment, only the penalties remain.

\subsection{COARSE and its guarantees}
\label{sec:guarantees}

We now turn to estimation.
Following \citet{madaleno_coarsening_2026-1}, we collect the signatures in the \emph{intervention descendant matrix} $M \in \{0,1\}^{d \times m}$, with $M_{v,e} = \1[e \in \sigma(v)]$, and define the blocks of $\Pi^{\cE}$ as the classes of equal rows of $M$.
Each entry tests whether the marginal of $X_v$ in environment $e \neq 0$ differs from the observational one.
Under Gaussianity we use a Welch $t$-test \citep{welch_generalization_1947} at level $\alpha$:
\[
  M_{v,e} = \1\bigl[p_{v,e} < \alpha\bigr], \qquad p_{v,e}
          := \mathrm{Welch}_t\bigl(X_v^{(e)}, X_v^{(0)}\bigr).
\]
The Welch $t$-test only detects mean shifts.
An intervention that changes only higher-order moments of $X_v$ (e.g., its variance) gives $M_{v,e} = 0$ even though $f^{(e)}(X_v) \neq f^{(0)}(X_v)$.
\autoref{ass:ident}, however, allows any distributional change.
We keep the Welch $t$-test as a cheap default for the Gaussian linear setting, but it can be swapped for a general two-sample test (e.g., Kolmogorov-Smirnov, kernel-based \citep{gretton_kernel_2012} or energy distance \citep{szekely_energy_2013} tests) with no other change to the algorithm.
Given $M$, \autoref{alg:refinetest} refines the partition greedily, moving down the meet-subsemilattice of valid coarsenings to the partition identified by $M$.
Once the partition is returned, we compute the estimated support of each block and its parent pool $\hatpastar{\pi}$ as in \autoref{def:support}.
For each block we then run a single \textsc{GrowShrink} call (\autoref{alg:gs}) over its parent pool, estimating its coarse parent set $\widehat{\Pa}_j$.
The procedure returns the coarse DAG $\hat G = \big(\hat\Pi,\{\pi_k \to \pi_j:\pi_k \in \widehat{\Pa}_j\}\big)$ (\autoref{alg:coarse}).

As is standard for constraint-based and hybrid methods, including RePaRe \citep{madaleno_coarsening_2026-1}, we treat the shift tests as an oracle: $\hat M_{v,e} = M_{v,e}$ for all $v$ and $e$.
Since $\hat\Pi$, the supports, and the pools are deterministic functions of $\hat M$, they then coincide respectively with $\Pi^{\cE}$, $\supp$, and $\pastar{\cdot}$, and the order of \autoref{sec:order} is exact.
What remains is to prove consistency of the score.

\begin{theorem}[Consistency of \coarse]
  \label{thm:consistency}
  Under \autoref{ass:ident}, \autoref{def:blockmodel} and oracle tests, as $n \to \infty$ the output of \autoref{alg:coarse} satisfies $\PP(\hat G = G^{\cE}) \to 1$, and it does not depend on the order in which blocks are processed.
\end{theorem}

To be precise, being score equivalent, the pooled BIC alone identifies $G^{\cE}$ only up to $[G^{\cE}]$, whose members share the skeleton.
The heterogeneity of the data enters only through the supports, which fix every orientation, so no interventional equivalence class or interventional score is needed.

\begin{theorem}[Worst-case complexity]
  \label{thm:complexity}
  With the Welch test and $n_0 = \cO(n/|\cE|)$, \coarse runs in time $\cO\big(d^2 n + |\cE|\,q^{3}d^{3}\big)$,
  where $d$ is the number of fine-grained nodes, $n = \sum_e n_e$ the total sample size and
  $q = |\widehat\Pi^{\cE}| \le d$ the number of estimated blocks.
\end{theorem}

\autoref{thm:consistency} assumes oracle shift tests and enough samples for every block MLE to exist, i.e.\ $n_e \ge r_j + s_j$ (\autoref{app:block}).
In finite samples a fixed $\alpha$ over $dm$ tests can split or merge blocks, and large blocks inflate the model parameters, making estimation complex and data hungry.
Three variants address these issues (\autoref{app:heuristics}; used in \autoref{sec:experiments}): CV-\coarse selects $\alpha$ by cross-validation, kPC-\coarse compresses blocks to their leading principal components, and avg-\coarse summarizes blocks by their standardized mean.

\section{Experiments}
\label{sec:experiments}

We validate the results of \autoref{sec:method} on synthetic and real data: convergence of the recovered partition and edges, scaling in the number of variables (\autoref{sec:exp-scal}), edge recovery and runtime against RePaRe \citep{madaleno_coarsening_2026-1} (\autoref{sec:exp-compare}), and the light-tunnel causal chamber \citep{gamella_causal_2025} (\autoref{sec:exp-chamber}).
Additional experiments on sensitivity and robustness to assumption violations, together with theoretical results for multi-target interventions and high-dimensional regimes, are reported in \autoref{app:experiments}.

\subsection{Setup}
\label{sec:exp-setup}

We sample linear Gaussian models \eqref{eq:lin-gauss} with \texttt{sempler} \citep{gamella_characterization_2025}.
Ground-truth DAGs are Erd\H{o}s--R\'enyi (ER) graphs with edge density $\rho$.
Edge weights have magnitude uniform in $[0.5, 2]$ and random sign, and each model is standardized to unit observational variance \citep{ormaniec_standardizing_2025}, to avoid variance artifacts that would invalidate the results \citep{reisach_beware_2021}.
Each of the $m$ interventional environments is a soft shift intervention on unknown targets drawn without replacement, adding $\N(2,1)$ noise to each targeted variable.
All environments have the same sample size $n_e$, and every configuration is repeated over 20 seeds.
Unless stated otherwise, \coarse builds $\hat M$ with a Welch $t$-test at level $\alpha = 10^{-4}$ and uses $\lambda = 1$; full details in App.~\ref{app:exp-details}.

Partition recovery is measured by the adjusted Rand index (ARI) between $\hat\Pi$ and $\Pi^\cE$ \citep{warrens_understanding_2022} ($1$ for a perfect match, $0$ at chance), edge recovery by precision, recall and $F_1$-score on the coarse DAG, and runtime as the wall-clock time of one fit on one CPU core.

\begin{figure*}[t]
  \centering
  \begin{minipage}{0.245\textwidth}
    \centering
    \includegraphics[width=0.98\linewidth]{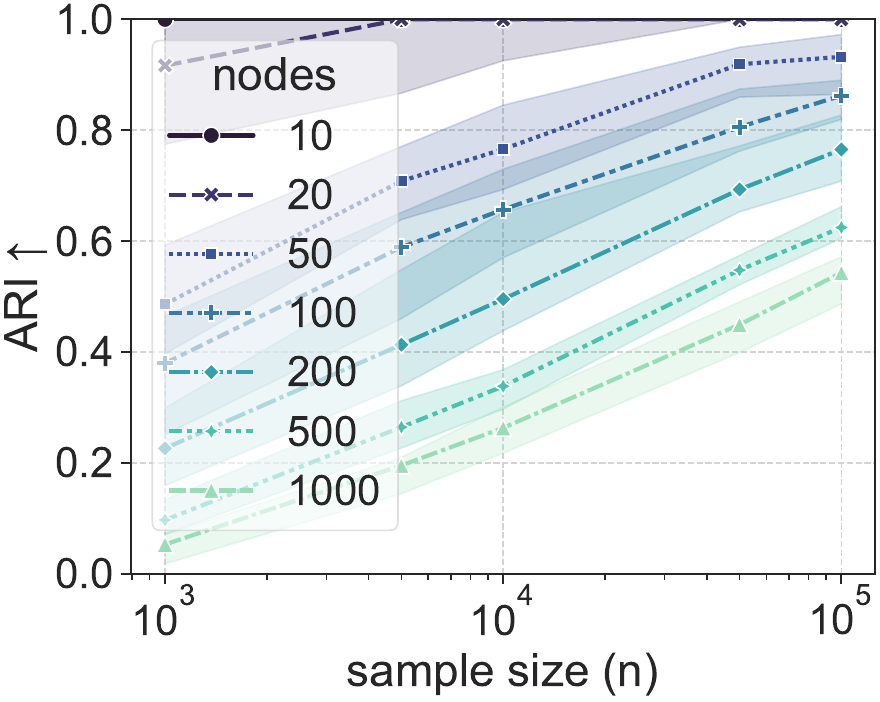}\\
    {\small (a) partition recovery}
  \end{minipage}
  \hfill
  \begin{minipage}{0.245\textwidth}
    \centering
    \includegraphics[width=0.98\linewidth]{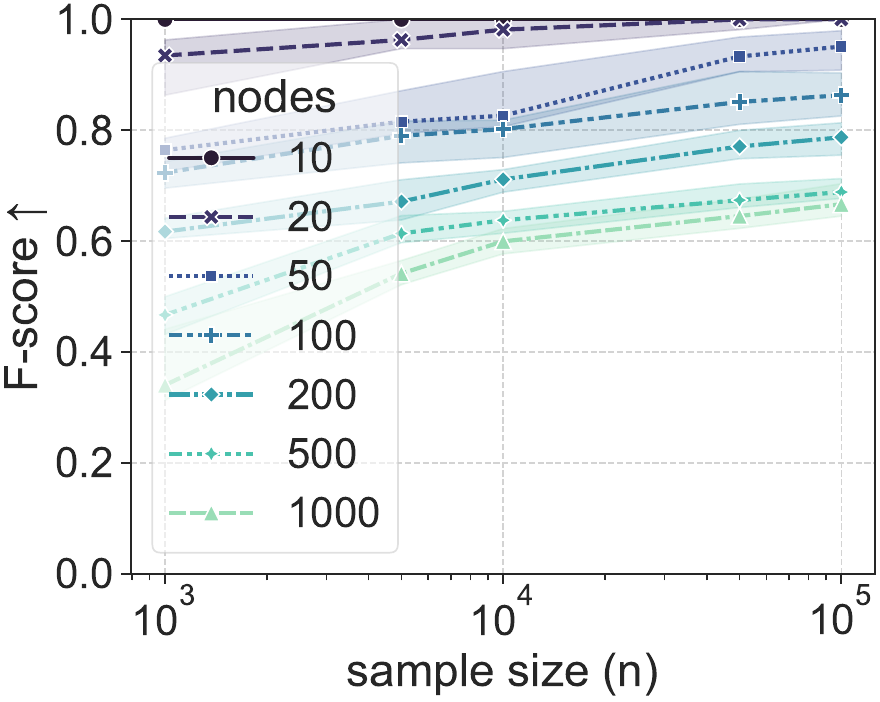}\\
    {\small (b) coarse-edge recovery}
  \end{minipage}
  \hfill
  \begin{minipage}{0.245\textwidth}
    \centering
    \includegraphics[width=0.98\linewidth]{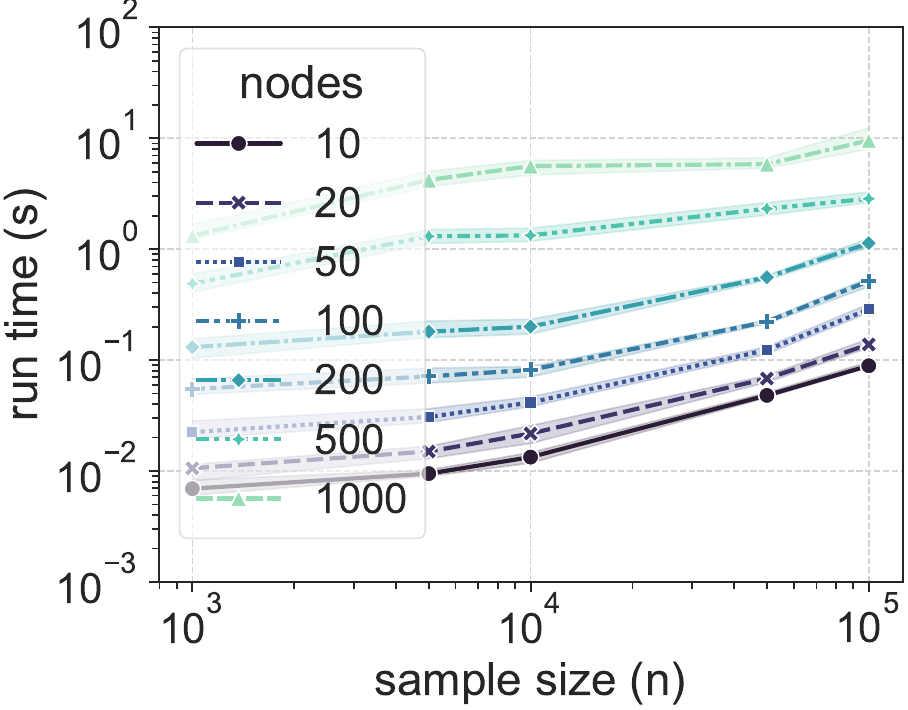}\\
    {\small (c) runtime in $n_e$}
  \end{minipage}
  \hfill
  \begin{minipage}{0.245\textwidth}
    \centering
    \includegraphics[width=0.98\linewidth]{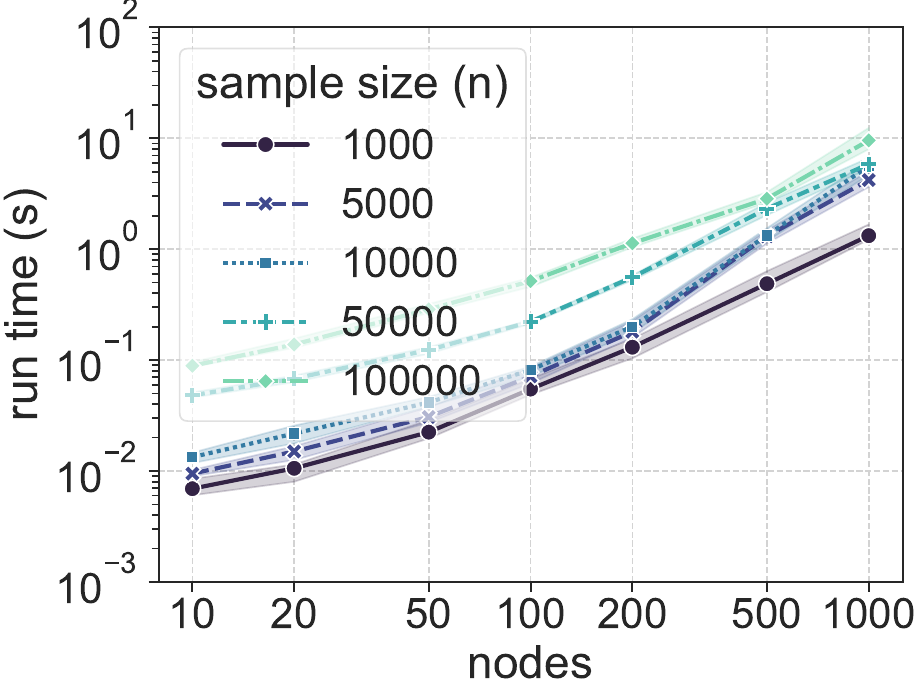}\\
    {\small (d) runtime in $d$}
  \end{minipage}
  \caption{
    Full \coarse pipeline on ER graphs, $\rho = 0.2$, $m = 5$.
    Both recovery metrics improve with $n_e$ at every $d$.
  }
  \label{fig:scalability}
\end{figure*}

\subsection{Consistency and scalability}
\label{sec:exp-scal}

\autoref{fig:scalability} reports runs of the full \autoref{alg:coarse} for $d$ from $10$ to $1000$ and $n_e$ from $10^3$ to $10^5$.
At every $d$, both ARI and F-score increase with $n_e$, confirming \autoref{thm:consistency}.
For small graphs ($\approx d \le 20$), both the partition and the coarse edges are recovered near-perfectly once $n_e \ge 10^4$.
Larger graphs need more data: at $n_e = 10^5$ the F-score varies from $0.95$ for $d = 50$ to $0.79$ for $d = 200$ and $0.67$ for $d = 1000$, while ARI drops from $0.93$ at $d = 50$ to $0.54$ at $d = 1000$.
Nonetheless, under the oracle partition, \autoref{tab:compare} shows near-perfect F-scores up to $d = 500$ once $n_e \ge 5 \times 10^3$, suggesting the loss in performance stems from the partition estimation stage: $\hat M$ rests on $d \cdot m$ Welch tests, and one missed or spurious shift moves a variable to another block, so partition errors grow with $d$.
Runtime instead grows far below the worst-case bound of \autoref{thm:complexity} in both $n_e$ and $d$: about $10$ seconds for $d = 1000$ with $10^5$ samples, and about one second for $d = 200$.

\subsection{Comparison against RePaRe}
\label{sec:exp-compare}

RePaRe keeps a coarse edge only if a CI test rejects in every environment, while \coarse adds a parent block as soon as the pooled BIC improves.
\autoref{tab:compare} reports edge recovery under the oracle partition---a stage common to both methods---on ER graphs with $\rho = 0.2$, $m = 5$ and $d$ up to $500$ (full grid in \autoref{tab:compare-full}).
With small samples, RePaRe's many CI tests lack power, and a single environment that fails to reject is enough to miss an edge.
The \coarse score instead suffers on large graphs with few interventions and small samples, where blocks of tens of variables make the model parameters hard to estimate.
From $n_e = 5 \times 10^3$ on, \coarse reaches an F-score of at least $0.99$ at every $d$, with speed-ups from $9 \times$ to $510 \times$.

\begin{table*}[t]
  \centering
  \footnotesize
  \caption{
    Edge recovery under the oracle partition, ER graphs, $\rho = 0.2$, $m = 5$.
    F-score as mean $\pm$ sample standard deviation over 20 seeds; speed-up is the ratio of mean RePaRe to mean \coarse runtime.
    Bold marks the better method where the two differ at the reported precision.
  }
  \label{tab:compare}
  \begin{tabular}{lccccccc}
    \toprule
    $d$          &       10        &            20            &            50            &           100            &       200       &           300            &           500            \\
    \midrule
    \multicolumn{8}{l}{\textbf{F-score $\uparrow$, $n_e = 10^3$}}                                                                                                                           \\[2pt]
    \coarse      & 1.00 $\pm$ 0.00 & \textbf{1.00 $\pm$ 0.02} & \textbf{0.99 $\pm$ 0.02} & \textbf{0.98 $\pm$ 0.02} & 0.95 $\pm$ 0.05 &     0.84 $\pm$ 0.08      &     0.71 $\pm$ 0.05      \\
    RePaRe       & 1.00 $\pm$ 0.01 &     0.96 $\pm$ 0.03      &     0.95 $\pm$ 0.05      &     0.93 $\pm$ 0.07      & 0.95 $\pm$ 0.08 & \textbf{0.95 $\pm$ 0.10} & \textbf{0.97 $\pm$ 0.04} \\
    \midrule
    \multicolumn{8}{l}{\textbf{F-score $\uparrow$, $n_e = 10^5$}}                                                                                                                           \\[2pt]
    \coarse      & 1.00 $\pm$ 0.00 &     1.00 $\pm$ 0.00      &     1.00 $\pm$ 0.00      &     1.00 $\pm$ 0.00      & 1.00 $\pm$ 0.00 & \textbf{1.00 $\pm$ 0.00} &     1.00 $\pm$ 0.00      \\
    RePaRe       & 1.00 $\pm$ 0.00 &     1.00 $\pm$ 0.00      &     1.00 $\pm$ 0.00      &     1.00 $\pm$ 0.00      & 1.00 $\pm$ 0.00 &     0.99 $\pm$ 0.01      &     1.00 $\pm$ 0.01      \\
    \midrule
    \multicolumn{8}{l}{\textbf{Speed-up $\uparrow$ (RePaRe runtime / \coarse runtime)}}                                                                                                     \\[2pt]
    $n_e = 10^3$ &    9$\times$    &        21$\times$        &        38$\times$        &        44$\times$        &   37$\times$    &        37$\times$        &        32$\times$        \\
    $n_e = 10^5$ &   51$\times$    &       172$\times$        &       342$\times$        &       510$\times$        &   509$\times$   &       424$\times$        &       329$\times$        \\
    \bottomrule
  \end{tabular}
\end{table*}

\subsection{Real data: the causal chamber}
\label{sec:exp-chamber}

We use measurements from the light-tunnel causal chamber \citep{gamella_causal_2025}, whose causal graph is known by construction: $d = 20$ actuator and sensor variables, one observational dataset and five interventional ones, each perturbing one light intensity (R, G, B) or one polarizer angle.
Besides RePaRe, we run three fine-grained baselines that also use interventional data to estimate the true graph: GIES \citep{hauser_characterization_2011} (with the intervention target assumed known), GnIES \citep{gamella_characterization_2025} and UT-IGSP \citep{squires_permutation-based_2019} (both of which estimate the intervention targets).
For each method with hyperparameters, we compare selected values against oracle ones: \coarse selects by cross-validation, UT-IGSP by BIC, and RePaRe uses the GnIES score, which needs the intervention targets to pool data across environments.

\coarse and RePaRe recover the true partition (ARI $= 1$) and the same coarse DAG, with $8$ of the $10$ true coarse edges and no false edge (\autoref{tab:chamber}).
Since the fine-grained methods learn a different object, we score all methods at both resolutions: a coarse edge expands to all variable pairs it connects, and a fine-grained graph collapses onto the true partition.
Each directed F-score is paired with a negative control \citep{petersen_are_2025} measuring whether a method improves over chance.
\coarse and RePaRe score highest at both levels and above their null intervals.
GIES beats chance at both levels while UT-IGSP only at the variable level.
GnIES recovers directions no better than chance, as reported by \citet{gamella_characterization_2025}: interventions in this dataset mostly shift the noise mean, which GnIES removes by centering the data.

\begin{table*}[t]
  \centering
  \footnotesize
  \caption{
    Light tunnel, ungrouped mode.
    Directed F-score at the variable and block level, each against its negative control (mean and 95\% interval over $1000$ random DAGs with the same number of edges).
    Fit and Search are wall-clock seconds of the final fit and hyperparameter search.
  }
  \label{tab:chamber}
  \begin{tabular}{lccccccc}
    \toprule
            &      & \multicolumn{2}{c}{\textbf{Variable level}} &  \multicolumn{2}{c}{\textbf{Block level}}  &        &          \\
    \cmidrule(lr){3-4}\cmidrule(lr){5-6}
    \textbf{Method}  & \textbf{ARI}  & \textbf{F1 $\uparrow$} &       \textbf{Null}        & \textbf{F1 $\uparrow$} &       \textbf{Null}        & \textbf{Fit(s)} & \textbf{Search(s)} \\
    \midrule
    \coarse & 1.00 & \textbf{0.74} & 0.13 [0.00, 0.38] & \textbf{0.89} & 0.17 [0.00, 0.44] &  0.02  &   0.40    \\
    RePaRe  & 1.00 & \textbf{0.74} & 0.12 [0.00, 0.35] & \textbf{0.89} & 0.16 [0.00, 0.44] &  0.10  &   1.54    \\
    GIES    &  --  &     0.62      & 0.10 [0.02, 0.21] &     0.80      & 0.27 [0.12, 0.45] &  1.83  &    --     \\
    GnIES   &  --  &     0.02      & 0.12 [0.02, 0.23] &     0.07      & 0.27 [0.11, 0.46] &  376   &    --     \\
    UT-IGSP &  --  &     0.44      & 0.12 [0.02, 0.23] &     0.39      & 0.28 [0.11, 0.45] &  0.19  &   2.69    \\
    \bottomrule
  \end{tabular}
\end{table*}

\section{Discussion}
\label{sec:discussion}

We introduced \coarse, the first score-based method for learning a coarsened causal model from heterogeneous data.
Intervention signatures fix the partition and a causal order over the blocks, and a BIC pooled over environments drives one local search per block.
The target, the interventional coarsening $G^{\cE}$, extends that of \citet{madaleno_coarsening_2026-1} to unknown, multi-target interventions and is identifiable from marginal shifts alone.
The procedure is consistent and runs in polynomial time;
it matches RePaRe on edge recovery given enough samples and runs faster in every configuration, by up to $510 \times$ under the oracle partition,
and on the light-tunnel chamber it recovers the block structure of the physical system in a fraction of a second.

The design carries four limitations.
First, the pipeline is hybrid and sequential: the partition rests on $d \cdot m$ Welch tests, with their multiple-testing burden and limited power at small samples, and a misread signature moves a variable between blocks before any edge is scored;
at large $d$ the loss comes mostly from this stage (\autoref{fig:scalability}).
These tests also detect only mean shifts, so with them the oracle-test assumption of \autoref{thm:consistency} requires every shift in \autoref{ass:ident}(iii) to move the mean.
Second, the pooled BIC is liberal by construction: a coarse edge survives on the evidence of a single environment (\autoref{sec:score}), whereas the all-environments rule of RePaRe errs conservative (\autoref{sec:exp-compare}).
The same score can also err conservative through its penalties, since adding a parent block costs $r_j r_k$ parameters per environment and large blocks observed in few environments need many samples before their true edges appear.
Third, the resolution of $G^{\cE}$ is bounded by the available interventions: multi-target environments can merge variables that single-target ones would separate.
Fourth, the guarantees rest on acyclic fine graphs, a linear Gaussian model without hidden confounders, soft interventions, and an observational baseline: hard interventions would have each environment score a mutilated coarse graph, which we conjecture breaks the guarantee,
and without a baseline neither the partition nor the order is defined.

Each limitation suggests future work.
Signatures could be read from non-parametric or description-length criteria rather than tests \citep{janzing_causal_2008,marx_formally_2021,mameche_discovering_2022}, or the partition and the coarse edges estimated jointly under a single objective \citep{niu_learning_2022}.
A target-aware score in the spirit of GnIES \citep{gamella_characterization_2025} that does not know the targets could pool environments more sample-efficiently,
and resampling with stability selection \citep{meinshausen_stability_2010} would attach confidence to each coarse edge.
Interventions themselves could be chosen to refine the coarsening where it matters most.
Further out, natural next steps include cyclic fine graphs, already identifiable in the linear non-Gaussian setting \citep{madaleno_coarsening_cycles_2026}, baseline-free regimes, perhaps via context variables \citep{mooij_joint_2020}, and nonlinear or non-Gaussian scores.

The systems that motivate causal abstraction---high-dimensional measurements of cells, brains, and images---face two obstacles at once: the right macro-variables are unknown, and the fine DAG is out of reach.
When heterogeneity is available, \coarse addresses both:
it reads the grouping and the order off the shifts, leaving the data only a small, polynomial skeleton search.

\clearpage
\bibliography{references}

@article{benjamini_control_2001,
  title = {The control of the false discovery rate in multiple testing under dependency},
  volume = {29},
  doi = {10.1214/aos/1013699998},
  language = {en},
  number = {4},
  urldate = {2026-10-06},
  journal = {The Annals of Statistics},
  publisher = {Institute of Mathematical Statistics},
  author = {Benjamini, Yoav and Yekutieli, Daniel},
  month = aug,
  year = {2001},
  pages = {1165--1188},
}

@article{benjamini_controlling_1995,
  title = {Controlling the {False} {Discovery} {Rate}: {A} {Practical} and {Powerful} {Approach} to {Multiple} {Testing}},
  volume = {57},
  copyright = {https://academic.oup.com/journals/pages/open\_access/funder\_policies/chorus/standard\_publication\_model},
  issn = {1369-7412, 1467-9868},
  shorttitle = {Controlling the {False} {Discovery} {Rate}},
  url = {https://academic.oup.com/jrsssb/article/57/1/289/7035855},
  doi = {10.1111/j.2517-6161.1995.tb02031.x},
  language = {en},
  number = {1},
  urldate = {2026-10-06},
  journal = {Journal of the Royal Statistical Society Series B: Statistical Methodology},
  author = {Benjamini, Yoav and Hochberg, Yosef},
  month = jan,
  year = {1995},
  pages = {289--300},
}

@article{dunn_multiple_1961,
  title = {Multiple {Comparisons} among {Means}},
  volume = {56},
  issn = {0162-1459},
  url = {https://www.tandfonline.com/doi/abs/10.1080/01621459.1961.10482090},
  doi = {10.1080/01621459.1961.10482090},
  number = {293},
  urldate = {2026-10-06},
  journal = {Journal of the American Statistical Association},
  publisher = {Taylor \& Francis},
  author = {Dunn, Olive Jean},
  month = mar,
  year = {1961},
  pages = {52--64},
}

@misc{lam_greedy_2022,
  title = {Greedy {Relaxations} of the {Sparsest} {Permutation} {Algorithm}},
  url = {http://arxiv.org/abs/2206.05421},
  doi = {10.48550/arXiv.2206.05421},
  urldate = {2026-10-06},
  publisher = {arXiv},
  author = {Lam, Wai-Yin and Andrews, Bryan and Ramsey, Joseph},
  month = jun,
  year = {2022},
  note = {arXiv:2206.05421 [cs.AI]},
}

@article{szekely_energy_2013,
  title = {Energy statistics: {A} class of statistics based on distances},
  volume = {143},
  issn = {03783758},
  shorttitle = {Energy statistics},
  url = {https://linkinghub.elsevier.com/retrieve/pii/S0378375813000633},
  doi = {10.1016/j.jspi.2013.03.018},
  language = {en},
  number = {8},
  urldate = {2026-10-06},
  journal = {Journal of Statistical Planning and Inference},
  author = {Székely, Gábor J. and Rizzo, Maria L.},
  month = aug,
  year = {2013},
  pages = {1249--1272},
}

@misc{bing_structural_2026,
  title = {Structural {Causal} {Bottleneck} {Models}},
  url = {http://arxiv.org/abs/2603.08682},
  doi = {10.48550/arXiv.2603.08682},
  urldate = {2026-05-27},
  publisher = {arXiv},
  author = {Bing, Simon and Wahl, Jonas and Runge, Jakob},
  month = mar,
  year = {2026},
  note = {arXiv:2603.08682 [stat.ML]},
}

@inproceedings{gamella_active_2020,
  title = {Active {Invariant} {Causal} {Prediction}: {Experiment} {Selection} through {Stability}},
  volume = {33},
  shorttitle = {Active {Invariant} {Causal} {Prediction}},
  url = {https://proceedings.neurips.cc/paper/2020/hash/b197ffdef2ddc3308584dce7afa3661b-Abstract.html},
  urldate = {2026-09-23},
  booktitle = {Advances in {Neural} {Information} {Processing} {Systems}},
  publisher = {Curran Associates, Inc.},
  author = {Gamella, Juan L. and Heinze-Deml, Christina},
  year = {2020},
  pages = {15464--15475},
}

@misc{kawahara_grouplingam_2010,
  title = {{GroupLiNGAM}: {Linear} non-{Gaussian} acyclic models for sets of variables},
  shorttitle = {{GroupLiNGAM}},
  url = {http://arxiv.org/abs/1006.5041},
  doi = {10.48550/arXiv.1006.5041},
  urldate = {2026-09-21},
  publisher = {arXiv},
  author = {Kawahara, Yoshinobu and Bollen, Kenneth and Shimizu, Shohei and Washio, Takashi},
  month = jun,
  year = {2010},
  note = {arXiv:1006.5041 [cs.AI]},
}

@misc{reisach_beware_2021,
  title = {Beware of the {Simulated} {DAG}! {Causal} {Discovery} {Benchmarks} {May} {Be} {Easy} {To} {Game}},
  url = {http://arxiv.org/abs/2102.13647},
  doi = {10.48550/arXiv.2102.13647},
  urldate = {2026-05-31},
  publisher = {arXiv},
  author = {Reisach, Alexander G. and Seiler, Christof and Weichwald, Sebastian},
  month = nov,
  year = {2021},
  note = {arXiv:2102.13647 [stat.ML]},
}

@inproceedings{petersen_are_2025,
  title = {Are {You} {Doing} {Better} {Than} {Random} {Guessing}? {A} {Call} for {Using} {Negative} {Controls} {When} {Evaluating} {Causal} {Discovery} {Algorithms}},
  issn = {2640-3498},
  shorttitle = {Are {You} {Doing} {Better} {Than} {Random} {Guessing}?},
  url = {https://proceedings.mlr.press/v286/petersen25a.html},
  language = {en},
  urldate = {2026-07-13},
  booktitle = {Proceedings of the {Forty}-first {Conference} on {Uncertainty} in {Artificial} {Intelligence}},
  publisher = {PMLR},
  author = {Petersen, Anne Helby},
  month = jul,
  year = {2025},
  pages = {3465--3479},
}

@inproceedings{squires_causal_2022,
  title = {Causal {Structure} {Discovery} between {Clusters} of {Nodes} {Induced} by {Latent} {Factors}},
  issn = {2640-3498},
  url = {https://proceedings.mlr.press/v177/squires22a.html},
  language = {en},
  urldate = {2026-09-09},
  booktitle = {Proceedings of the {First} {Conference} on {Causal} {Learning} and {Reasoning}},
  publisher = {PMLR},
  author = {Squires, Chandler and Yun, Annie and Nichani, Eshaan and Agrawal, Raj and Uhler, Caroline},
  month = jun,
  year = {2022},
  pages = {669--687},
}

@misc{ormaniec_standardizing_2025,
  title = {Standardizing {Structural} {Causal} {Models}},
  url = {http://arxiv.org/abs/2406.11601},
  doi = {10.48550/arXiv.2406.11601},
  urldate = {2026-09-08},
  publisher = {arXiv},
  author = {Ormaniec, Weronika and Sussex, Scott and Lorch, Lars and Schölkopf, Bernhard and Krause, Andreas},
  month = mar,
  year = {2025},
  note = {arXiv:2406.11601 [cs.LG]
    version: 3},
}

@article{tsamardinos_max-min_2006,
  title = {The {Max}-{Min} {Hill}-{Climbing} {Bayesian} {Network} {Structure} {Learning} {Algorithm}},
  volume = {65},
  doi = {10.1007/s10994-006-6889-7},
  journal = {Machine Learning},
  author = {Tsamardinos, Ioannis and Brown, Laura and Aliferis, Constantin},
  month = oct,
  year = {2006},
  pages = {31--78},
}

@misc{chen_i-flop_2026,
  title = {I-{FLOP}: {Fast} {Learning} of {Order} and {Parents} from {Interventional} {Data}},
  shorttitle = {I-{FLOP}},
  url = {http://arxiv.org/abs/2608.28245},
  doi = {10.48550/arXiv.2608.28245},
  urldate = {2026-09-03},
  publisher = {arXiv},
  author = {Chen, Liuting and Markham, Alex},
  month = aug,
  year = {2026},
  note = {arXiv:2608.28245 [stat.ML]},
}

@article{hyttinen_experiment_2013,
  title = {Experiment {Selection} for {Causal} {Discovery}},
  volume = {14},
  issn = {1533-7928},
  url = {http://jmlr.org/papers/v14/hyttinen13a.html},
  number = {93},
  urldate = {2026-08-11},
  journal = {Journal of Machine Learning Research},
  author = {Hyttinen, Antti and Eberhardt, Frederick and Hoyer, Patrik O.},
  year = {2013},
  pages = {3041--3071},
}

@misc{uhler_gaussian_2017,
  title = {Gaussian {Graphical} {Models}: {An} {Algebraic} and {Geometric} {Perspective}},
  shorttitle = {Gaussian {Graphical} {Models}},
  url = {http://arxiv.org/abs/1707.04345},
  doi = {10.48550/arXiv.1707.04345},
  urldate = {2026-05-24},
  publisher = {arXiv},
  author = {Uhler, Caroline},
  month = jul,
  year = {2017},
  note = {arXiv:1707.04345 [math.ST]},
}

@inproceedings{zennaro_jointly_2023,
  title = {Jointly {Learning} {Consistent} {Causal} {Abstractions} {Over} {Multiple} {Interventional} {Distributions}},
  issn = {2640-3498},
  url = {https://proceedings.mlr.press/v213/zennaro23a.html},
  language = {en},
  urldate = {2026-05-27},
  booktitle = {Proceedings of the {Second} {Conference} on {Causal} {Learning} and {Reasoning}},
  publisher = {PMLR},
  author = {Zennaro, Fabio Massimo and Drávucz, Máté and Apachitei, Geanina and Widanage, W. Dhammika and Damoulas, Theodoros},
  month = aug,
  year = {2023},
  pages = {88--121},
}

@article{zennaro_abstraction_2022,
  title={Abstraction between structural causal models: {A} review of definitions and properties},
  author={Zennaro, Fabio Massimo},
  journal={arXiv preprint arXiv:2207.08603},
  year={2022}
}

@article{xia_neural_2024,
  title = {Neural {Causal} {Abstractions}},
  volume = {38},
  copyright = {Copyright (c) 2024 Association for the Advancement of Artificial Intelligence},
  issn = {2374-3468},
  url = {https://ojs.aaai.org/index.php/AAAI/article/view/30044},
  doi = {10.1609/aaai.v38i18.30044},
  language = {en},
  number = {18},
  urldate = {2026-05-27},
  journal = {Proceedings of the AAAI Conference on Artificial Intelligence},
  author = {Xia, Kevin and Bareinboim, Elias},
  month = mar,
  year = {2024},
  pages = {20585--20595},
}

@InProceedings{schooltink_aligning_2024,
  title =	 {Aligning Graphical and Functional Causal
    Abstractions},
  author =	 {Schooltink, Willem and Zennaro, Fabio Massimo},
  booktitle =	 {Proceedings of the Fourth Conference on Causal
    Learning and Reasoning},
  pages =	 {704--730},
  year =	 2025,
  editor =	 {Huang, Biwei and Drton, Mathias},
  volume =	 275,
  series =	 {Proceedings of Machine Learning Research},
  month =	 {07--09 May},
  publisher =	 {PMLR},
  url =
  {https://proceedings.mlr.press/v275/schooltink25a.html},
}

@inproceedings{massidda_learning_2024,
  address = {Barcelona, Spain},
  series = {{UAI} '24},
  title = {Learning causal abstractions of linear structural causal models},
  volume = {244},
  urldate = {2026-04-17},
  booktitle = {Proceedings of the {Fortieth} {Conference} on {Uncertainty} in {Artificial} {Intelligence}},
  publisher = {JMLR.org},
  author = {Massidda, Riccardo and Magliacane, Sara and Bacciu, Davide},
  month = jul,
  year = {2024},
  pages = {2486--2515},
}

@misc{massidda_causal_2022,
  title = {Causal {Abstraction} with {Soft} {Interventions}},
  url = {http://arxiv.org/abs/2211.12270},
  doi = {10.48550/arXiv.2211.12270},
  urldate = {2026-05-07},
  author = {Massidda, Riccardo and Geiger, Atticus and Icard, Thomas and Bacciu, Davide},
  month = nov,
  year = {2022},
  note = {arXiv:2211.12270 [cs]},
}

@misc{felekis_distributionally_2025,
  title = {Distributionally {Robust} {Causal} {Abstractions}},
  copyright = {Creative Commons Attribution 4.0 International},
  url = {https://arxiv.org/abs/2510.04842},
  doi = {10.48550/ARXIV.2510.04842},
  language = {en},
  urldate = {2026-03-28},
  publisher = {arXiv},
  author = {Felekis, Yorgos and Damoulas, Theodoros and Giampouras, Paris},
  year = {2025},
  note = {Version Number: 3},
}

@misc{dacunto_causal_2025,
  title = {Causal {Abstraction} {Learning} based on the {Semantic} {Embedding} {Principle}},
  url = {http://arxiv.org/abs/2502.00407},
  doi = {10.48550/arXiv.2502.00407},
  urldate = {2026-05-27},
  publisher = {arXiv},
  author = {D'Acunto, Gabriele and Zennaro, Fabio Massimo and Felekis, Yorgos and Lorenzo, Paolo Di},
  month = may,
  year = {2025},
  note = {arXiv:2502.00407 [cs.LG]},
}

@article{beckers_abstracting_2019,
  title = {Abstracting {Causal} {Models}},
  volume = {33},
  copyright = {Copyright (c) 2019 Association for the Advancement of Artificial Intelligence},
  issn = {2374-3468},
  url = {https://ojs.aaai.org/index.php/AAAI/article/view/4117},
  doi = {10.1609/aaai.v33i01.33012678},
  language = {en},
  number = {01},
  urldate = {2026-04-16},
  journal = {Proceedings of the AAAI Conference on Artificial Intelligence},
  author = {Beckers, Sander and Halpern, Joseph Y.},
  month = jul,
  year = {2019},
  pages = {2678--2685},
}

@misc{beckers_approximate_2019,
  title = {Approximate {Causal} {Abstraction}},
  url = {https://arxiv.org/abs/1906.11583v2},
  language = {en},
  author = {Beckers, Sander and Eberhardt, Frederick and Halpern, Joseph Y.},
  month = jun,
  year = {2019},
}

@misc{madaleno_coarsening_2026-1,
  title = {Coarsening {Causal} {DAG} {Models}},
  url = {http://arxiv.org/abs/2601.10531},
  doi = {10.48550/arXiv.2601.10531},
  urldate = {2026-05-10},
  publisher = {arXiv},
  author = {Madaleno, Francisco and Misra, Pratik and Markham, Alex},
  month = apr,
  year = {2026},
  note = {arXiv:2601.10531 [stat]},
}

@misc{eberhardt_almost_2012,
  title = {Almost {Optimal} {Intervention} {Sets} for {Causal} {Discovery}},
  url = {https://resolver.caltech.edu/CaltechAUTHORS:20190327-085859738},
  doi = {10.48550/arxiv.1206.3250},
  author = {Eberhardt, Frederick},
  month = jun,
  year = {2012},
}

@misc{ghassami_optimal_2017,
  title = {Optimal {Experiment} {Design} for {Causal} {Discovery} from {Fixed} {Number} of {Experiments}},
  url = {https://arxiv.org/pdf/1702.08567},
  doi = {10.48550/arxiv.1702.08567},
  author = {Ghassami, AmirEmad and Salehkaleybar, Saber and Kiyavash, Negar},
  month = feb,
  year = {2017},
}

@article{hauser_two_2014,
  title = {Two optimal strategies for active learning of causal models from interventional data},
  volume = {55},
  issn = {0888613X},
  url = {https://linkinghub.elsevier.com/retrieve/pii/S0888613X13002879},
  doi = {10.1016/j.ijar.2013.11.007},
  language = {en},
  number = {4},
  urldate = {2026-07-13},
  journal = {International Journal of Approximate Reasoning},
  author = {Hauser, Alain and Bühlmann, Peter},
  month = jun,
  year = {2014},
  pages = {926--939},
}

@article{touw_clusterpath_2026,
  title = {Clusterpath {Gaussian} {Graphical} {Modeling}},
  volume = {0},
  issn = {1061-8600},
  url = {https://doi.org/10.1080/10618600.2026.2653764},
  doi = {10.1080/10618600.2026.2653764},
  number = {0},
  urldate = {2026-06-09},
  journal = {Journal of Computational and Graphical Statistics},
  publisher = {Taylor \& Francis},
  author = {Touw, D. J. W. and Alfons, A. and Groenen, P. J. F. and Wilms, I.},
  month = may,
  year = {2026},
  note = {\_eprint: https://doi.org/10.1080/10618600.2026.2653764},
  pages = {1--15},
}

@inproceedings{niu_learning_2022,
  address = {Vienna, Austria},
  title = {Learning {Cluster} {Causal} {Diagrams}: {An} {Information}-{Theoretic} {Approach}},
  isbn = {978-1-956792-00-3},
  shorttitle = {Learning {Cluster} {Causal} {Diagrams}},
  url = {https://www.ijcai.org/proceedings/2022/675},
  doi = {10.24963/ijcai.2022/675},
  language = {en},
  urldate = {2026-03-28},
  booktitle = {Proceedings of the {Thirty}-{First} {International} {Joint} {Conference} on {Artificial} {Intelligence}},
  publisher = {International Joint Conferences on Artificial Intelligence Organization},
  author = {Niu, Xueyan and Li, Xiaoyun and Li, Ping},
  month = jul,
  year = {2022},
  pages = {4871--4877},
}

@book{lauritzen_graphical_1996,
  address = {Oxford, New York},
  series = {Oxford {Statistical} {Science} {Series}},
  title = {Graphical {Models}},
  isbn = {978-0-19-852219-5},
  publisher = {Oxford University Press},
  author = {Lauritzen, Steffen L.},
  month = may,
  year = {1996},
}

@article{andrews_fast_2023,
  title = {Fast {Scalable} and {Accurate} {Discovery} of {DAGs} {Using} the {Best} {Order} {Score} {Search} and {Grow}-{Shrink} {Trees}},
  volume = {36},
  issn = {1049-5258},
  url = {https://pmc.ncbi.nlm.nih.gov/articles/PMC11393735/},
  urldate = {2026-04-03},
  journal = {Advances in neural information processing systems},
  author = {Andrews, Bryan and Ramsey, Joseph and Sánchez-Romero, Rubén and Camchong, Jazmin and Kummerfeld, Erich},
  month = dec,
  year = {2023},
  pages = {63945--63956},
}

@misc{wienobst_embracing_2026,
  title = {Embracing {Discrete} {Search}: {A} {Reasonable} {Approach} to {Causal} {Structure} {Learning}},
  shorttitle = {Embracing {Discrete} {Search}},
  url = {http://arxiv.org/abs/2510.04970},
  doi = {10.48550/arXiv.2510.04970},
  urldate = {2026-06-01},
  publisher = {arXiv},
  author = {Wienöbst, Marcel and Henckel, Leonard and Weichwald, Sebastian},
  month = feb,
  year = {2026},
  note = {arXiv:2510.04970 [stat.ML]},
}

@misc{squires_permutation-based_2019,
  title = {Permutation-{Based} {Causal} {Structure} {Learning}  with {Unknown} {Intervention} {Targets}},
  url = {https://arxiv.org/abs/1910.09007v2},
  language = {en},
  urldate = {2026-04-03},
  author = {Squires, Chandler and Wang, Yuhao and Uhler, Caroline},
  month = oct,
  year = {2019},
}

@misc{solus_consistency_2021,
  title = {Consistency {Guarantees} for {Greedy} {Permutation}-{Based} {Causal} {Inference} {Algorithms}},
  url = {http://arxiv.org/abs/1702.03530},
  doi = {10.48550/arXiv.1702.03530},
  urldate = {2026-04-03},
  publisher = {arXiv},
  author = {Solus, Liam and Wang, Yuhao and Uhler, Caroline},
  month = jun,
  year = {2021},
  note = {arXiv:1702.03530 [math]},
}

@inproceedings{margaritis_bayesian_1999,
  title = {Bayesian {Network} {Induction} via {Local} {Neighborhoods}},
  volume = {12},
  url = {https://proceedings.neurips.cc/paper_files/paper/1999/hash/5d79099fcdf499f12b79770834c0164a-Abstract.html},
  urldate = {2026-04-14},
  booktitle = {Advances in {Neural} {Information} {Processing} {Systems}},
  publisher = {MIT Press},
  author = {Margaritis, Dimitris and Thrun, Sebastian},
  year = {1999},
}

@misc{gamella_characterization_2025,
  title = {Characterization and {Greedy} {Learning} of {Gaussian} {Structural} {Causal} {Models} under {Unknown} {Interventions}},
  url = {http://arxiv.org/abs/2211.14897},
  doi = {10.48550/arXiv.2211.14897},
  urldate = {2026-04-10},
  publisher = {arXiv},
  author = {Gamella, Juan L. and Taeb, Armeen and Heinze-Deml, Christina and Bühlmann, Peter},
  month = mar,
  year = {2025},
  note = {arXiv:2211.14897 [stat]},
}

@article{chickering_optimal_2003,
  title = {Optimal structure identification with greedy search},
  volume = {3},
  issn = {1532-4435},
  url = {https://dl.acm.org/doi/10.1162/153244303321897717},
  doi = {10.1162/153244303321897717},
  number = {null},
  urldate = {2026-04-10},
  journal = {J. Mach. Learn. Res.},
  author = {Chickering, David Maxwell},
  month = mar,
  year = {2003},
  pages = {507--554},
}

@misc{marx_formally_2021,
  title = {Formally {Justifying} {MDL}-based {Inference} of {Cause} and {Effect}},
  url = {http://arxiv.org/abs/2105.01902},
  doi = {10.48550/arXiv.2105.01902},
  urldate = {2026-06-01},
  publisher = {arXiv},
  author = {Marx, Alexander and Vreeken, Jilles},
  month = may,
  year = {2021},
  note = {arXiv:2105.01902 [cs.IT]},
}

@misc{janzing_causal_2008,
  title = {Causal inference using the algorithmic {Markov} condition},
  url = {http://arxiv.org/abs/0804.3678},
  doi = {10.48550/arXiv.0804.3678},
  urldate = {2026-06-01},
  publisher = {arXiv},
  author = {Janzing, Dominik and Schoelkopf, Bernhard},
  month = apr,
  year = {2008},
  note = {arXiv:0804.3678 [math.ST]},
}

@inproceedings{mameche_discovering_2022,
  address = {Washington DC USA},
  title = {Discovering {Invariant} and {Changing} {Mechanisms} from {Data}},
  isbn = {978-1-4503-9385-0},
  url = {https://dl.acm.org/doi/10.1145/3534678.3539479},
  doi = {10.1145/3534678.3539479},
  language = {en},
  urldate = {2026-06-01},
  booktitle = {Proceedings of the 28th {ACM} {SIGKDD} {Conference} on {Knowledge} {Discovery} and {Data} {Mining}},
  publisher = {ACM},
  author = {Mameche, Sarah and Kaltenpoth, David and Vreeken, Jilles},
  month = aug,
  year = {2022},
  pages = {1242--1252},
}

@article{warrens_understanding_2022,
  title = {Understanding the {Adjusted} {Rand} {Index} and {Other} {Partition} {Comparison} {Indices} {Based} on {Counting} {Object} {Pairs}},
  volume = {39},
  issn = {1432-1343},
  url = {https://doi.org/10.1007/s00357-022-09413-z},
  doi = {10.1007/s00357-022-09413-z},
  language = {en},
  number = {3},
  urldate = {2026-06-01},
  journal = {Journal of Classification},
  author = {Warrens, Matthijs J. and van der Hoef, Hanneke},
  month = nov,
  year = {2022},
  pages = {487--509},
}

@article{gamella_causal_2025,
  title = {Causal chambers as a real-world physical testbed for {AI} methodology},
  volume = {7},
  copyright = {2025 The Author(s)},
  issn = {2522-5839},
  url = {https://www.nature.com/articles/s42256-024-00964-x},
  doi = {10.1038/s42256-024-00964-x},
  language = {en},
  number = {1},
  urldate = {2026-06-01},
  journal = {Nature Machine Intelligence},
  publisher = {Nature Publishing Group},
  author = {Gamella, Juan L. and Peters, Jonas and Bühlmann, Peter},
  month = jan,
  year = {2025},
  pages = {107--118},
}

@article{welch_generalization_1947,
  title = {The {Generalization} of `{Student}'s' {Problem} when {Several} {Different} {Population} {Variances} are {Involved}},
  volume = {34},
  issn = {0006-3444},
  url = {https://www.jstor.org/stable/2332510},
  doi = {10.2307/2332510},
  number = {1/2},
  urldate = {2026-05-30},
  journal = {Biometrika},
  publisher = {[Oxford University Press, Biometrika Trust]},
  author = {Welch, B. L.},
  year = {1947},
  pages = {28--35},
}

@article{gretton_kernel_2012,
  title = {A {Kernel} {Two}-{Sample} {Test}},
  volume = {13},
  issn = {1533-7928},
  url = {http://jmlr.org/papers/v13/gretton12a.html},
  number = {25},
  urldate = {2026-05-30},
  journal = {Journal of Machine Learning Research},
  author = {Gretton, Arthur and Borgwardt, Karsten M. and Rasch, Malte J. and Schölkopf, Bernhard and Smola, Alexander},
  year = {2012},
  pages = {723--773},
}

@article{andersson_alternative_2001,
  title = {Alternative {Markov} {Properties} for {Chain} {Graphs}},
  volume = {28},
  issn = {0303-6898},
  url = {https://www.jstor.org/stable/4616644},
  number = {1},
  urldate = {2026-05-28},
  journal = {Scandinavian Journal of Statistics},
  publisher = {[Board of the Foundation of the Scandinavian Journal of Statistics, Wiley]},
  author = {Andersson, Steen A. and Madigan, David and Perlman, Michael D.},
  year = {2001},
  pages = {33--85},
}

@inproceedings{ramsey_adjacency-faithfulness_2006,
  address = {Arlington, Virginia, USA},
  series = {{UAI}'06},
  title = {Adjacency-faithfulness and conservative causal inference},
  isbn = {978-0-9749039-2-7},
  urldate = {2026-05-28},
  booktitle = {Proceedings of the {Twenty}-{Second} {Conference} on {Uncertainty} in {Artificial} {Intelligence}},
  publisher = {AUAI Press},
  author = {Ramsey, Joseph and Spirtes, Peter and Zhang, Jiji},
  month = jul,
  year = {2006},
  pages = {401--408},
}

@article{wahl_vector_2023,
  title = {Vector {Causal} {Inference} between {Two} {Groups} of {Variables}},
  volume = {37},
  copyright = {Copyright (c) 2023 Association for the Advancement of Artificial Intelligence},
  issn = {2374-3468},
  url = {https://ojs.aaai.org/index.php/AAAI/article/view/26450},
  doi = {10.1609/aaai.v37i10.26450},
  language = {en},
  number = {10},
  urldate = {2026-05-27},
  journal = {Proceedings of the AAAI Conference on Artificial Intelligence},
  author = {Wahl, Jonas and Ninad, Urmi and Runge, Jakob},
  month = jun,
  year = {2023},
  pages = {12305--12312},
}

@misc{verma_equivalence_2013,
  title = {On the {Equivalence} of {Causal} {Models}},
  url = {http://arxiv.org/abs/1304.1108},
  doi = {10.48550/arXiv.1304.1108},
  urldate = {2026-04-02},
  publisher = {arXiv},
  author = {Verma, Tom S. and Pearl, Judea},
  month = mar,
  year = {2013},
  note = {arXiv:1304.1108 [cs]},
}

@misc{peters_causal_2015,
  title = {Causal inference using invariant prediction: identification and confidence intervals},
  shorttitle = {Causal inference using invariant prediction},
  url = {http://arxiv.org/abs/1501.01332},
  doi = {10.48550/arXiv.1501.01332},
  urldate = {2026-04-13},
  publisher = {arXiv},
  author = {Peters, Jonas and Bühlmann, Peter and Meinshausen, Nicolai},
  month = nov,
  year = {2015},
  note = {arXiv:1501.01332 [stat]},
}

@misc{entner_estimating_2012,
  title = {Estimating a {Causal} {Order} among {Groups} of {Variables} in {Linear} {Models}},
  url = {http://arxiv.org/abs/1207.1977},
  doi = {10.48550/arXiv.1207.1977},
  urldate = {2026-05-27},
  publisher = {arXiv},
  author = {Entner, Doris and Hoyer, Patrik O.},
  month = jul,
  year = {2012},
  note = {arXiv:1207.1977 [stat.ML]},
}

@misc{tikka_clustering_2023,
  title = {Clustering and {Structural} {Robustness} in {Causal} {Diagrams}},
  url = {http://arxiv.org/abs/2111.04513},
  doi = {10.48550/arXiv.2111.04513},
  urldate = {2026-05-26},
  publisher = {arXiv},
  author = {Tikka, Santtu and Helske, Jouni and Karvanen, Juha},
  month = aug,
  year = {2023},
  note = {arXiv:2111.04513 [stat.ML]},
}

@misc{hauser_characterization_2011,
  title = {Characterization and {Greedy} {Learning} of {Interventional} {Markov} {Equivalence} {Classes} of {Directed} {Acyclic} {Graphs}},
  url = {https://arxiv.org/abs/1104.2808v2},
  language = {en},
  urldate = {2026-05-26},
  author = {Hauser, Alain and Bühlmann, Peter},
  month = apr,
  year = {2011},
}

@article{spirtes_introduction_2010,
  title = {Introduction to {Causal} {Inference}},
  volume = {11},
  issn = {1533-7928},
  url = {http://jmlr.org/papers/v11/spirtes10a.html},
  number = {54},
  urldate = {2026-05-25},
  journal = {Journal of Machine Learning Research},
  author = {Spirtes, Peter},
  year = {2010},
  pages = {1643--1662},
}

@article{spirtes_algorithm_1991,
  title = {An {Algorithm} for {Fast} {Recovery} of {Sparse} {Causal} {Graphs}},
  volume = {9},
  issn = {0894-4393},
  url = {https://doi.org/10.1177/089443939100900106},
  doi = {10.1177/089443939100900106},
  language = {EN},
  number = {1},
  urldate = {2026-05-25},
  journal = {Social Science Computer Review},
  publisher = {SAGE Publications Inc},
  author = {Spirtes, Peter and Glymour, Clark},
  month = apr,
  year = {1991},
  pages = {62--72},
}

@article{grosse-wentrup_neuro-cognitive_2024,
  title = {Neuro-cognitive multilevel causal modeling: {A} framework that bridges the explanatory gap between neuronal activity and cognition},
  volume = {20},
  issn = {1553-7358},
  shorttitle = {Neuro-cognitive multilevel causal modeling},
  url = {https://journals.plos.org/ploscompbiol/article?id=10.1371/journal.pcbi.1012674},
  doi = {10.1371/journal.pcbi.1012674},
  language = {en},
  number = {12},
  urldate = {2026-05-14},
  journal = {PLOS Computational Biology},
  publisher = {Public Library of Science},
  author = {Grosse-Wentrup, Moritz and Kumar, Akshey and Meunier, Anja and Zimmer, Manuel},
  month = dec,
  year = {2024},
  pages = {e1012674},
}

@misc{chalupka_visual_2014,
  title = {Visual {Causal} {Feature} {Learning}},
  url = {https://arxiv.org/abs/1412.2309v2},
  language = {en},
  urldate = {2026-05-14},
  author = {Chalupka, Krzysztof and Perona, Pietro and Eberhardt, Frederick},
  month = dec,
  year = {2014},
}

@misc{rubenstein_causal_2017,
  title = {Causal {Consistency} of {Structural} {Equation} {Models}},
  url = {http://arxiv.org/abs/1707.00819},
  doi = {10.48550/arXiv.1707.00819},
  urldate = {2026-05-13},
  publisher = {arXiv},
  author = {Rubenstein, Paul K. and Weichwald, Sebastian and Bongers, Stephan and Mooij, Joris M. and Janzing, Dominik and Grosse-Wentrup, Moritz and Schölkopf, Bernhard},
  month = jul,
  year = {2017},
  note = {arXiv:1707.00819 [stat.ML]},
}

@article{meinshausen_stability_2010,
  title = {Stability selection},
  volume = 72,
  copyright = {© 2010 Royal Statistical Society},
  issn = {1467-9868},
  doi = {10.1111/j.1467-9868.2010.00740.x},
  language = {en},
  number = 4,
  urldate = {2026-05-12},
  journal = {Journal of the Royal Statistical Society: Series B (Statistical Methodology)},
  author = {Meinshausen, Nicolai and Bühlmann, Peter},
  year = 2010,
  pages = {417--473},
}

@article{mooij_joint_2020,
  title = {Joint {Causal} {Inference} from {Multiple} {Contexts}},
  volume = {21},
  issn = {1533-7928},
  url = {http://jmlr.org/papers/v21/17-123.html},
  number = {99},
  urldate = {2026-04-17},
  journal = {Journal of Machine Learning Research},
  author = {Mooij, Joris M. and Magliacane, Sara and Claassen, Tom},
  year = {2020},
  pages = {1--108},
}

@article{felice_causally_2025,
  title={Causally reliable concept bottleneck models},
  author={De Felice, Giovanni and Casanova Flores, Arianna and De Santis, Francesco and Santini, Silvia and Schneider, Johannes and Barbiero, Pietro and Termine, Alberto},
  journal={Advances in Neural Information Processing Systems},
  volume={38},
  pages={149099--149139},
  year={2026},
  url = {https://openreview.net/forum?id=UX143QGvb8},
}

@book{vaart_asymptotic_1998,
  address = {Cambridge},
  series = {Cambridge {Series} in {Statistical} and {Probabilistic} {Mathematics}},
  title = {Asymptotic {Statistics}},
  isbn = {978-0-521-78450-4},
  url = {https://www.cambridge.org/core/books/asymptotic-statistics/A3C7DAD3F7E66A1FA60E9C8FE132EE1D},
  doi = {10.1017/CBO9780511802256},
  urldate = {2026-05-05},
  publisher = {Cambridge University Press},
  author = {Vaart, A. W. van der},
  year = {1998},
}

@book{pearl_probabilistic_1988,
  address = {San Francisco, CA, United States},
  title = {Probabilistic {Reasoning} in {Intelligent} {Systems}},
  isbn = {978-0-08-051489-5},
  doi = {10.1016/C2009-0-27609-4},
  language = {en},
  urldate = {2026-05-04},
  publisher = {Morgan Kaufmann Publishers Inc.},
  author = {Pearl, Judea},
  year = {1988},
  note = {ISBN: 9780080514895},
}

@article{margaritis_learning_2003,
  title = {Learning {Bayesian} {Network} {Model} {Structure} from {Data}},
  language = {en},
  journal = {PhD Thesis},
  author = {Margaritis, Dimitris},
  year = {2003},
}

@misc{teyssier_ordering-based_2012,
  title = {Ordering-{Based} {Search}: {A} {Simple} and {Effective} {Algorithm} for {Learning} {Bayesian} {Networks}},
  shorttitle = {Ordering-{Based} {Search}},
  url = {http://arxiv.org/abs/1207.1429},
  doi = {10.48550/arXiv.1207.1429},
  urldate = {2026-05-04},
  publisher = {arXiv},
  author = {Teyssier, Marc and Koller, Daphne},
  month = jul,
  year = {2012},
  note = {arXiv:1207.1429 [cs]},
}

@misc{verma_causal_1988,
  title = {Causal {Networks}: {Semantics} and {Expressiveness}},
  shorttitle = {Causal {Networks}},
  url = {http://arxiv.org/abs/1304.2379},
  doi = {10.48550/arXiv.1304.2379},
  urldate = {2026-05-03},
  publisher = {arXiv},
  author = {Verma, Tom S. and Pearl, Judea},
  year = {1988},
  note = {arXiv:1304.2379 [cs]},
}

@article{haughton_choice_1988,
  title = {On the {Choice} of a {Model} to {Fit} {Data} from an {Exponential} {Family}},
  volume = {16},
  issn = {0090-5364},
  url = {https://www.jstor.org/stable/2241441},
  number = {1},
  urldate = {2026-04-27},
  journal = {The Annals of Statistics},
  publisher = {Institute of Mathematical Statistics},
  author = {Haughton, Dominique M. A.},
  year = {1988},
  pages = {342--355},
}

@article{schwarz_estimating_1978,
  title = {Estimating the {Dimension} of a {Model}},
  volume = {6},
  issn = {0090-5364},
  url = {https://www.jstor.org/stable/2958889},
  number = {2},
  urldate = {2026-04-26},
  journal = {The Annals of Statistics},
  publisher = {Institute of Mathematical Statistics},
  author = {Schwarz, Gideon},
  year = {1978},
  pages = {461--464},
}

@article{drton_maximum_2006,
  title = {Maximum {Likelihood} {Estimation} in {Gaussian} {Chain} {Graph} {Models} under the {Alternative} {Markov} {Property}},
  volume = {33},
  issn = {0303-6898, 1467-9469},
  url = {https://onlinelibrary.wiley.com/doi/10.1111/j.1467-9469.2006.00482.x},
  doi = {10.1111/j.1467-9469.2006.00482.x},
  language = {en},
  number = {2},
  urldate = {2026-04-09},
  journal = {Scandinavian Journal of Statistics},
  author = {Drton, Mathias and Eichler, Michael},
  month = jun,
  year = {2006},
  pages = {247--257},
}

@article{andersson_normal_1998,
  title = {Normal {Linear} {Regression} {Models} {With} {Recursive} {Graphical} {Markov} {Structure}},
  volume = {66},
  copyright = {https://www.elsevier.com/tdm/userlicense/1.0/},
  issn = {0047259X},
  url = {https://linkinghub.elsevier.com/retrieve/pii/S0047259X98917456},
  doi = {10.1006/jmva.1998.1745},
  language = {en},
  number = {2},
  urldate = {2026-04-20},
  journal = {Journal of Multivariate Analysis},
  author = {Andersson, Steen A and Perlman, Michael D},
  month = aug,
  year = {1998},
  pages = {133--187},
}

@article{wahl_foundations_2024,
  title = {Foundations of causal discovery on groups of variables},
  volume = {12},
  copyright = {http://creativecommons.org/licenses/by/4.0},
  issn = {2193-3685},
  url = {https://www.degruyterbrill.com/document/doi/10.1515/jci-2023-0041/html},
  doi = {10.1515/jci-2023-0041},
  language = {en},
  number = {1},
  urldate = {2026-04-16},
  journal = {Journal of Causal Inference},
  author = {Wahl, Jonas and Ninad, Urmi and Runge, Jakob},
  month = jul,
  year = {2024},
  pages = {20230041},
}

@article{kitson_survey_2023,
  title = {A survey of {Bayesian} {Network} structure learning},
  volume = {56},
  issn = {0269-2821, 1573-7462},
  url = {https://link.springer.com/10.1007/s10462-022-10351-w},
  doi = {10.1007/s10462-022-10351-w},
  language = {en},
  number = {8},
  urldate = {2026-04-02},
  journal = {Artificial Intelligence Review},
  author = {Kitson, Neville Kenneth and Constantinou, Anthony C. and Guo, Zhigao and Liu, Yang and Chobtham, Kiattikun},
  month = aug,
  year = {2023},
  pages = {8721--8814},
}

@misc{chickering_large-sample_2012,
  title = {Large-{Sample} {Learning} of {Bayesian} {Networks} is {NP}-{Hard}},
  url = {http://arxiv.org/abs/1212.2468},
  doi = {10.48550/arXiv.1212.2468},
  urldate = {2026-04-02},
  publisher = {arXiv},
  author = {Chickering, David Maxwell and Meek, Christopher and Heckerman, David},
  month = oct,
  year = {2012},
  note = {arXiv:1212.2468 [cs]},
}

@article{anand_causal_2023,
  title = {Causal {Effect} {Identification} in {Cluster} {DAGs}},
  volume = {37},
  issn = {2374-3468, 2159-5399},
  url = {https://ojs.aaai.org/index.php/AAAI/article/view/26435},
  doi = {10.1609/aaai.v37i10.26435},
  language = {en},
  number = {10},
  urldate = {2026-03-28},
  journal = {Proceedings of the AAAI Conference on Artificial Intelligence},
  author = {Anand, Tara V. and Ribeiro, Adele H. and Tian, Jin and Bareinboim, Elias},
  month = jun,
  year = {2023},
  pages = {12172--12179},
}

@inproceedings{madaleno_coarsening_cycles_2026,
  title =	 {Coarsening Linear {Non-Gaussian} Causal Models with
    Cycles},
  author =	 {Madaleno, Francisco and Pereira, Francisco C. and
    Markham, Alex},
  booktitle =	 {Advances in Neural Information Processing Systems
    (NeurIPS)},
  year =	 2026,
  publisher =	 {In Press},
  url =		 {https://arxiv.org/abs/2605.10163}
}

@inproceedings{lacerda2008discovering,
  title =	 {Discovering cyclic causal models by independent
    components analysis},
  author =	 {Lacerda, Gustavo and Spirtes, Peter and Ramsey,
    Joseph and Hoyer, Patrik O.},
  booktitle =	 {Conference on Uncertainty in Artificial Intelligence
    (UAI)},
  year =	 2008
}

@book{stanley2011enumerative,
  title =	 {Enumerative Combinatorics, Volume 1},
  author =	 {Stanley, Richard P.},
  publisher =	 {Cambridge Studies in Advanced Mathematics},
  year =	 2011,
  edition =	 {Second}
}

@article{gu2020learning,
  author =	 {Jiaying Gu and Qing Zhou},
  title =	 {Learning Big Gaussian Bayesian Networks:
    {P}artition, Estimation and Fusion},
  journal =	 {Journal of Machine Learning Research},
  year =	 2020,
  volume =	 21,
  number =	 158,
  pages =	 {1--31},
  url =		 {http://jmlr.org/papers/v21/19-318.html}
}

@article{caspard2003lattices,
  title={The lattices of closure systems, closure operators, and implicational systems on a finite set: a survey},
  author={Caspard, Nathalie and Monjardet, Bernard},
  journal={Discrete Applied Mathematics},
  volume={127},
  number={2},
  pages={241--269},
  year={2003},
  publisher={Elsevier},
  doi={10.1016/S0166-218X(02)00209-3}
}

\clearpage
\appendix
\appendixautoref
\thispagestyle{empty}

\onecolumn
\aistatstitle{Score-Based Learning of Cluster DAGs from Interventions: \\
  Supplementary Materials}

\section{Extended related work}
\label{app:related-work}

\paragraph{Causal abstraction}

The functional view formalizes an abstraction as a pair of maps between structural causal models that commute with interventions.
\citet{rubenstein_causal_2017} introduce exact transformations, which \citet{beckers_abstracting_2019} refine into a hierarchy of uniform, $\tau$-, strong $\tau$-, and constructive abstractions, the last of which arises from an explicit partition of the low-level variables; approximate abstractions relax commutativity \citep{beckers_approximate_2019}, \citet{massidda_causal_2022} extend $\tau$-abstraction to soft interventions, and \citet{zennaro_abstraction_2022} survey the resulting landscape.
On the graphical side, \citet{anand_causal_2023} introduce C-DAGs---graphs over an assumed clustering with unspecified intra-cluster structure---and prove d-separation, do-calculus, and the ID algorithm sound and complete on them whenever the clustering is admissible.
\citet{tikka_clustering_2023} take the causal query as given and characterize the transit clusters whose collapse preserves identifiability, and \citet{wahl_foundations_2024} generalize to group directed mixed graphs with cycles and latents, showing that the Markov property transfers under grouping while faithfulness requires internally well-connected clusters.
\citet{schooltink_aligning_2024} prove that C-DAGs correspond exactly to constructive $\tau$-abstractions; a coarsening is the C-DAG induced by its partition, so learning one from data is implicitly learning a constructive functional abstraction.
\citet{madaleno_coarsening_cycles_2026} extend the same framework to linear non-Gaussian models with cycles: any DAG-coarsening must group strongly connected components, and the resulting condensation is identifiable from observational data alone, invariant across the ICA equivalence class of \citet{lacerda2008discovering}.

\paragraph{Learning abstractions from data}

Existing efforts parameterize the abstraction map with deep architectures \citep{xia_neural_2024,dacunto_causal_2025,bing_structural_2026}, learn it under linearity assumptions \citep{massidda_learning_2024}, use optimal transport for alignment \citep{felice_causally_2025}, or target distributionally robust abstractions \citep{felekis_distributionally_2025}.
\citet{zennaro_jointly_2023} jointly learn abstractions consistent over multiple interventional distributions, and \citet{niu_learning_2022} learn cluster causal diagrams information-theoretically from observational data.
Another thread learns structure given a known clustering \citep{entner_estimating_2012,kawahara_grouplingam_2010,wahl_vector_2023}, and partitions have long been used to scale structure learning of large networks \citep{gu2020learning,touw_clusterpath_2026}.
Most of these works either assume the partition is provided or impose a strong parametric structure; jointly discovering the partition and its abstract DAG from data remains an exception, as does any score-based treatment.

\paragraph{Structure learning paradigms}

\citet{kitson_survey_2023} survey the field.
Constraint-based methods, e.g.\ PC \citep{spirtes_algorithm_1991}, prune a complete graph by CI tests; their order dependence and the deterioration of high-order tests make errors accumulate \citep{ramsey_adjacency-faithfulness_2006,tsamardinos_max-min_2006}.
Score-based methods optimize a decomposable criterion such as the BIC \citep{schwarz_estimating_1978}, at the price of an NP-hard search \citep{chickering_large-sample_2012}: greedy equivalence search \citep{chickering_optimal_2003}, ordering-based search \citep{teyssier_ordering-based_2012}, permutation-based methods with consistency guarantees \citep{solus_consistency_2021} and their greedy relaxations \citep{lam_greedy_2022}, reinsertion-based searches with grow-shrink trees \citep{andrews_fast_2023}, and discrete-search approaches \citep{wienobst_embracing_2026}.
Our edge phase uses the score-based restricted grow-shrink backbone \citep{lam_greedy_2022,wienobst_embracing_2026}, with a coarse, multi-environment score in place of the fine, single-environment BIC.

\paragraph{Learning from interventional data}

Interventions come in hard (graph-mutilating) and soft (mechanism-changing) varieties; soft interventions preserve the DAG and are the natural regime for heterogeneous data whose targets are unknown.
\citet{hauser_characterization_2011} characterize interventional Markov equivalence for known hard targets and learn it greedily (GIES), \citet{gamella_characterization_2025} treat unknown soft targets (GnIES), \citet{squires_permutation-based_2019} estimate targets jointly with the graph (UT-IGSP), and \citet{peters_causal_2015} exploit invariance across environments without target knowledge (ICP).
Further work addresses experimental design and active learning \citep{eberhardt_almost_2012,hyttinen_experiment_2013,ghassami_optimal_2017,gamella_active_2020,hauser_two_2014}.
\citet{chen_i-flop_2026} extend the order-based FLOP algorithm \citep{wienobst_embracing_2026} to interventional data with known targets, scoring prefix-restricted parent sets with a target-filtered pooled interventional BIC.
Our setting is multi-environment with unknown, possibly multi-target soft interventions, and the interventions determine not only the identifiability target but also the search order.

\section{Additional theory and proofs}
\label{app:theory}

\autoref{app:graphs} to \autoref{app:multi-target} collect background results; \autoref{app:order} to \autoref{app:complexity} prove the results of \autoref{sec:method}.

\subsection{Graph and probability theory}
\label{app:graphs}

\subsubsection{d-separation.}

Given a DAG $G$, we say a path between $u$ and $w$ in $G$ is \emph{blocked} by a set $S \subseteq V \setminus \{u,w\}$ if there exists a node $z$ such that either:
\begin{itemize}
  \item $z \in S$ and $z$ is a non-collider;
  \item $z$ is a collider but neither $z$ nor any of its descendants is in $S$, $\de[G]{z} \cap S = \emptyset$.
\end{itemize}
A path that is not blocked by $S$ is \emph{active given} $S$.
Disjoint sets $A,B \subseteq V$ are \emph{d-separated} by $S$ in $G$, written $A \indepsym_G B \mid S$, if $S$ blocks every path between any node of $A$ and any node of $B$ \citep{pearl_probabilistic_1988}.

\subsubsection{Graphoids.}

A \emph{CI relation} on a finite set $U$ is a set of triples $(A,B \mid C)$ of pairwise disjoint subsets of $U$, written $A \indepsym B \mid C$.
Both $\indepsym_P$ (with $U = V$) and the pooled block relation $\indepsym_{\cE}$ of \autoref{sec:score} (with $U = \Pi^{\cE}$) are CI relations.

\begin{definition}[Semi-graphoid, graphoid, compositional graphoid]
  \label{def:graphoid}
  A CI relation $\indepsym$ on $U$ is a \emph{semi-graphoid} if for all pairwise disjoint $A,B,C,D \subseteq U$:
  \begin{enumerate}[label=(\arabic*), nosep]
    \item \emph{Symmetry:} $A \indepsym B \mid C \Rightarrow B \indepsym A \mid C$;
    \item \emph{Decomposition:} $A \indepsym B \cup D \mid C \Rightarrow A \indepsym B \mid C$;
    \item \emph{Weak union:} $A \indepsym B \cup D \mid C \Rightarrow A \indepsym B \mid C \cup D$;
    \item \emph{Contraction:} $A \indepsym B \mid C$ and $A \indepsym D \mid B \cup C \Rightarrow A \indepsym B \cup D \mid C$.
  \end{enumerate}
  It is a \emph{graphoid} if it also satisfies
  \begin{enumerate}[label=(\arabic*), nosep, start=5]
    \item \emph{Intersection:} $A \indepsym B \mid C \cup D$ and $A \indepsym D \mid B \cup C \Rightarrow A \indepsym B \cup D \mid C$,
  \end{enumerate}
  and a \emph{compositional graphoid} if it also satisfies
  \begin{enumerate}[label=(\arabic*), nosep, start=6]
    \item \emph{Composition:} $A \indepsym B \mid C$ and $A \indepsym D \mid C \Rightarrow A \indepsym B \cup D \mid C$.
  \end{enumerate}
\end{definition}

The semi-graphoid axioms hold for the CI relation of any distribution, intersection holds under a strictly positive density, and the CI relation of a Gaussian distribution with positive definite covariance is a compositional graphoid \citep{lauritzen_graphical_1996}.
In this work we assume the data follow a Gaussian distribution, hence $\indepsym_P$ will always be regarded as a graphoid.

\subsubsection{Markov boundary and restricted Markov boundary.}

We now show the results needed for existence and uniqueness of the Markov boundary and its restricted version.

\begin{theorem}[Existence and uniqueness of the Markov boundary]
  \label{thm:markov-boundary-unique}
  If $\indepsym_P$ is a graphoid, then for every $v \in V$ the Markov boundary $\B(v)$ exists and is unique.
\end{theorem}

The proof combines weak union and intersection applied to the symmetric difference of two candidate boundaries.
Strict positivity of $P$ delivers the intersection axiom and hence the hypothesis of the theorem, so $\B(v)$ exists and is unique in the regime assumed throughout this work.
For the full argument see \citet{verma_causal_1988}.

In practice, we restrict the search procedures for the Markov boundary from all of $V \setminus \{v\}$ to a candidate pool $Z$ fixed in advance, typically by a topological order.
The corresponding refinement is given in \autoref{def:rmb}.
In this case, existence and uniqueness extend without modification: the graphoid axioms are closed under restriction to triples $(A, B \mid C)$ with $A \cup B \cup C \subseteq Z \cup \{v\}$, so \autoref{thm:markov-boundary-unique} applied to the restricted relation, with $Z \cup \{v\}$ in place of $V$, gives existence and uniqueness of $\B(v, Z)$ for every $Z \subseteq V \setminus \{v\}$.

What \autoref{prop:collapse} requires of the pool is only that it contains all parents and no descendants, $\pa[G]{v} \subseteq Z \subseteq \mathrm{nd}_G(v)$.
The prefix $\tau_{\prec v}$ of any topological order $\tau$ of $G$ (\autoref{def:topological-order}) is the standard example.
In \textsc{coarse} the pool is instead the parent pool $\mathrm{pa}^\star(\pi)$ of \autoref{def:support}, which satisfies the same property by \autoref{thm:order}(ii), so the restricted boundary again collapses to the parents.

\begin{proof}[Proof of \autoref{prop:collapse}]
  By assumption, $\pa[G]{v} \subseteq Z \subseteq \mathrm{nd}_G(v)$.
  The local Markov property yields $X_v \indepsym_P X_{\nd[G]{v} \setminus \pa[G]{v}} \mid X_{\pa[G]{v}}$, and decomposition restricts this to $X_v \indepsym_P X_{Z \setminus \pa[G]{v}} \mid X_{\pa[G]{v}}$, so $\pa[G]{v}$ satisfies (i) of \autoref{def:rmb}.
  For minimality, suppose $S \subsetneq \pa[G]{v}$ satisfies (i); decomposition gives $X_v \indepsym_P X_u \mid X_S$ for every $u \in \pa[G]{v} \setminus S$, but faithfulness forces $u$ and $v$ to be d-connected since they are adjacent in $G$, a contradiction.
  Uniqueness then gives $B(v, Z) = \pa[G]{v}$.
\end{proof}

\subsection{Grow-shrink algorithm}
\label{app:gs}

\autoref{alg:gs} is the non-greedy, score-based, restricted grow-shrink of \citet{wienobst_embracing_2026}, stated for a generic target $v$ and pool $Z$.
In \coarse, the target is a block $\pi_j$, the pool is $\pastar{\pi_j}$, the candidates are blocks, and the local score is the pooled BIC of \autoref{def:bic}.

\begin{algorithm}[h]
  \caption{Non-greedy, score-based, restricted grow-shrink}\label{alg:gs}
  \begin{algorithmic}[1]
    \Require target $v$, pool $Z$, local score $s(v,\cdot\,;\cD)$
    \State $S \gets \emptyset$
    \Repeat \Comment{grow}
    \State $S_{\mathrm{old}} \gets S$
    \For{$w \in Z \setminus S$}
    \If{$s(v,S \cup \{w\};\cD) > s(v,S;\cD)$} $S \gets S \cup \{w\}$ \EndIf
    \EndFor
    \Until{$S = S_{\mathrm{old}}$}
    \Repeat \Comment{shrink}
    \State $S_{\mathrm{old}} \gets S$
    \For{$w \in S$}
    \If{$s(v,S \setminus \{w\};\cD) > s(v,S;\cD)$} $S \gets S \setminus \{w\}$ \EndIf
    \EndFor
    \Until{$S = S_{\mathrm{old}}$}
    \State \Return $S$
  \end{algorithmic}
\end{algorithm}

\begin{lemma}[Correctness of grow-shrink]
  \label{lem:gs}
  Let $\indepsym$ be a compositional graphoid on $U$, $v \in U$ and $Z \subseteq U \setminus \{v\}$.
  \begin{enumerate}[label=(\roman*), nosep]
    \item \emph{(Oracle.)}
          If \autoref{alg:gs} adds $w$ iff $v \not\indepsym w \mid S$ and removes $w$ iff $v \indepsym w \mid S \setminus \{w\}$, it returns the Markov boundary of $v$ relative to $Z$ after $\cO(|Z|^2)$ queries \citep{margaritis_learning_2003}.
    \item \emph{(Score.)}
          If $s(v,S \cup \{w\};\cD) - s(v,S;\cD)$ diverges in probability to $+\infty$ when $v \not\indepsym w \mid S$ and to $-\infty$ when $v \indepsym w \mid S$, for all $S \subseteq Z$ and $w \in Z \setminus S$, then \autoref{alg:gs} returns the same set with probability tending to one.
  \end{enumerate}
\end{lemma}
\begin{proof}
  (i) Each grow sweep that changes $S$ adds an element, so the grow phase stops after at most $|Z| + 1$ sweeps of at most $|Z|$ queries, and the same holds for the shrink phase.
  At the end of the grow phase, $v \indepsym w \mid S$ for every $w \in Z \setminus S$, so composition gives $v \indepsym Z \setminus S \mid S$.
  The shrink phase preserves this property: if $v \indepsym Z \setminus S \mid S$ and $v \indepsym w \mid S \setminus \{w\}$, contraction gives $v \indepsym (Z \setminus S) \cup \{w\} \mid S \setminus \{w\}$.
  At termination, $v \not\indepsym w \mid S \setminus \{w\}$ for every $w \in S$.
  Suppose some $S' \subsetneq S$ also satisfied $v \indepsym Z \setminus S' \mid S'$, and take $w \in S \setminus S'$.
  Decomposition gives $v \indepsym S \setminus S' \mid S'$ and weak union gives $v \indepsym w \mid S \setminus \{w\}$, a contradiction.
  Hence $S$ is minimal, and \autoref{thm:markov-boundary-unique} identifies it with the Markov boundary.
  (ii) A removal comparison is the negative of an addition comparison from $S \setminus \{w\}$, so every comparison the algorithm can make has the sign of a CI statement in the limit.
  There are finitely many pairs $(S,w)$, so with probability tending to one all of them have the correct sign simultaneously, and on this event the run coincides with the oracle run of (i).
\end{proof}

\subsection{Coarsening theory}

\subsubsection{The lattice of valid coarsenings}
\label{app:lattice}

\citet[Theorem~3]{madaleno_coarsening_2026-1} show that the valid coarsenings of a DAG form a lattice under refinement.
The theorem statement (and its proof) is correct except for its last phrase (and sentence), which claims that because the set of valid coarsenings is itself a lattice, it is a sublattice of the partition refinement lattice.
\autoref{prop:meet-subsemilattice} is the corrected version of this claim, in case it is of independent mathematical interest.
All results that build on the theorem are unaffected: they use only the lattice structure of the valid coarsenings, which the proof does establish, and not the false sublattice claim.

\begin{proposition}
  \label{prop:meet-subsemilattice}
  Let $G = (V,E)$ be a DAG, and let $\cL(G)$ be the set of valid coarsening partitions of $V$, ordered by refinement ($\Pi \preceq \Pi'$ if $\Pi$ refines $\Pi'$).
  Then $\cL(G)$ is a meet-subsemilattice of the partition refinement lattice $\Pi_V$ \citep[Definition~1 and Remark~2]{caspard2003lattices}.
\end{proposition}

\begin{proof}
  The partition of $V$ into the non-empty intersections of blocks of $\Pi$ and $\Pi'$ induces an acyclic coarsening whenever $\Pi$ and $\Pi'$ do, as shown in the proof of \citet[Theorem~3]{madaleno_coarsening_2026-1}, and being the greatest partition refining both, it is their meet in $\cL(G)$.
  Since $\Pi_V$ is finite and $\{V\} \in \cL(G)$, $\cL(G)$ is a lattice \citep[Proposition~3.3.1]{stanley2011enumerative}, with join $\Pi \vee_{\cL} \Pi'$ the finest valid common coarsening of $\Pi$ and $\Pi'$.
\end{proof}

The last line of the above proof is why the coarsening lattice is a meet-subsemilattice rather than a sublattice of the partition refinement lattice: the join of two elements of the coarsening lattice may be coarser than their corresponding (ambient) join in the partition refinement lattice, as shown in the following example.

\begin{example}
  \label{ex:join}
  Let $G$ be the DAG $1 \rightdir 2 \rightdir 3$ with node $4$ isolated, and consider the partitions $\Pi_1 = \{14 \mid 2 \mid 3\}$ and $\Pi_2 = \{1 \mid 2 \mid 34\}$ of $V = [4]$.
  Both are valid: the coarsenings they induce are the directed paths $\{14\} \rightdir \{2\} \rightdir \{3\}$ and $\{1\} \rightdir \{2\} \rightdir \{34\}$.
  The ambient join merges the overlapping blocks $14$ and $34$, giving $\Pi_1 \vee \Pi_2 = \{134 \mid 2\}$, which is not valid: the fine edges $1 \rightdir 2$ and $2 \rightdir 3$ induce the coarse edges $\{134\} \rightdir \{2\}$ and $\{2\} \rightdir \{134\}$, a directed cycle.
  The common upper bounds of $\Pi_1$ and $\Pi_2$ in $\Pi_V$ are $\{134 \mid 2\}$ and the top $\{1234\}$, so the only valid one is the top, giving $\Pi_1 \vee_{\cL} \Pi_2 = \{1234\}$, strictly coarser than $\Pi_1 \vee \Pi_2$.
\end{example}

\subsubsection{
  Relation to the interventional coarsening of \citet{madaleno_coarsening_2026-1}
}
\label{app:multi-target}

\citet[Def.~8]{madaleno_coarsening_2026-1} define the interventional coarsening $G^\cI$ through intervened ancestors, while \autoref{def:icoarsening} defines $G^\cE$ through signatures.
Here we show that the two coincide for single-target interventions, and that with multi-target interventions $G^\cE$ can be strictly coarser and $G^\cI$ is not identifiable from marginal shifts.

\begin{proposition}[The signature does not determine $\Pi^\cI$ under multi-target interventions]
  \label{prop:interv-nonidentifiable}
  There exist a DAG $G$ and environments $\cE$ with a multi-target intervention for which $\Pi^\cE$ is strictly coarser than $\Pi^\cI$.
  Consequently the signature-based \textsc{Refine} oracle of \citet[Thm.~10]{madaleno_coarsening_2026-1} does not recover $G^\cI$.
\end{proposition}
\begin{proof}
  Take $G : 1 \to 2,\ 1 \to 3$, $\cE = \{0,1\}$ and $I_1 = \{1,2\}$.
  The intervened-ancestor sets are
  \[
    \cI\text{-}\an[G]{1} = \{1\}, \qquad \cI\text{-}\an[G]{2}
                         = \{1,2\}, \qquad \cI\text{-}\an[G]{3}
                         = \{1\},
  \]
  so $\Pi^\cI = \bigl\{ \{1,3\},\ \{2\} \bigr\}$.
  Since $\de[G]{I_1} = \{1,2,3\}$, \autoref{ass:ident}(iii) gives $\sigma(1) = \sigma(2) = \sigma(3) = \{1\}$, so $\Pi^\cE = \{V\}$, strictly coarser than $\Pi^\cI$.
  In particular $\sigma(2) = \sigma(3)$ while $\cI\text{-}\an[G]{2} \neq \cI\text{-}\an[G]{3}$.
\end{proof}

\begin{proposition}
  \label{prop:refinement}
  Under \autoref{ass:ident}, $\Pi^\cI \preceq \Pi^\cE$, and $G^\cE$ is a valid coarsening of $G^\cI$.
\end{proposition}

\begin{proof}
  Both partitions are equivalence classes of $V$: $v \sim_{\Pi^\cI} w$ iff $\cI\text{-}\an[G]{v} = \cI\text{-}\an[G]{w}$, and $v \sim_{\Pi^\cE} w$ iff $\sigma(v) = \sigma(w)$.
  It suffices to show that
  \[
    v \sim_{\Pi^\cI} w \; \Longrightarrow \; v \sim_{\Pi^\cE} w.
  \]
  Let $v, w \in V$ with $\cI\text{-}\an[G]{v} = \cI\text{-}\an[G]{w}$.
  By \autoref{def:signature} and fixing the correspondence of each environment $e \in \cE \setminus \{0\}$ with a set $I_e$ of intervention targets this implies $\sigma(v) = \sigma(w)$.
  Thus $\cI\text{-}\an[G]{v} = \cI\text{-}\an[G]{w}$ gives $\sigma(v) = \sigma(w)$, i.e., $v \sim_{\Pi^\cE} w$.
\end{proof}

On the other hand, if we assume that each environment is represented by a single-target intervention the two partitions agree.

\begin{assumption}[Single-target environments]
  \label{ass:single-target}
  For every $e \in \cE \setminus \{0\}$, $|I_e| = 1$.
\end{assumption}

\begin{proposition}
  \label{prop:partition-equivalence}
  Under \autoref{ass:ident} and \autoref{ass:single-target}, $\Pi^\cE = \Pi^\cI$.
\end{proposition}

\begin{proof}
  Let $T := \bigcup_{e \in \cE \setminus \{0\}} I_e$ be the set of all intervention targets.
  From \citet[Definition~8]{madaleno_coarsening_2026-1} we have,
  \begin{equation}
    \label{eq:ian-intersection}
    \cI\text{-}\an[G]{v} = \an[G]{v} \cap T.
  \end{equation}
  Single-target gives $I_e = \{t_e\}$ for some $t_e \in T$, hence $T = \{ t_e \mid e \in \cE \setminus \{0\}\}$ and $v \in \de[G]{I_e} \iff t_e \in \an[G]{v}$.
  Together with \autoref{ass:ident}(iii) this yields
  \begin{equation}
    \label{eq:sigma-single}
    \sigma(v) = \bigl\{ e \in \cE \setminus \{0\} \mid t_e \in \an[G]{v} \bigr\}.
  \end{equation}
  Quantifying over environments then coincides with quantifying over targets in $T$, and for every $v, w \in V$,
  \[
    \sigma(v) = \sigma(w) \iff \an[G]{v} \cap T = \an[G]{w} \cap T.
  \]
  By \eqref{eq:ian-intersection} the right-hand side reads $\cI\text{-}\an[G]{v} = \cI\text{-}\an[G]{w}$.
  Since $\Pi^\cE$ and $\Pi^\cI$ partition $V$ into the equivalence classes of equal $\sigma(\cdot)$ and equal $\cI\text{-}\an[G] \cdot$ respectively, this gives $\Pi^\cE = \Pi^\cI$.
\end{proof}

\subsubsection{
  Validity of the interventional coarsening and causal orders from interventions
}
\label{app:order}

\begin{proof}[Proof of \autoref{lem:monotone}]
  We first show that, for every $v \in V$,
  \begin{equation}
    \label{eq:sigma-de}
    \sigma(v) = \bigl\{ e \in \cE \setminus \{0\} \mid v \in \de[G]{I_e} \bigr\}.
  \end{equation}
  The inclusion $\supseteq$ is given by \autoref{ass:ident}(iii).
  For $\subseteq$, let $v \notin \de[G]{I_e}$, so that $\an[G]{v} \cap I_e = \emptyset$.
  Since $\an[G]{v}$ is ancestral, integrating \eqref{eq:int-factorization} over $x_{V \setminus \an[G]{v}}$ gives
  \[
    p^{(e)}\big(x_{\an[G]{v}}\big) = \prod_{u \in \an[G]{v}} p^{(0)}\big(x_u \mid x_{\pa[G]{u}}\big)
                                   = p^{(0)}\big(x_{\an[G]{v}}\big),
  \]
  hence $f^{(e)}(X_v) = f^{(0)}(X_v)$ and $e \notin \sigma(v)$.
  Now let $u \to w$ in $G$.
  Then $\an[G]{u} \subseteq \an[G]{w}$, so $u \in \de[G]{I_e}$ implies $w \in \de[G]{I_e}$ for every $e \in \cE \setminus \{0\}$, and \eqref{eq:sigma-de} gives $\sigma(u) \subseteq \sigma(w)$.
\end{proof}

\begin{proof}[Proof of \autoref{prop:validity}]
  By \autoref{def:coarsening}, only acyclicity of $G^\cE$ needs checking.
  Suppose for contradiction that $G^\cE$ contained a directed cycle $\mu_0 \to \mu_1 \to \cdots \to \mu_k = \mu_0$ with $\mu_{j-1} \neq \mu_j$ for every $j$.
  Each coarse edge $\mu_{j-1} \to \mu_j$ arises from some fine edge $u_j \to v_j$ with $u_j \in \mu_{j-1}$ and $v_j \in \mu_j$, while consecutive fine endpoints share a block, $v_j, u_{j+1} \in \mu_j$ (with $u_{k+1} := u_1$), whence $\sigma(v_j) = \sigma(u_{j+1})$.
  Chaining these equalities with the inclusions $\sigma(u_j) \subseteq \sigma(v_j)$ of \autoref{lem:monotone} around the cycle yields
  \[
    \sigma(u_1) \subseteq \sigma(v_1)
                = \sigma(u_2)
                \subseteq \cdots
                = \sigma(u_k)
                \subseteq \sigma(v_k)
                = \sigma(u_1),
  \]
  so every inclusion is an equality.
  In particular $\sigma(u_1) = \sigma(v_1)$, placing $u_1$ and $v_1$ in one block of $\Pi^\cE$ and contradicting $\mu_0 \neq \mu_1$.
\end{proof}

\begin{proof}[Proof of \autoref{thm:order}]
  (i) A coarse edge $\pi \to \pi'$ comes from a fine edge $u \to w$ with $u \in \pi$ and $w \in \pi'$, so \autoref{lem:monotone} gives $\supp(\pi) = \sigma(u) \subseteq \sigma(w) = \supp(\pi')$.
  Chaining along a directed path from $\pi_a$ to $\pi_b$ gives $\supp(\pi_a) \subseteq \supp(\pi_b)$, and the inclusion is strict since $\pi_a \neq \pi_b$.
  (ii) Parents are ancestors, so the first inclusion follows from (i).
  If some $\pi' \in \pastar{\pi_a}$ were a descendant of $\pi_a$, (i) would give $\pi_a <_{\supp} \pi'$, contradicting $\pi' <_{\supp} \pi_a$.
  If the supports of $\pi_a$ and $\pi_b$ are incomparable, (i) rules out ancestry in both directions, hence also adjacency.
  (iii) Adjacent blocks carry exactly one orientation, and by (i) it runs from the smaller support to the larger.
\end{proof}

\begin{example}[Counterexample for the converse of \autoref{thm:order}(i)]
  \label{ex:nested}
  Let $G$ be $a \to c \leftarrow b$ with targets $I_1 = \{a, b\}$, $I_2 = \{b\}$ and $I_3 = \{c\}$.
  Then $\sigma(a) = \{1\}$, $\sigma(b) = \{1, 2\}$ and $\sigma(c) = \{1, 2, 3\}$, so $\Pi^{\cE}$ is the discrete partition and $G^{\cE} = G$.
  Here $\supp(\{a\}) \subsetneq \supp(\{b\})$, yet $a$ and $b$ are non-adjacent and neither is an ancestor of the other.
\end{example}

\begin{proof}[Proof of \autoref{cor:noorder}]
  (i) Fix $\tau \in \Tau( <_{\supp} )$.
  If $\pi_a \to \pi_b \in G^{\cE}$, then
  $\pi_a \prec_{G^{\cE}} \pi_b$, so $\pi_a <_{\supp} \pi_b$ by
  \autoref{thm:order}(i), and $\pi_a$ precedes $\pi_b$ in $\tau$ since $\tau$
  extends $<_{\supp}$.
  Hence $\tau$ is a topological order of $G^{\cE}$.
  Likewise, every $\pi' \in \pastar{\pi}$ satisfies $\pi' <_{\supp} \pi$ by
  \autoref{def:support}, so $\pi'$ precedes $\pi$ in $\tau$.
  (ii) Under \autoref{ass:ident}, $P^{(e)}$ is Markov and faithful to
  $G^{\cE}$, and $\pa[G^{\cE}]{\pi} \subseteq \pastar{\pi} \subseteq \nd[G^{\cE}]{\pi}$ by \autoref{thm:order}(ii).
  \autoref{prop:collapse},
  read with blocks as nodes, then gives the claim.
\end{proof}

\subsection{The coarse block-Gaussian model}
\label{app:block}

This section gives a short primer on the coarse block-Gaussian model.
It is not exhaustive, but it covers what is needed to follow the method.
For a fuller account we refer the reader to \citet{andersson_normal_1998, drton_maximum_2006}.

\subsubsection{Data and sample covariances}

We observe $n_e$ i.i.d.\ realizations of $X^{(e)}$ in environment $e$, collected columnwise into the centered data matrix $\bX^{(e)} \in \R^{d \times n_e}$.
Lowercase $x^{(e)} \in \R^d$ denotes a generic column and $x^{(e)}_j := (x^{(e)}_v)_{v \in \pi_j} \in \R^{r_j}$ its restriction to block $\pi_j$.
Stacking the realizations of $X^{(e)}_j$ and of $X^{(e)}_{\Pa_j}$ gives the row-submatrices $\bX^{(e)}_j \in \R^{r_j \times n_e}$ and $\bX^{(e)}_{\Pa_j} \in \R^{s_j \times n_e}$, that is, the rows of $\bX^{(e)}$ indexed by $\pi_j$ and by $\bigcup_{\pi_k \in \Pa_j} \pi_k$.
The sample covariance is $S^{(e)} = \tfrac{1}{n_e}\bX^{(e)}(\bX^{(e)})^\top$, with blocks
\[
  S^{(e)}_{j,j}
  := \tfrac{1}{n_e}\bX^{(e)}_j(\bX^{(e)}_j)^\top, \qquad S^{(e)}_{j,\Pa_j}
  := \tfrac{1}{n_e}\bX^{(e)}_j(\bX^{(e)}_{\Pa_j})^\top, \qquad S^{(e)}_{\Pa_j,\Pa_j}
  := \tfrac{1}{n_e}\bX^{(e)}_{\Pa_j}(\bX^{(e)}_{\Pa_j})^\top.
\]
For a candidate $B^{(e)}_j \in \R^{r_j \times s_j}$, the residual sample covariance is
\begin{align}
  \label{eq:residual-cov} S^{(e)}_j(B^{(e)}_j) & := \tfrac{1}{n_e}\bigl(\bX^{(e)}_j - B^{(e)}_j \bX^{(e)}_{\Pa_j}\bigr) \bigl(\bX^{(e)}_j - B^{(e)}_j \bX^{(e)}_{\Pa_j}\bigr)^\top \nonumber \\
                                               & = S^{(e)}_{j,j} - B^{(e)}_j S^{(e)}_{\Pa_j,j} - S^{(e)}_{j,\Pa_j}(B^{(e)}_j)^\top + B^{(e)}_j S^{(e)}_{\Pa_j,\Pa_j}(B^{(e)}_j)^\top .
\end{align}

\subsubsection{Maximum likelihood estimation}

Taking logs of the block-recursive density and summing over the $n_e$ columns yields the decomposable log-likelihood
$\ell^{(e)}(G^\cE, \theta^{(e)}) = \sum_{j=1}^q \ell^{(e)}_j(G^\cE, \theta^{(e)}_j)$, with
\begin{equation}
  \label{eq:loglik-j}
  \ell^{(e)}_j(G^\cE, \theta^{(e)}_j)
  = -\frac{n_e r_j}{2}\log(2\pi)
    - \frac{n_e}{2}\log\!\left| \Sigma^{(e)}_j \right|
    - \frac{n_e}{2}\tr\!\left[ (\Sigma^{(e)}_j)^{-1} S^{(e)}_j(B^{(e)}_j) \right].
\end{equation}
Since the parameter space also factorizes across blocks \citep[Sec.~2]{drton_maximum_2006}, the maximum likelihood estimate (MLE) is obtained by maximizing each $\ell^{(e)}_j$ separately, one block regression at a time.
With no constraints on $B^{(e)}_j$ or $\Sigma^{(e)}_j$, it is the ordinary two-step least-squares estimator
\begin{align}
  \label{eq:mle-B} \widehat{B}^{(e)}_j          & = \bX^{(e)}_j(\bX^{(e)}_{\Pa_j})^\top \left[ \bX^{(e)}_{\Pa_j}(\bX^{(e)}_{\Pa_j})^\top \right]^{-1} = S^{(e)}_{j,\Pa_j}\left( S^{(e)}_{\Pa_j,\Pa_j} \right)^{-1}, \\
  \label{eq:mle-Sigma} \widehat{\Sigma}^{(e)}_j & = S^{(e)}_{j,j} - \widehat{B}^{(e)}_j S^{(e)}_{\Pa_j,j} = S^{(e)}_{j,j} - S^{(e)}_{j,\Pa_j}\left( S^{(e)}_{\Pa_j,\Pa_j} \right)^{-1} S^{(e)}_{\Pa_j,j} .
\end{align}
Substituting \eqref{eq:mle-B} and \eqref{eq:mle-Sigma} into \eqref{eq:loglik-j}, using $S^{(e)}_j(\widehat{B}^{(e)}_j) = \widehat{\Sigma}^{(e)}_j$ so that $\tr[(\widehat{\Sigma}^{(e)}_j)^{-1} S^{(e)}_j(\widehat{B}^{(e)}_j)] = r_j$, and summing over $j$ with $\sum_j r_j = d$, gives the profile log-likelihood
\begin{align}
  \label{eq:profile} \widehat{\ell}^{(e)}(G^\cE) & = -\frac{n_e}{2}\sum_{j=1}^{q} \log\!\left| \widehat{\Sigma}^{(e)}_j \right| - \frac{n_e d}{2}\bigl(1 + \log(2\pi)\bigr) \\
                                                 & = -\sum_{j=1}^{q} \log\!\left| \widehat{\Sigma}^{(e)}_j \right| + \mathrm{const}. \nonumber
\end{align}

\subsubsection{Existence of the MLE}

For a Gaussian model, the MLE exists if and only if the partial covariance admits a positive-definite completion; see \citet{uhler_gaussian_2017} for a detailed treatment.
Let $\mlt(G^\cE)$ denote the \emph{maximum-likelihood threshold}, the smallest sample size at which the Gaussian MLE exists with probability one.
Since the MLE above is computed block by block, existence reduces to a per-block condition.
Write $F_j := \pi_j \cup \Pa_j$ for the \emph{family} of block $j$ and
\[
  S^{(e)}_{F_j} = \begin{pmatrix}
                    S^{(e)}_{j,j}     & S^{(e)}_{j,\Pa_j}     \\[2pt]
                    S^{(e)}_{\Pa_j,j} & S^{(e)}_{\Pa_j,\Pa_j}
                  \end{pmatrix}
                \in \R^{(r_j+s_j) \times (r_j+s_j)}
\]
for the corresponding family covariance.
The estimate $\widehat{\Sigma}^{(e)}_j$ in \eqref{eq:mle-Sigma} is the Schur complement of $S^{(e)}_{\Pa_j,\Pa_j}$ in $S^{(e)}_{F_j}$.
Hence $\widehat{\Sigma}^{(e)}_j \succ 0$ if and only if $S^{(e)}_{F_j} \succ 0$, which holds with probability one if and only if $n_e > |F_j| = r_j + s_j$.
The threshold of the model is therefore the size of its largest family,
\[
  \mlt(G^\cE) \; = \; \max_{j}\,|F_j| \; = \; \max_{j} (r_j + s_j),
\]
and the block estimates, the profile log-likelihood \eqref{eq:profile} and the BIC score exist as soon as $n_e \ge \mlt(G^\cE)$.
Throughout, we assume that in every environment $e \in \cE$ the family covariances $S^{(e)}_{F_j}$ have full rank, so that the MLE exists.

Since $\pi_j$ and $\Pa_j$ are disjoint, $\mlt(G^\cE) \le d$, and typically $\mlt(G^\cE) \ll d$.
The model is therefore well defined even in high-dimensional regimes with $n_e < d$: there the pooled covariance $S^{(e)}$ is singular and the unrestricted Gaussian on all $d$ coordinates has no MLE, yet every block residual covariance $\widehat{\Sigma}^{(e)}_j$ remains positive definite.

\subsubsection{Curved exponential family structure}

\begin{proof}[
  Proof of \autoref{thm:curved}, adapted from \citet{drton_maximum_2006}
]
  The model is the Gaussian AMP chain-graph model whose chain components are the blocks, taken complete \citep{andersson_alternative_2001,drton_maximum_2006}.
  It is a subfamily of the regular exponential family of centered Gaussians on $\R^d$ with positive definite covariance.
  Its parameter space $\Theta = \{(B_j,\Sigma_j)_{j=1}^q : B_j \in \R^{r_j \times s_j},\ \Sigma_j \succ 0\}$ is an open subset of a Euclidean space of dimension $\dim(H)$.
  Let $B(\theta) \in \R^{d \times d}$ carry $B_j$ in the rows of $\pi_j$ and the columns of its parent blocks, and let $\bar\Sigma(\theta)$ be block-diagonal with blocks $\Sigma_1,\dots,\Sigma_q$.
  In a topological order of $H$, $B(\theta)$ is block strictly lower triangular, so $I - B(\theta)$ is invertible, and the joint concentration is
  \[
    \psi(\theta) = \big(I - B(\theta)\big)^\top\bar\Sigma(\theta)^{-1}\big(I - B(\theta)\big).
  \]
  The map $\psi$ is injective: from a covariance $\Sigma$ in its image, the parameters are recovered by the population regressions $B_j = \Sigma_{\pi_j,\Pa_j}\Sigma_{\Pa_j,\Pa_j}^{-1}$ and $\Sigma_j = \Sigma_{\pi_j,\pi_j} - B_j\Sigma_{\Pa_j,\pi_j}$.
  Both $\psi$ and its inverse are rational with non-vanishing denominators on their domains, so $\psi$ is a diffeomorphism onto its image, which is therefore a smooth manifold of dimension $\dim(H)$.
\end{proof}

\subsection{Pooled BIC score}
\label{app:score}

Throughout, $H$ is a coarse DAG on $\Pi^{\cE}$ with parent sets $\Pa_j$.
For a set of blocks $T$, $\widehat\Sigma^{(e)}_j(T)$ denotes the residual covariance \eqref{eq:mle-Sigma} of the regression of $X_j$ on the blocks in $T$, fitted on $\cD^{(e)}$, so that $\widehat\Sigma^{(e)}_j = \widehat\Sigma^{(e)}_j(\Pa_j)$.
The regression parameters $(B^{(e)}_j,\Sigma^{(e)}_j)$ are fitted separately in each environment; only $H$ is shared.
We take $n_e$ large enough for every family covariance to be full rank (\autoref{app:block}), so every term below is finite.

Plugging the profile log-likelihood \eqref{eq:profile} and $\dim(H)$ from \autoref{thm:curved} into \eqref{eq:bic} with $n = n_e$ gives the per-environment score
\begin{equation}
  \label{eq:bic-env}
  \BIC(H,\cD^{(e)}) := 2\hat\ell^{(e)}(H) - \lambda\log(n_e)\dim(H)
                    = \sum_{j=1}^{q} \BIC^{(e)}_j(\pi_j,\Pa_j)
                      - n_e d\big(1 + \log(2\pi)\big),
\end{equation}
with $\BIC^{(e)}_j$ as in \autoref{def:bic}.
The constant does not depend on $H$, so it can be dropped.
Summing over $\cE$ and exchanging the finite sums over blocks and environments gives
\[
  \sum_{e \in \cE}\BIC(H,\cD^{(e)}) = \BIC(H,\cD) + \mathrm{const},
\]
so the pooled BIC of \autoref{def:bic} is the sum of per-environment BICs.
Decomposability is immediate, since each $\BIC^{(e)}_j$ depends on $H$ only through $(\pi_j,\Pa_j)$.
\begin{proposition}[Consistency]
  \label{prop:bic-consistency}
  Let $G^\cE$ be the true coarse DAG over $\Pi^\cE$ and $[G^\cE]$ its Markov equivalence class.
  Assuming each environment share the same DAG, for every coarse DAG $H$ over $\Pi^\cE$,
  \[
  \lim_{n\to\infty} \Pr\bigl(\BIC(G^\cE, \cD_n) > \BIC(H, \cD_n)\bigr) = 1
  \qquad \text{for every } H \notin [G^\cE].
\]
\end{proposition}
\begin{proof}
  The coarse Markov and coarse faithfulness clauses of \autoref{ass:ident} make $G^\cE$ a perfect map of the block-level law $P^{(e)}$ in every environment, hence the inclusion-minimal I-map common to all of them.
  Consequently, any $H$ that is an I-map of every environment yet lies outside $[G^\cE]$ has strictly larger dimension.
  Fix such a $H$ and set
  \[
    \Delta_e := \ell^{(e)}(G^\cE) - \ell^{(e)}(H) \ge 0,
  \]
  the population log-likelihood gap at the respective maximisers, with equality exactly when $H$ is also an I-map of $P^{(e)}$.
  As $n_e \to \infty$, the empirical law converges to $P^{(e)}$, giving $\widehat\ell^{(e)}(G^\cE) - \widehat\ell^{(e)}(H) \approx n_e \Delta_e$.
  Using that $\dim(H)$ is common to all environments, this reads,
  \[
    \BIC(G^\cE, \cD) - \BIC(H, \cD)
    \approx 2\sum_{e \in \cE} n_e \Delta_e
            + \lambda\bigl(\dim(H) - \dim(G^\cE)\bigr)\sum_{e \in \cE}\log n_e.
  \]
  Two cases cover $H \notin [G^\cE]$.
  If $\Delta_{e^\star} > 0$ for some $e^\star$, then $H$ encodes an independence that $P^{(e^\star)}$ violates, the term $2 n_{e^\star}\Delta_{e^\star}$ is linear in $n_{e^\star}$, and it dominates the logarithmic penalty, so the difference tends to $+\infty$.
  If instead $\Delta_e = 0$ for every $e$, then $H$ is an I-map of each environment and, being outside $[G^\cE]$, satisfies $\dim(H) > \dim(G^\cE)$; the likelihood gap vanishes while the penalty term diverges, again to $+\infty$.
  In both cases the stated probability tends to one, so the pooled score prefers the true graph in the large-sample limit.
\end{proof}

\begin{proposition}[Score equivalence]
  \label{prop:score-equiv}
  If $H$ and $H'$ are Markov equivalent coarse DAGs on $\Pi^{\cE}$, then $\BIC(H,\cD) = \BIC(H',\cD)$.
\end{proposition}
\begin{proof}
  Markov equivalent DAGs impose the same block-level CI statements, so the models of \autoref{def:blockmodel} on $H$ and $H'$ coincide and attain the same maximized likelihood in every environment.
  Moreover, $\dim(H) = \sum_{\pi_k \to \pi_j \in H} r_jr_k + \sum_j r_j(r_j + 1)/2$ depends only on the skeleton, which $H$ and $H'$ share.
  Hence each summand of \eqref{eq:bic-env} agrees, and so does their sum over $\cE$.
\end{proof}

For a set of blocks $S$, let $\Sigma^{(e)}_{\pi_j \mid S}$ be the population residual covariance of $X^{(e)}_j$ given $X^{(e)}_S$, and $\Sigma^{(e)}_{\pi_j,\pi_k \mid S}$ the partial cross-covariance of $X^{(e)}_j$ and $X^{(e)}_k$ given $X^{(e)}_S$.
For Gaussian $P^{(e)}$, $\indep[P^{(e)}]{X_{\pi_j}}{X_{\pi_k}}{X_S}$ holds if and only if $\Sigma^{(e)}_{\pi_j,\pi_k \mid S} = 0$ \citep{lauritzen_graphical_1996}.

\begin{lemma}[Per-environment local consistency]
  \label{lem:env-localcons}
  Fix $e \in \cE$, a block $\pi_j$, a set of blocks $S \not \ni \pi_j$ and a block $\pi_k \notin S \cup \{\pi_j\}$, and let
  $\Delta^{(e)} := \BIC^{(e)}_j(\pi_j, S \cup \{\pi_k\}) - \BIC^{(e)}_j(\pi_j, S)$.
  As $n_e \to \infty$, $\Delta^{(e)} \xrightarrow{p} + \infty$ if $\Sigma^{(e)}_{\pi_j,\pi_k \mid S} \neq 0$, and $\Delta^{(e)} \xrightarrow{p} - \infty$ otherwise.
\end{lemma}
\begin{proof}
  Write $S' := S \cup \{\pi_k\}$.
  Both parent sets share the $r_j(r_j + 1)/2$ covariance parameters, and the regression block grows by $r_jr_k$ entries, so
  \[
    \Delta^{(e)} = n_e\log\frac{\big|\widehat\Sigma^{(e)}_j(S)\big|}{\big|\widehat\Sigma^{(e)}_j(S')\big|}
                   - \lambda\log(n_e)\,r_jr_k.
  \]
  Since $S^{(e)} \xrightarrow{p} \Sigma^{(e)}$, the continuous mapping theorem gives $\widehat\Sigma^{(e)}_j(T) \xrightarrow{p} \Sigma^{(e)}_{\pi_j \mid T}$ for $T \in \{S,S'\}$.
  By the Schur complement formula,
  \[
    \Sigma^{(e)}_{\pi_j \mid S'}
    = \Sigma^{(e)}_{\pi_j \mid S}
      - \Sigma^{(e)}_{\pi_j,\pi_k \mid S}\big(\Sigma^{(e)}_{\pi_k \mid S}\big)^{-1}\Sigma^{(e)}_{\pi_k,\pi_j \mid S},
  \]
  so the population log-ratio is non-negative and vanishes exactly when $\Sigma^{(e)}_{\pi_j,\pi_k \mid S} = 0$.
  If it is positive, the likelihood term grows like $n_e$ and dominates the $\cO(\log n_e)$ penalty.
  Otherwise, Wilks' theorem for nested Gaussian regressions \citep{vaart_asymptotic_1998} gives $n_e\log\big(|\widehat\Sigma^{(e)}_j(S)|/|\widehat\Sigma^{(e)}_j(S')|\big) \xrightarrow{d} \chi^2_{r_jr_k}$, so the likelihood term is $\cO_p(1)$ while the penalty diverges.
\end{proof}

\begin{proof}[Proof of \autoref{prop:localcons}]
  By decomposability, the increment of $\BIC_j$ is $\sum_{e \in \cE}\Delta^{(e)}$, with $n_e = \gamma_e n$.
  If $\nindep[\cE]{\pi_j}{\pi_k}{S}$, at least one environment has $\Sigma^{(e)}_{\pi_j,\pi_k \mid S} \neq 0$ and, by \autoref{lem:env-localcons}, contributes a term of order $n$, which outgrows the finitely many terms of order $\log n$ from the others.
  If $\indep[\cE]{\pi_j}{\pi_k}{S}$, every $\Delta^{(e)}$ diverges to $-\infty$, and so does their finite sum.
\end{proof}

\subsection{Estimation and guarantees}

\subsubsection{Partition estimation and refinement}

\begin{algorithm}[H]
  \caption{
    Compute interventional descendants matrix; adapted from \citealp{madaleno_coarsening_2026-1}
  }\label{alg:refineaux}
  \begin{algorithmic}[1]
    \Require data $\cD = \{X^{(e)}\}_{e \in \cE}$, threshold $\alpha$, $\mathrm{Test}$
    \Ensure interventional descendants matrix $M \in \{0,1\}^{|V| \times (|\cE|-1)}$
    \For{$(v,e) \in V \times (\cE \setminus \{0\})$} \Comment{in parallel}
    \State $p_{v,e} \gets \mathrm{Test} \big(X_v^{(e)}, X_v^{(0)}\big)$ \Comment{Welch-t, LRT, etc.}
    \State $M_{v,e} \gets \1[p_{v,e} < \alpha]$
    \EndFor
    \State \Return $M$
  \end{algorithmic}
\end{algorithm}

\begin{algorithm}[H]
  \caption{
    Refine via interventional descendant patterns; adapted from \citealp{madaleno_coarsening_2026-1}
  }\label{alg:refinetest}
  \begin{algorithmic}[1]
    \Require partition $\Pi$, interventional descendant matrix $M$
    \Ensure pair $(\pi^*, \{\pi_a, \pi_b\})$ or $(\varnothing, \varnothing)$
    \For{$\pi \in \Pi$ with $|\pi| > 1$}
    \State pick $u \in \pi$
    \State $\pi_a \gets \{v \in \pi \mid M_{v, \cdot} = M_{u, \cdot}\}$
    \State $\pi_b \gets \pi \setminus \pi_a$
    \If{$\pi_b \neq \varnothing$} \Return $(\pi, \{\pi_a, \pi_b\})$ \EndIf
    \EndFor
    \State \Return $(\varnothing, \varnothing)$
  \end{algorithmic}
\end{algorithm}

\subsubsection{Full COARSE procedure}

Each block is searched over its estimated parent pool $\hatpastar{\pi} := \{\pi' \in \widehat\Pi^{\cE} : \widehat{\supp}(\pi') \subsetneq \widehat{\supp}(\pi)\}$.
The pools are read off the supports alone, so no actual linear extension of $<_{\supp}$ needs to be computed.

\begin{algorithm}[H]
  \caption{\coarse}\label{alg:coarse}
  \begin{algorithmic}[1]
    \Require data $\cD = \{X^{(e)}\}_{e \in \cE}$, threshold $\alpha$, $\mathrm{Test}$, penalty $\lambda$
    \Ensure estimated coarse DAG $\hat G$
    \State $\widehat M \gets \textsc{InterventionalDescendants}(\cD, \alpha, \mathrm{Test})$ \Comment{\autoref{alg:refineaux}}
    \State $\widehat\Pi^{\cE} \gets \{V\}$
    \Repeat
    \State $(\pi^*, \{\pi_a, \pi_b\}) \gets \textsc{Refine}(\widehat\Pi^{\cE}, \widehat M)$ \Comment{\autoref{alg:refinetest}}
    \If{$\pi^* \neq \varnothing$} $\widehat\Pi^{\cE} \gets (\widehat\Pi^{\cE} \setminus \{\pi^*\}) \cup \{\pi_a, \pi_b\}$ \EndIf
    \Until{$\pi^* = \varnothing$}
    \For{each $\pi \in \widehat\Pi^{\cE}$}
    \State $\widehat{\supp}(\pi) \gets \{e \in \cE \setminus \{0\} : \widehat M_{v,e} = 1\}$ for any $v \in \pi$
    \EndFor
    \For{each $\pi_j \in \widehat\Pi^{\cE}$} \Comment{any order}
    \State $\hatpastar{\pi_j} \gets \{\pi_k \in \widehat\Pi^{\cE} : \widehat{\supp}(\pi_k) \subsetneq \widehat{\supp}(\pi_j)\}$
    \State $\widehat{\Pa}_j \gets \textsc{GrowShrink}\big(\pi_j, \hatpastar{\pi_j}, \BIC_j\big)$ \Comment{\autoref{alg:gs}}
    \EndFor
    \State \Return $\hat G = \bigl(\widehat\Pi^{\cE}, \{\pi_k \to \pi_j : \pi_k \in \widehat{\Pa}_j\}\bigr)$
  \end{algorithmic}
\end{algorithm}

\subsubsection{Consistency of \coarse}
\label{app:consistency}

At block level, the Markov boundary of $\pi_j$ relative to a set of blocks $Z$, written $\B(\pi_j,Z)$, is the inclusion-minimal $S \subseteq Z$ with $\indep[\cE]{\pi_j}{Z \setminus S}{S}$.

\begin{lemma}[Block-level compositional graphoid]
  \label{lem:graphoid}
  $\indepsym_{\cE}$ is a compositional graphoid on $\Pi^{\cE}$.
  Hence $\B(\pi_j,Z)$ exists and is unique for every block $\pi_j$ and every $Z \subseteq \Pi^{\cE} \setminus \{\pi_j\}$.
\end{lemma}
\begin{proof}
  Each $P^{(e)}$ is Gaussian with positive definite covariance, so its CI relation is a compositional graphoid on $V$.
  Restricting it to unions of blocks preserves the axioms.
  The intersection over $e \in \cE$ is again a compositional graphoid: whenever the premises of an axiom hold in every $P^{(e)}$, so does its conclusion.
  The second claim follows from \autoref{thm:markov-boundary-unique}.
\end{proof}

\begin{proof}[Proof of \autoref{thm:consistency}]
  Under oracle tests, \autoref{alg:coarse} runs with $\Pi^{\cE}$ and the pools $\pastar{\cdot}$.
  Fix a block $\pi_j$ and let $Z_j := \pastar{\pi_j}$.
  By \autoref{cor:noorder}(ii), $\Pa_j$ is the Markov boundary of $\pi_j$ relative to $Z_j$ in every $P^{(e)}$.
  Hence $\indep[\cE]{\pi_j}{Z_j \setminus \Pa_j}{\Pa_j}$, while no proper subset of $\Pa_j$ has this property, so $\Pa_j$ is the Markov boundary of $\pi_j$ relative to $Z_j$ under $\indepsym_{\cE}$, unique since $\indepsym_{\cE}$ is a compositional graphoid (\autoref{lem:graphoid}).
  By \autoref{prop:localcons}, every increment of $\BIC_j$ has, in the limit, the sign of the matching $\indepsym_{\cE}$ statement, so the score-based \textsc{GrowShrink} returns $\Pa_j$ with probability tending to one (\autoref{lem:gs}).
  A union bound over the $q$ blocks gives $\PP(\widehat\Pa_j = \Pa_j\ \forall j) \to 1$, and on this event the edges of $\hat G$ are exactly $\pi_k \to \pi_j$ with $\pi_k \in \pa[G^{\cE}]{\pi_j}$, that is $\hat G = G^{\cE}$.
  Finally, each \textsc{GrowShrink} call depends only on $\pi_j$, $\pastar{\pi_j}$ and $\BIC_j$, and not on the order in which blocks are processed.
\end{proof}

\subsection{Worst-case complexity}
\label{app:complexity}

\begin{proof}[Proof of \autoref{thm:complexity}]
  (1) For each $e \neq 0$, the $d$ Welch statistics need the means and variances of $X^{(e)}$
  and $X^{(0)}$, at cost $\cO(d(n_e + n_0))$, so $\cO(dn)$ in total.
  (2) Each refinement pass compares at most $d$ rows of $\hat M$ to a representative, and there
  are at most $q - 1$ splits: $\cO(qd|\cE|)$.
  (3) Support inclusion for every ordered pair of blocks costs $\cO(q^2|\cE|)$.
  (4) Centering costs $\cO(dn)$ and the $|\cE|$ sample covariances $S^{(e)}$ cost $\cO(d^2n)$;
  they are computed once and cached.
  (5) Each grow sweep that does not end the grow phase adds a block, so the phase makes at most
  $|Z| + 1$ sweeps of at most $|Z|$ probes, whatever the scores; the shrink phase is symmetric.
  With $Z = \hatpastar{\pi_j}$ this gives $\cO(q^2)$ probes per block.
  Each probe slices $S^{(e)}_{F_j}$ in every environment and computes a Cholesky factorization,
  a triangular solve, an $r_j \times r_j$ Schur complement and a log-determinant, for
  $\cO((r_j + s_j)^3) = \cO(d^3)$ per environment.
  Over the $q$ blocks this gives $\cO(|\cE|\,q^3d^3)$.
  Since $q \le d$ and $|\cE| \le n$, phases (1)--(3) are dominated by $\cO(d^2n)$,
  and the total is
  \[
    \cO(dn) + \cO(qd|\cE|) + \cO(q^2|\cE|) + \cO(d^2n) + \cO(|\cE|\,q^3d^3) = \cO\big(d^2 n + |\cE|\,q^{3}d^{3}\big). \qedhere
  \]
\end{proof}

\section{Additional experiments}
\label{app:experiments}

We now present additional experiments that test our theoretical results and assess robustness to violations of our assumptions.
Each run used a single CPU core on an Apple M3 Max (36\,GB RAM).

\subsection{Experimental details}
\label{app:exp-details}

\paragraph{Synthetic data.}

ER graphs are sampled with \texttt{sempler} with expected degree $\rho(d - 1)$.
Scale-free (SF) graphs are Barab\'asi--Albert graphs in which each new node attaches to $\max\{1, \mathrm{round}(\rho(d - 1)/2)\}$ existing nodes, matching the expected degree of the ER graph with the same $\rho$.
Edge weights have magnitude uniform in $[0.5, 2]$ and independent random signs.
Observational noise means are uniform in $[-2, 2]$ and variances uniform in $[0.5, 2]$.
The model is standardized in causal order to unit observational variance \citep{ormaniec_standardizing_2025} while interventional environments are not re-standardized.
Each interventional environment shifts its targets by independent noise with mean $2$ and variance $1$.
With $k = 1$ target per environment, targets are drawn without replacement, with $k > 1$ (\autoref{app:exp-multitarget}) each environment draws $k$ targets and targets may repeat across environments.
For non-Gaussian families (\autoref{app:exp-noise}), the noise of each variable is rescaled to the mean and variance of its Gaussian counterpart.

\paragraph{Pipeline, cross-validated and oracle runs.}

In the scalability (\autoref{sec:exp-scal}) and density (\autoref{app:exp-density}) studies, we run the full pipeline of \autoref{alg:coarse} with a fixed $\alpha = 10^{-4}$.
All runtimes are wall-clock times of one fit on one CPU core (Apple M3 Max, 36\,GB RAM).
While in the multiple testing (\autoref{app:exp-threshold}) and multi-target (\autoref{app:exp-multitarget}) experiments, we choose $\alpha$ from $\{10^{-4}, 10^{-3}, 10^{-2}, 0.05, 0.1\}$ by 10-fold cross-validation, maximizing the held-out Gaussian log-likelihood of the fitted coarse model.
In the remaining experiments, namely the comparison with RePaRe (\autoref{sec:exp-compare}) and the sensitivity analyses for the penalty multiplier $\lambda$ (\autoref{app:exp-lambda}), noise (\autoref{app:exp-noise}), and dimensionality (\autoref{app:exp-highdim}), all methods are given the true partition, so no threshold is needed.

\paragraph{Causal chamber.}

We use the \texttt{lt\_interventions\_standard\_v1} dataset of the light tunnel: the actuators R, G, B, the current $\tilde C$, the infrared and visible sensors $\tilde I_{1:3}$, $\tilde V_{1:3}$, the polarizer actuators $\theta_{1,2}$ and angle sensors $\tilde\theta_{1,2}$, and the six lens readings $L_{11}, \dots, L_{32}$.
The observational dataset has $10^4$ samples and each interventional dataset $10^3$.
In the \emph{ungrouped} mode each of the five interventional datasets is one environment ($m = 5$); the true coarse DAG has $8$ blocks and $10$ edges, which expand to $39$ fine-grained edges.
In the \emph{grouped} mode the datasets are pooled into two environments, one for the colour actuators and one for the polarizers ($m = 2$); the true coarse DAG has $4$ blocks and $4$ edges.
Since physical interventions may change the spread of a variable as well as its mean, on these data \coarse and RePaRe build $\hat M$ with a Gaussian likelihood-ratio test for a change in mean or variance.

Hyperparameters are selected with different approaches.
For \coarse we select $\alpha$ by $10$-fold cross-validation with $\lambda = 1$ fixed.
RePaRe searches its partition and edge threshold over the same grid but uses the GnIES score on the expanded DAG, as in \citet{madaleno_coarsening_2026-1}.
Importantly, this means that RePaRE actually consumes the intervention target at selection time, while \coarse does not.
GIES and GnIES run with their default BIC score, and UT-IGSP selects its two test levels by its own score.
Oracle rows select the grid cell with the highest ARI, then the highest F-score (F-score alone for UT-IGSP).

\paragraph{Negative controls \citep{petersen_are_2025}}

For a coarse method, a null draw is a random DAG on the estimated blocks with the same number of edges, expanded to the variable level.
For a fine-grained method, a null draw is a random DAG on the $20$ variables with the same number of adjacencies, mapped to the type of equivalence class the method returns.
Each draw is scored at both levels; we report the mean and the $2.5\%$ and $97.5\%$ quantiles over $1000$ draws.
For adjacencies we also use the exact null, under which the number of true adjacencies among the estimated ones is hypergeometric.

\subsection{Finite-sample variants}
\label{app:heuristics}

In finite samples a fixed $\alpha$ over $dm$ tests can split or merge blocks, invalidating the procedure, and large blocks inflate the model parameters, making the estimation procedure more complex and data hungry.
We address both these issues through three variants of the procedure:

\begin{enumerate}
  \item CV-\coarse: Multiple-testing corrections control false shifts, not partition accuracy, and tend to be conservative as $d$ grows.
        CV-\coarse selects $\alpha$ from a grid by $K$-fold cross-validated held-out Gaussian log-likelihood of the fitted coarse model, without accessing the ground truth.
  \item kPC-\coarse: The block model is unrestricted, so each block costs $r_j s_j + r_j(r_j + 1)/2$ parameters.
        kPC-\coarse replaces each block by its first $k_j := \min\{k, r_j\}$ principal components, estimated in the observational environment and applied to all others, and scores with $k_j$ in place of $r_j$.
        The family size drops from $r_j + s_j$ to $k_j + \sum_{\pi_k \in \Pa_j} k_k$.
  \item avg-\coarse: Here we summarize each block by the mean of its standardized variables.
        In this way the score runs on $q$ scalar Gaussian nodes.
\end{enumerate}

\subsection{Graph density and graph family}
\label{app:exp-density}

\autoref{fig:density-er} and \autoref{fig:density-sf} repeat the study of \autoref{sec:exp-scal} at $d = 10$ for ER and SF graphs, with $\rho \in \{0.2, 0.5, 0.8\}$, $m \in \{2, 5, 8\}$ and $n_e$ from $10^2$ to $10^5$.
On ER graphs, the median F-score reaches $1$ faster than ARI in almost any setting, suggesting that constraint-based partition recovery is the bottleneck of the procedure.
Denser graphs and more environments need more samples: more environments give more blocks to separate, and each adds $d$ tests to $\hat M$.
The hardest setting, $\rho = 0.8$ and $m = 8$, starts from a median ARI of only $0.33$ at $n_e = 10^2$.
SF graphs converge faster: both medians reach $1$ by $n_e = 2 \times 10^3$ in every setting.

\begin{figure*}[t]
  \centering
  \begin{minipage}{0.32\textwidth}
    \centering
    \includegraphics[width=\linewidth]{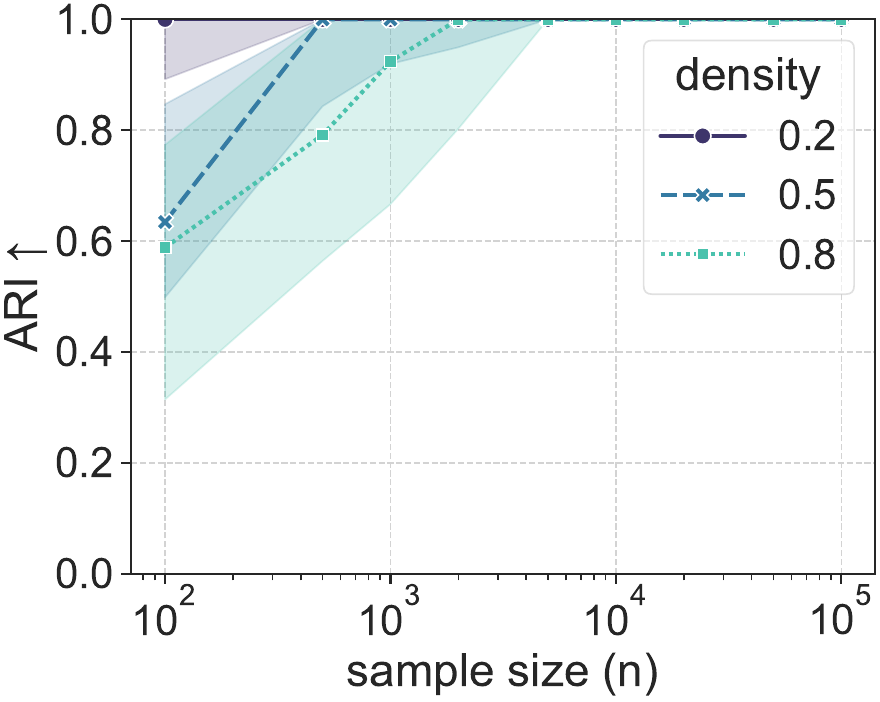}\\
    {\small (a) ARI, $m = 2$}
  \end{minipage}
  \hfill
  \begin{minipage}{0.32\textwidth}
    \centering
    \includegraphics[width=\linewidth]{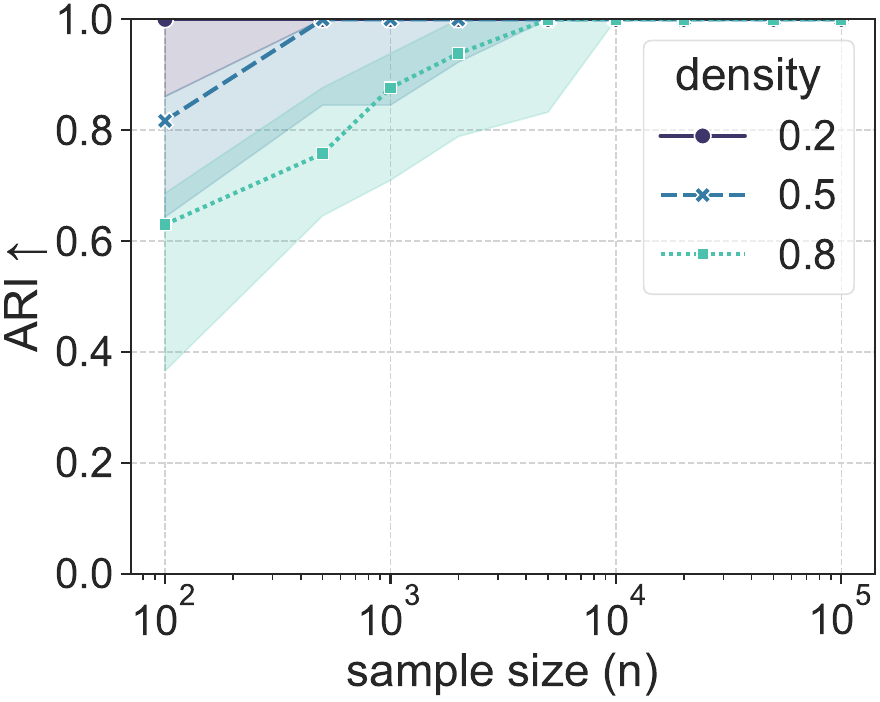}\\
    {\small (b) ARI, $m = 5$}
  \end{minipage}
  \hfill
  \begin{minipage}{0.32\textwidth}
    \centering
    \includegraphics[width=\linewidth]{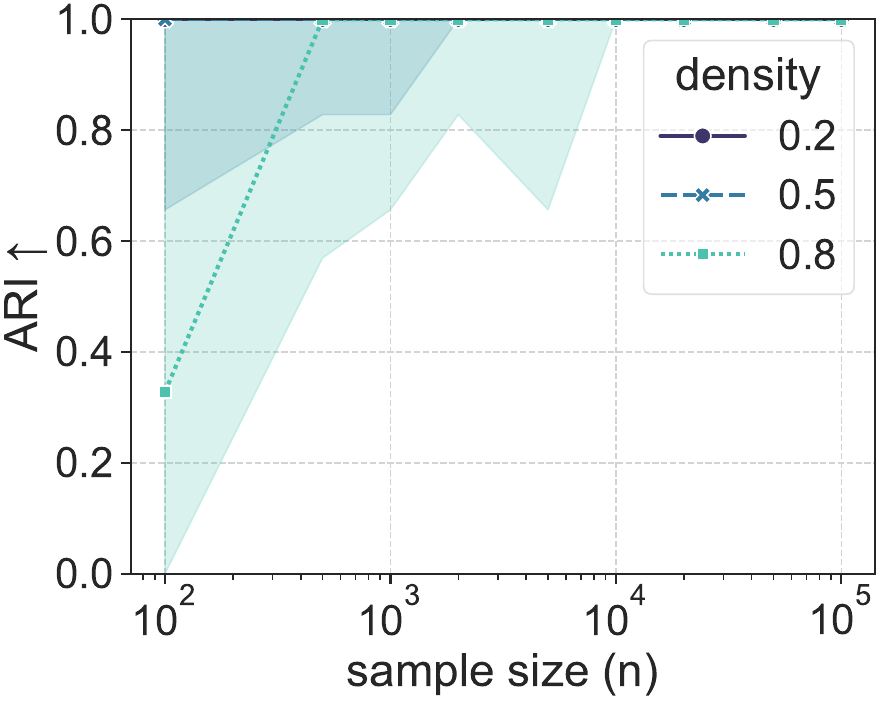}\\
    {\small (c) ARI, $m = 8$}
  \end{minipage}

  \vspace{4pt}
  \begin{minipage}{0.32\textwidth}
    \centering
    \includegraphics[width=\linewidth]{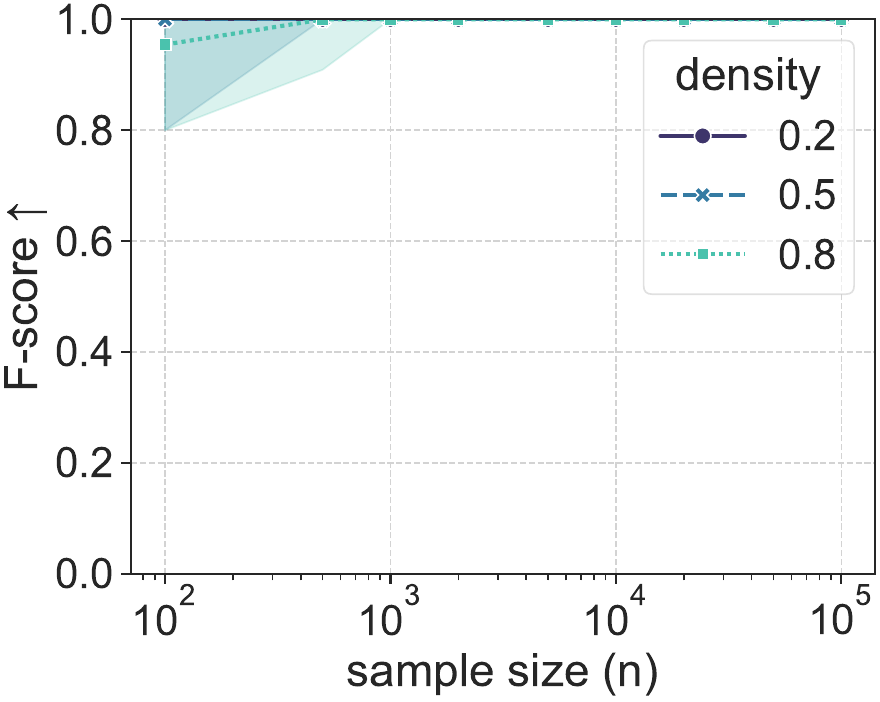}\\
    {\small (d) F-score, $m = 2$}
  \end{minipage}
  \hfill
  \begin{minipage}{0.32\textwidth}
    \centering
    \includegraphics[width=\linewidth]{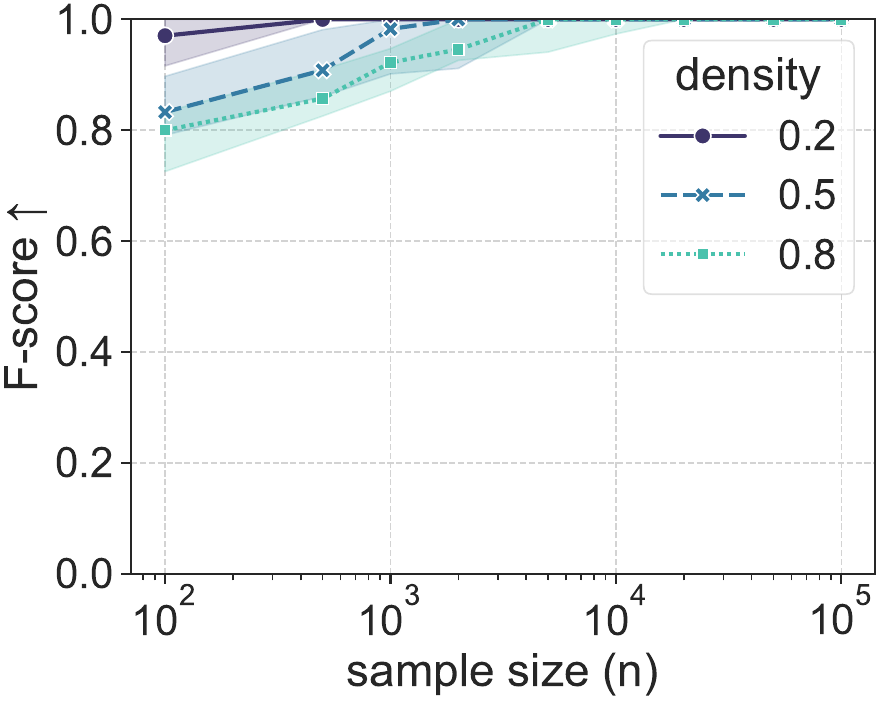}\\
    {\small (e) F-score, $m = 5$}
  \end{minipage}
  \hfill
  \begin{minipage}{0.32\textwidth}
    \centering
    \includegraphics[width=\linewidth]{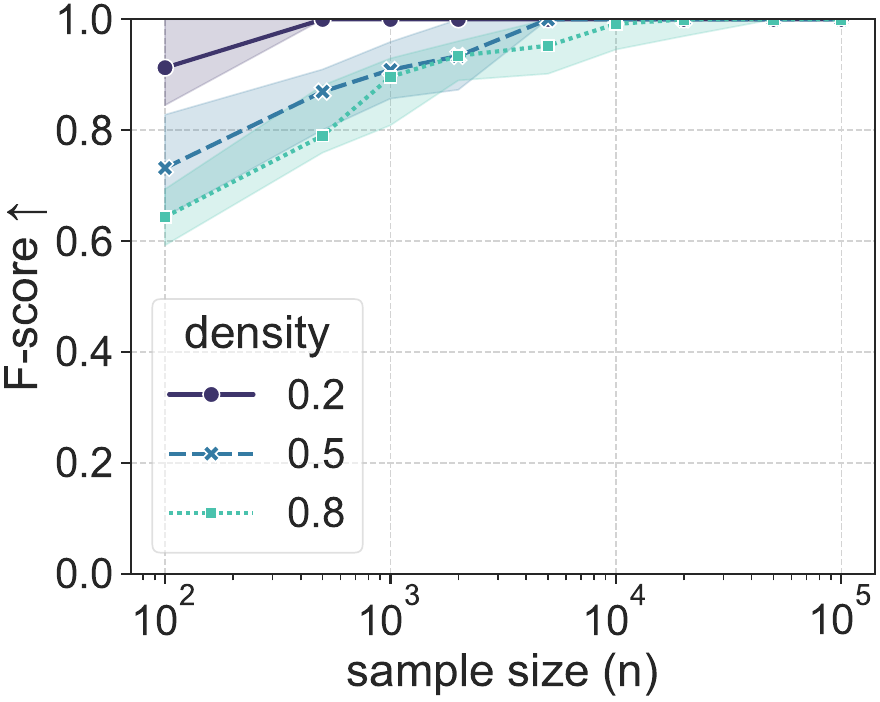}\\
    {\small (f) F-score, $m = 8$}
  \end{minipage}
  \caption{
    Full \coarse pipeline on ER graphs, $d = 10$; one curve per density $\rho$.
    Denser graphs and more environments delay convergence, but every setting converges.
  }
  \label{fig:density-er}
\end{figure*}

\begin{figure*}[t]
  \centering
  \begin{minipage}{0.32\textwidth}
    \centering
    \includegraphics[width=\linewidth]{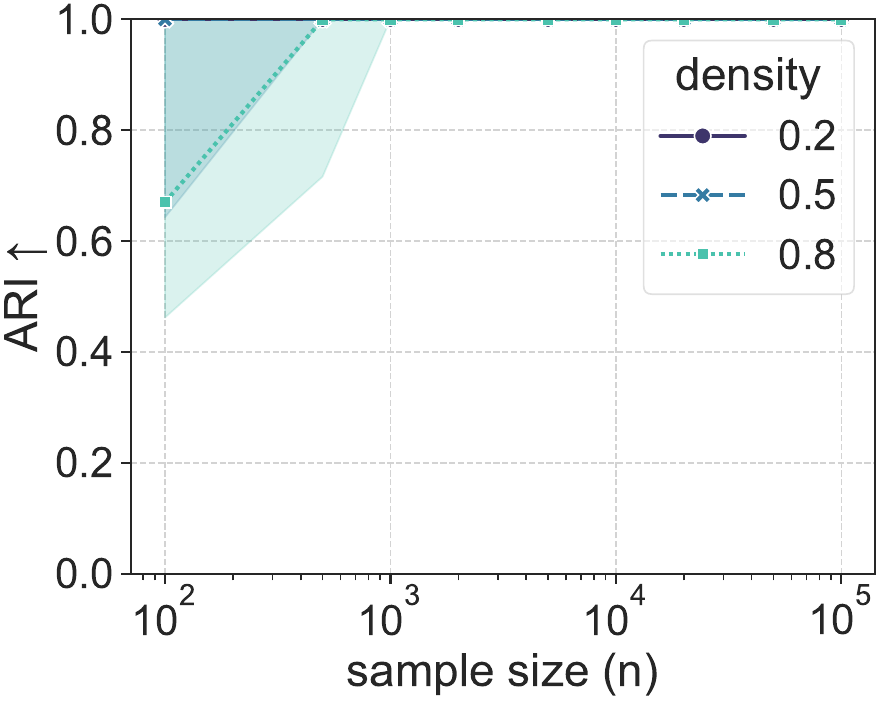}\\
    {\small (a) ARI, $m = 2$}
  \end{minipage}
  \hfill
  \begin{minipage}{0.32\textwidth}
    \centering
    \includegraphics[width=\linewidth]{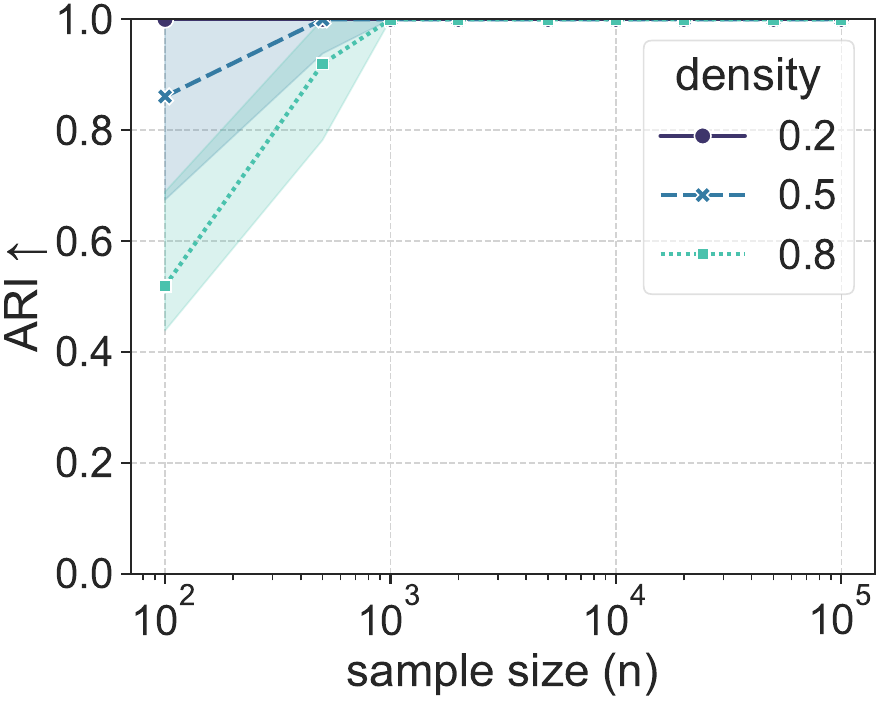}\\
    {\small (b) ARI, $m = 5$}
  \end{minipage}
  \hfill
  \begin{minipage}{0.32\textwidth}
    \centering
    \includegraphics[width=\linewidth]{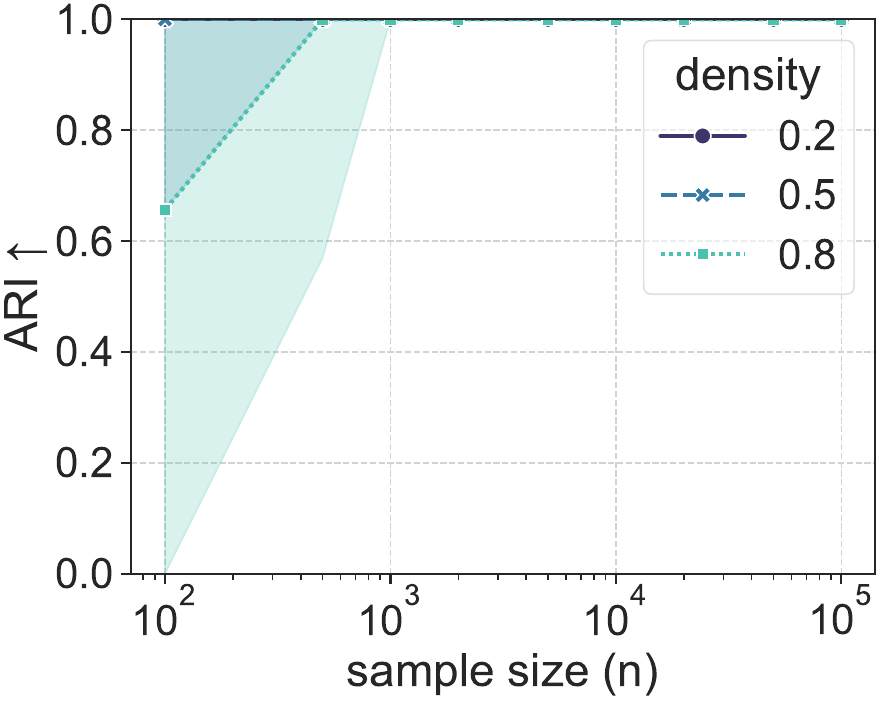}\\
    {\small (c) ARI, $m = 8$}
  \end{minipage}

  \vspace{4pt}
  \begin{minipage}{0.32\textwidth}
    \centering
    \includegraphics[width=\linewidth]{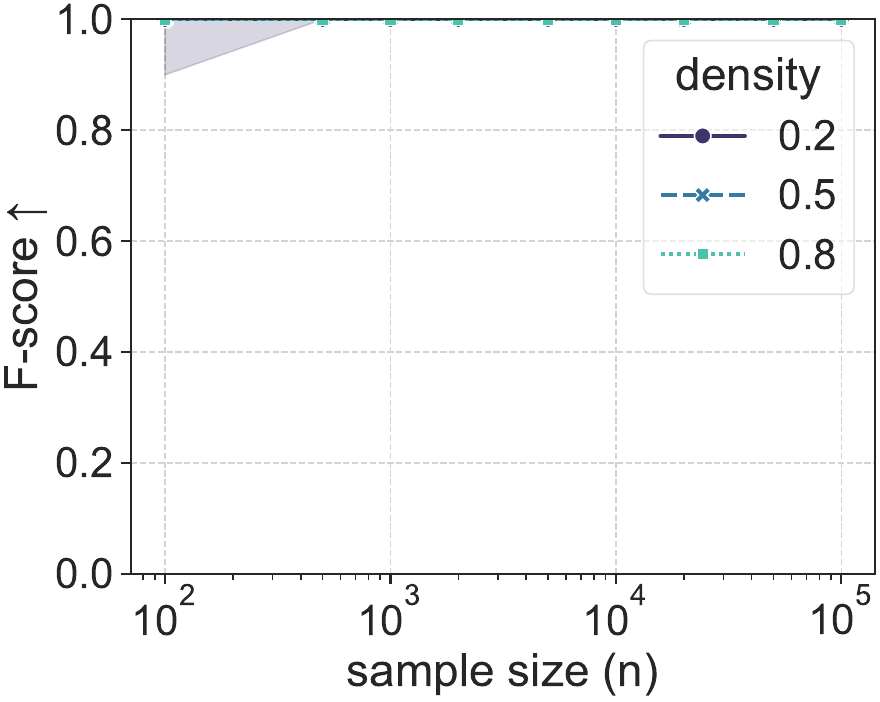}\\
    {\small (d) F-score, $m = 2$}
  \end{minipage}
  \hfill
  \begin{minipage}{0.32\textwidth}
    \centering
    \includegraphics[width=\linewidth]{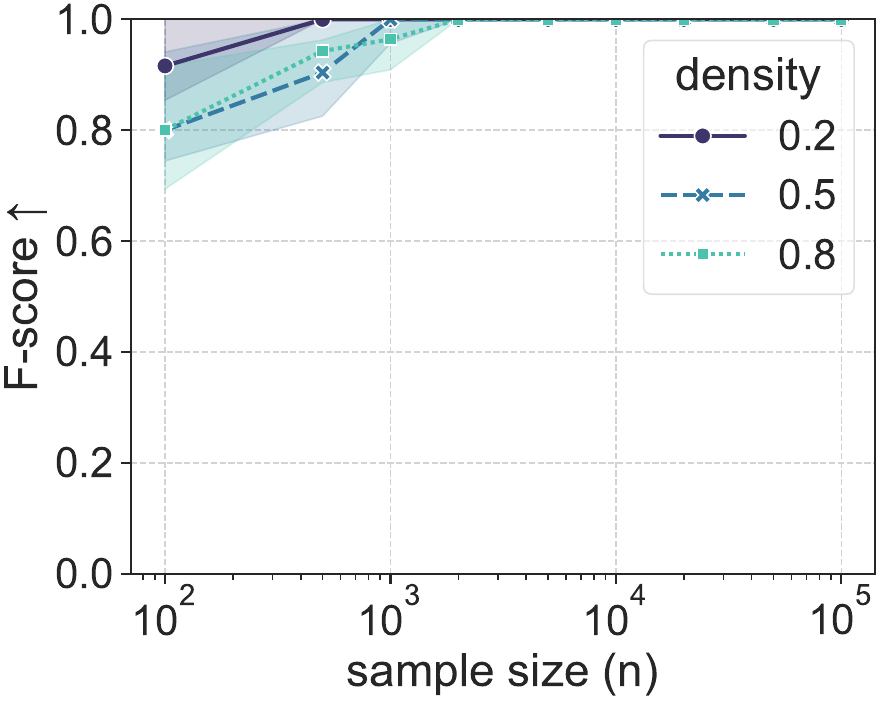}\\
    {\small (e) F-score, $m = 5$}
  \end{minipage}
  \hfill
  \begin{minipage}{0.32\textwidth}
    \centering
    \includegraphics[width=\linewidth]{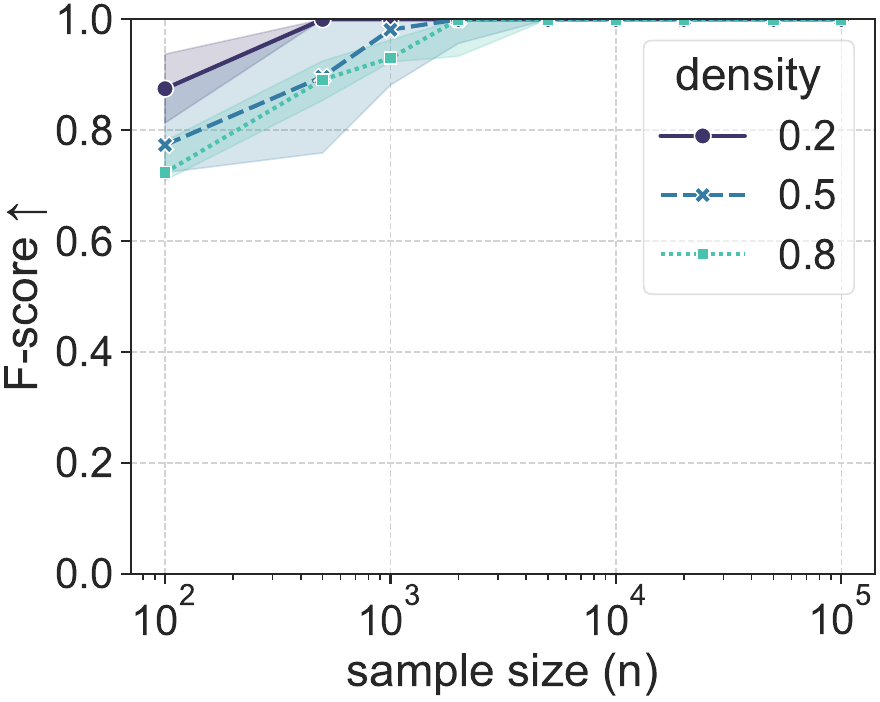}\\
    {\small (f) F-score, $m = 8$}
  \end{minipage}
  \caption{As \autoref{fig:density-er}, on scale-free graphs.}
  \label{fig:density-sf}
\end{figure*}

\subsection{Comparison against RePaRe: full grid}
\label{app:exp-compare}

\autoref{tab:compare-full} completes \autoref{tab:compare} with every sample size of the grid.

\begin{table}[h]
  \centering
  \caption{
    Edge recovery under the oracle partition, ER graphs, $\rho = 0.2$, $m = 5$, all sample sizes.
    F-score as mean $\pm$ sample standard deviation over 20 seeds; speed-up is the ratio of mean RePaRe to mean \coarse runtime.
    Bold marks the better method where the two differ at the reported precision.
  }
  \label{tab:compare-full}
  \resizebox{\textwidth}{!}{
    \begin{tabular}{lccccccc}
      \toprule
      $d$                   &       10        &       20        &       50        &       100       &       200       &       300       &       500       \\
      \midrule
      \multicolumn{8}{l}{\textbf{F-score $\uparrow$, \coarse}}                                                                                            \\[2pt]
      $n_e = 10^3$          & 1.00 $\pm$ 0.00 & \textbf{1.00 $\pm$ 0.02} & \textbf{0.99 $\pm$ 0.02} & \textbf{0.98 $\pm$ 0.02} & \textbf{0.95 $\pm$ 0.05} & 0.84 $\pm$ 0.08 & 0.71 $\pm$ 0.05 \\
      $n_e = 5 \times 10^3$ & 1.00 $\pm$ 0.00 & 1.00 $\pm$ 0.00 & \textbf{1.00 $\pm$ 0.01} & \textbf{1.00 $\pm$ 0.00} & \textbf{1.00 $\pm$ 0.01} & \textbf{1.00 $\pm$ 0.01} & \textbf{1.00 $\pm$ 0.01} \\
      $n_e = 10^4$          & 1.00 $\pm$ 0.00 & 1.00 $\pm$ 0.00 & 1.00 $\pm$ 0.00 & \textbf{1.00 $\pm$ 0.00} & \textbf{1.00 $\pm$ 0.01} & \textbf{1.00 $\pm$ 0.00} & \textbf{1.00 $\pm$ 0.01} \\
      $n_e = 5 \times 10^4$ & 1.00 $\pm$ 0.00 & 1.00 $\pm$ 0.00 & 1.00 $\pm$ 0.00 & 1.00 $\pm$ 0.00 & \textbf{1.00 $\pm$ 0.00} & 1.00 $\pm$ 0.00 & 1.00 $\pm$ 0.00 \\
      $n_e = 10^5$          & 1.00 $\pm$ 0.00 & 1.00 $\pm$ 0.00 & 1.00 $\pm$ 0.00 & 1.00 $\pm$ 0.00 & 1.00 $\pm$ 0.00 & 1.00 $\pm$ 0.00 & 1.00 $\pm$ 0.00 \\
      \midrule
      \multicolumn{8}{l}{\textbf{F-score $\uparrow$, RePaRe}}                                                                                             \\[2pt]
      $n_e = 10^3$          & 1.00 $\pm$ 0.01 & 0.96 $\pm$ 0.03 & 0.95 $\pm$ 0.05 & 0.93 $\pm$ 0.07 & 0.95 $\pm$ 0.08 & \textbf{0.95 $\pm$ 0.10} & \textbf{0.97 $\pm$ 0.04} \\
      $n_e = 5 \times 10^3$ & 1.00 $\pm$ 0.00 & 1.00 $\pm$ 0.01 & 0.99 $\pm$ 0.02 & 0.97 $\pm$ 0.04 & 0.97 $\pm$ 0.05 & 0.97 $\pm$ 0.07 & 0.99 $\pm$ 0.03 \\
      $n_e = 10^4$          & 1.00 $\pm$ 0.00 & 1.00 $\pm$ 0.00 & 1.00 $\pm$ 0.01 & 0.98 $\pm$ 0.03 & 0.99 $\pm$ 0.04 & 0.98 $\pm$ 0.06 & 0.99 $\pm$ 0.02 \\
      $n_e = 5 \times 10^4$ & 1.00 $\pm$ 0.00 & 1.00 $\pm$ 0.00 & 1.00 $\pm$ 0.00 & 1.00 $\pm$ 0.00 & 0.99 $\pm$ 0.02 & 1.00 $\pm$ 0.01 & 1.00 $\pm$ 0.01 \\
      $n_e = 10^5$          & 1.00 $\pm$ 0.00 & 1.00 $\pm$ 0.00 & 1.00 $\pm$ 0.00 & 1.00 $\pm$ 0.00 & 1.00 $\pm$ 0.00 & 0.99 $\pm$ 0.01 & 1.00 $\pm$ 0.01 \\
      \midrule
      \multicolumn{8}{l}{\textbf{Speed-up $\uparrow$ (RePaRe runtime / \coarse runtime)}}                                                                 \\[2pt]
      $n_e = 10^3$          &    9$\times$    &   21$\times$    &   38$\times$    &   44$\times$    &   37$\times$    &   37$\times$    &   32$\times$    \\
      $n_e = 5 \times 10^3$ &   25$\times$    &   43$\times$    &   84$\times$    &   108$\times$   &   101$\times$   &   94$\times$    &   63$\times$    \\
      $n_e = 10^4$          &   36$\times$    &   66$\times$    &   136$\times$   &   146$\times$   &   169$\times$   &   128$\times$   &   101$\times$   \\
      $n_e = 5 \times 10^4$ &   54$\times$    &   104$\times$   &   282$\times$   &   310$\times$   &   372$\times$   &   337$\times$   &   221$\times$   \\
      $n_e = 10^5$          &   51$\times$    &   172$\times$   &   342$\times$   &   510$\times$   &   509$\times$   &   424$\times$   &   329$\times$   \\
      \bottomrule
    \end{tabular}
  }
\end{table}

\subsection{Penalty multiplier}
\label{app:exp-lambda}

\autoref{fig:lambda} reports an ablation on the penalty multiplier $\lambda$.
We run the edge stage of \coarse under the oracle partition on ER graphs with $d = 50$, $m = 5$, $\rho \in \{0.2, 0.5\}$ and $n_e \in \{10^2, 5 \times 10^2, 10^3, 5 \times 10^3\}$, for $\lambda \in \{0.25, 0.5, 1, 2, 4\}$.
Precision is at least $0.93$ in every cell and at least $0.99$ for $\lambda \ge 0.5$, so $\lambda$ mainly affects recall.
At $n_e = 10^2$, the standard BIC ($\lambda = 1$) already misses some edges (F-score $0.84$ and $0.87$), $\lambda = 2$ misses between a third and a half of them, and $\lambda = 4$ misses most.
From $n_e = 5 \times 10^3$ on, every $\lambda$ reaches an F-score of at least $0.99$.
Consistently with \autoref{prop:localcons}, the choice of $\lambda$ matters only at small $n_e$.
There the score is still far from its asymptotic regime, the number of parameters is large relative to the sample size, and the penalty can outweigh the likelihood gain of true edges.
In these cases a lighter penalty is the better choice.
Across the regimes we tested, $\lambda = 1$ is a good default for \coarse.

\begin{figure}[t]
  \centering
  \begin{minipage}{0.4\textwidth}
    \centering
    \includegraphics[width=\linewidth]{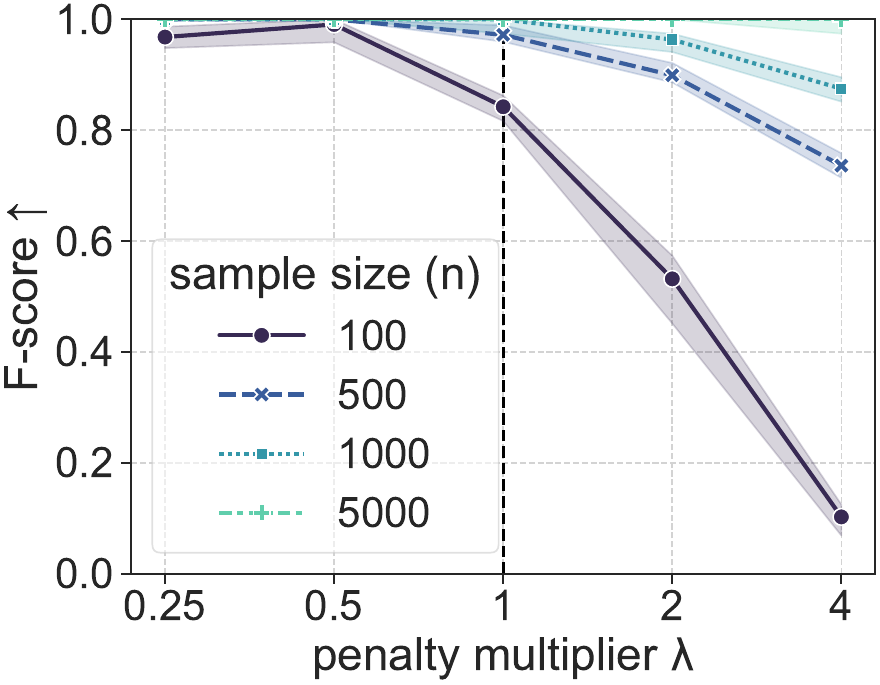}\\
    {\small (a) $\rho = 0.2$}
  \end{minipage}
  \hspace{0.04\textwidth}
  \begin{minipage}{0.4\textwidth}
    \centering
    \includegraphics[width=\linewidth]{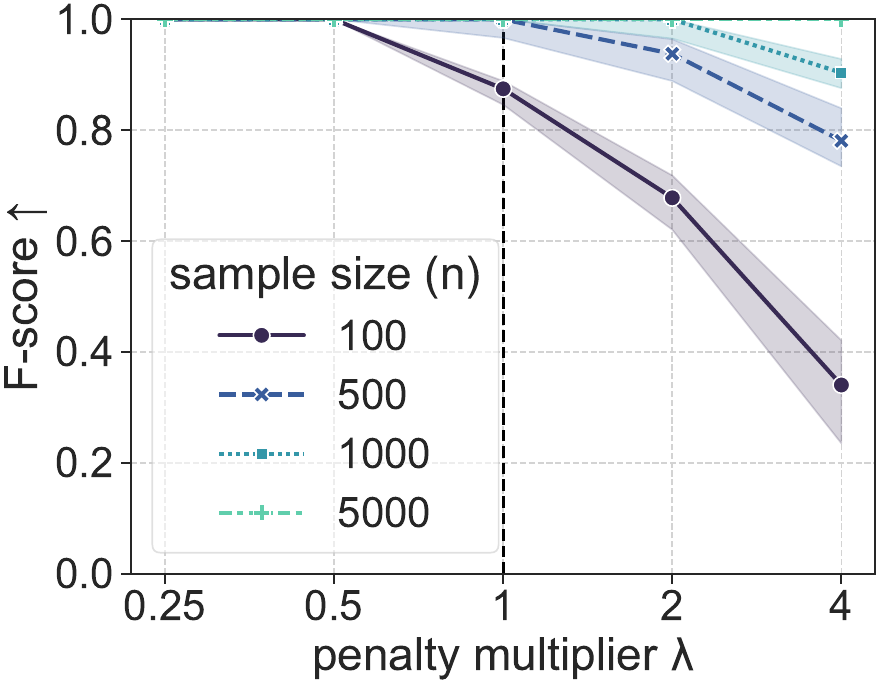}\\
    {\small (b) $\rho = 0.5$}
  \end{minipage}
  \caption{
    Edge F-score under the oracle partition against the BIC penalty multiplier, ER graphs, $d = 50$, $m = 5$; one curve per $n_e$, dashed line at the standard BIC.
    Precision stays above $0.99$ for $\lambda \ge 0.5$, so the loss at large $\lambda$ and small $n_e$ is recall.
  }
  \label{fig:lambda}
\end{figure}

\subsection{Cutoff rules for the descendant test}
\label{app:exp-threshold}

In order to build $\hat M$, \autoref{alg:refineaux} thresholds $d \cdot m$ Welch $p$-values at a fixed level $\alpha$.
\autoref{fig:threshold} compares this choice with rules that set the cutoff from the data, on one shared $p$-value matrix per dataset.
We tested several well known statistical correction at $\delta = 0.05$:

\begin{itemize}
  \item Bonferroni \citep{dunn_multiple_1961}: uses $\alpha = \delta/(dm)$ and control the probability of at least one false shift under any dependence between tests.
  \item Benjamini--Hochberg (BH) \citep{benjamini_controlling_1995}: uses $\alpha = k^\star\delta/(dm)$, with $k^\star := \max\{i : p_{(i)} \le i\delta/(dm)\}$ and bounds the false discovery rate by $\delta$, provided the $p$-values are independent or positively dependent.
  \item Benjamini--Yekutieli (BY) \citep{benjamini_control_2001}: which extends BH by keeping the false discovery rate below $\delta$ under any dependence.
\end{itemize}

And compare them against both our implementation of CV-\coarse over the $\alpha$ grid of \autoref{app:exp-details} and the oracle $\alpha$ (picked as the grid value with the highest ARI on that dataset).
For the study we used ER graphs, $d = 250$, $m = 5$, $\rho \in \{0.2, 0.5, 0.8\}$, $n_e \in \{10^4, 10^5\}$, 20 seeds and scored only the partition recovery stage.

The ARI of a fixed cutoff rises slowly from $\alpha = 10^{-8}$ to a maximum near $\alpha = 10^{-2}$ and collapses beyond it, where false shifts split true blocks.
Bonferroni sits at $\alpha \approx 4 \times 10^{-5}$ and loses the most (median gap to the oracle $-0.10$ to $-0.15$ at $n_e = 10^4$); BH and BY land in the flat region and lose at most $0.07$.
CV-\coarse, the only rule that needs neither a level nor the truth, is within $0.04$ of the oracle median at $n_e = 10^4$ and within $0.02$ at $n_e = 10^5$ (\autoref{fig:threshold-gap}).
The fixed $\alpha = 10^{-4}$ of \autoref{sec:exp-setup} is conservative at this $d$: $0.43$ against $0.55$ for the oracle at $\rho = 0.2$, $n_e = 10^4$, and $0.72$ against $0.79$ at $n_e = 10^5$.
This is the loss \autoref{sec:exp-scal} attributes to the partition stage, and CV recovers most of it.

\begin{figure*}[t]
  \centering
  \includegraphics[width=\textwidth]{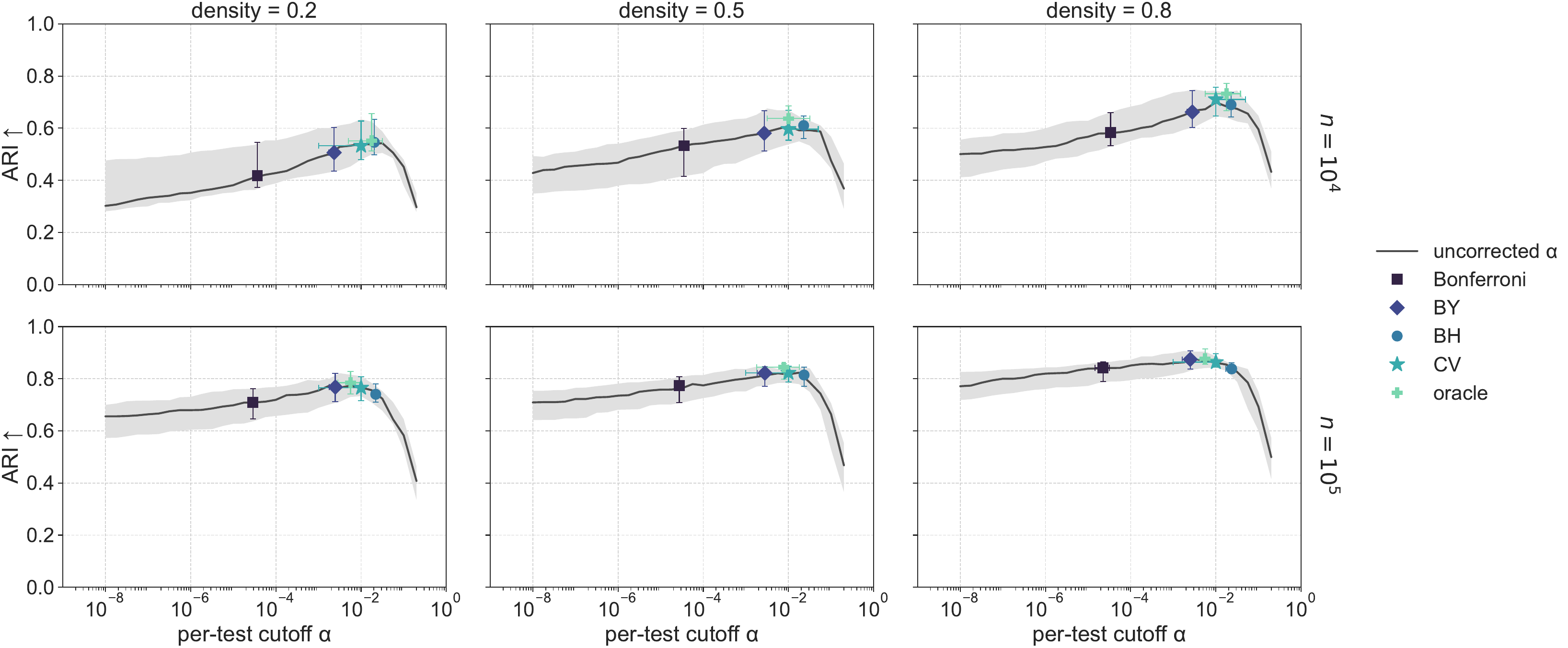}
  \caption{
    Partition ARI against the per-test cutoff, ER graphs, $d = 250$, $m = 5$.
    Grey: fixed $\alpha$ (median and interquartile range); markers: cutoff and ARI reached by each rule (medians with interquartile ranges).
    Every data-driven rule but Bonferroni lands in the flat region of the curve; CV reaches the oracle without ground truth.
  }
  \label{fig:threshold}
\end{figure*}

\begin{figure*}[t]
  \centering
  \includegraphics[width=0.85\textwidth]{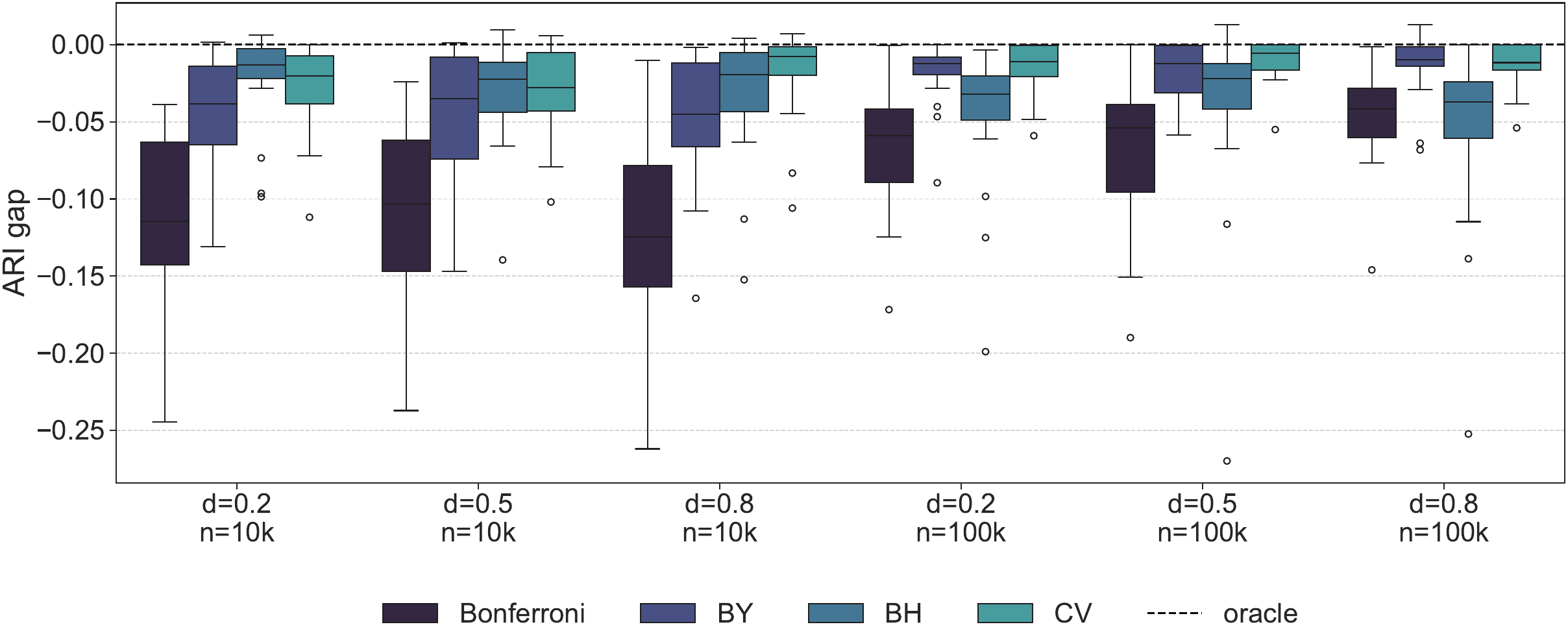}
  \caption{
    ARI of each rule minus the oracle ARI on the same dataset, over 20 seeds. $d$ stands for density. 
  }
  \label{fig:threshold-gap}
\end{figure*}

\subsection{Multi-target interventions}
\label{app:exp-multitarget}

\autoref{sec:target} defines $\Pi^\cE$ through signatures and shows that it is coarser than $\Pi^\cI$ \citep{madaleno_coarsening_2026-1} under multi-target interventions (\autoref{app:multi-target}).
We quantify the gap at the population level and check that \coarse estimates $\Pi^\cE$, not $\Pi^\cI$.

\paragraph{Population level.}

On ER graphs with $d = 50$ and $\rho \in \{0.2, 0.5, 0.8\}$ (200 seeds) we compute both partitions from the graph and the targets and report their ARI, the block ratio $|\Pi^\cE| / |\Pi^\cI|$ and the \emph{hidden-edge fraction}, the share of true coarse edges of $G^\cI$ that fall inside one block of $\Pi^\cE$.
In the \emph{native} design (\autoref{fig:mt-native}) each of $m$ environments draws $k$ targets.
In the \emph{overlap} design (\autoref{fig:mt-overlap}) the targets of all environments come from a fixed pool of $20$ nodes: a cover of the pool by disjoint $k$-subsets plus random $k$-subsets, up to $m$ environments.

In the native design the resolution falls quickly with $k$: at $k = 2$ the ARI is $0.38$--$0.75$ and at $k = 5$ below $0.26$, with $\Pi^\cE$ holding a quarter to a half of the blocks of $\Pi^\cI$.
More environments lower the ARI further, since every new multi-target environment refines $\Pi^\cI$ more than $\Pi^\cE$, but they lower the hidden-edge fraction, from $0.40$--$0.61$ at $m = 2$ to $0.18$--$0.33$ at $m = 20$ for $k = 5$.
The overlap design shows the other side: with a fixed target pool, adding overlapping environments recovers resolution, from ARI $0.56$ at $m = 10$ to $0.86$ at $m = 40$ for $k = 2$ ($\rho = 0.2$), with the hidden-edge fraction dropping towards zero.
What the interventions can resolve is set by how their target sets intersect, not by how many there are.

\paragraph{Estimation.}

\autoref{fig:mt-fit} runs \coarse with cross-validated $\alpha$ on ER graphs with $d = 50$, $m = 10$, $k \in \{1, \dots, 5\}$ targets per environment and $n_e \in \{10^3, 10^4, 10^5\}$ (20 seeds), scoring $\hat\Pi$ against both $\Pi^\cE$ and $\Pi^\cI$.
$\mathrm{ARI}(\hat\Pi, \Pi^\cE)$ is (almost) flat in $k$ and increases with $n_e$, reaching $0.85$--$0.92$ at $n_e = 10^5$, whereas $\mathrm{ARI}(\hat\Pi, \Pi^\cI)$ decreases with $k$ and, for $n_e \ge 10^4$, follows the population curve $\mathrm{ARI}(\Pi^\cE, \Pi^\cI)$.
Hence the estimate converges to $\Pi^\cE$, the object \autoref{def:icoarsening} identifies, and its distance to $\Pi^\cI$ reflects the population gap rather than estimation error.

\begin{figure*}[t]
  \centering
  \includegraphics[width=\textwidth]{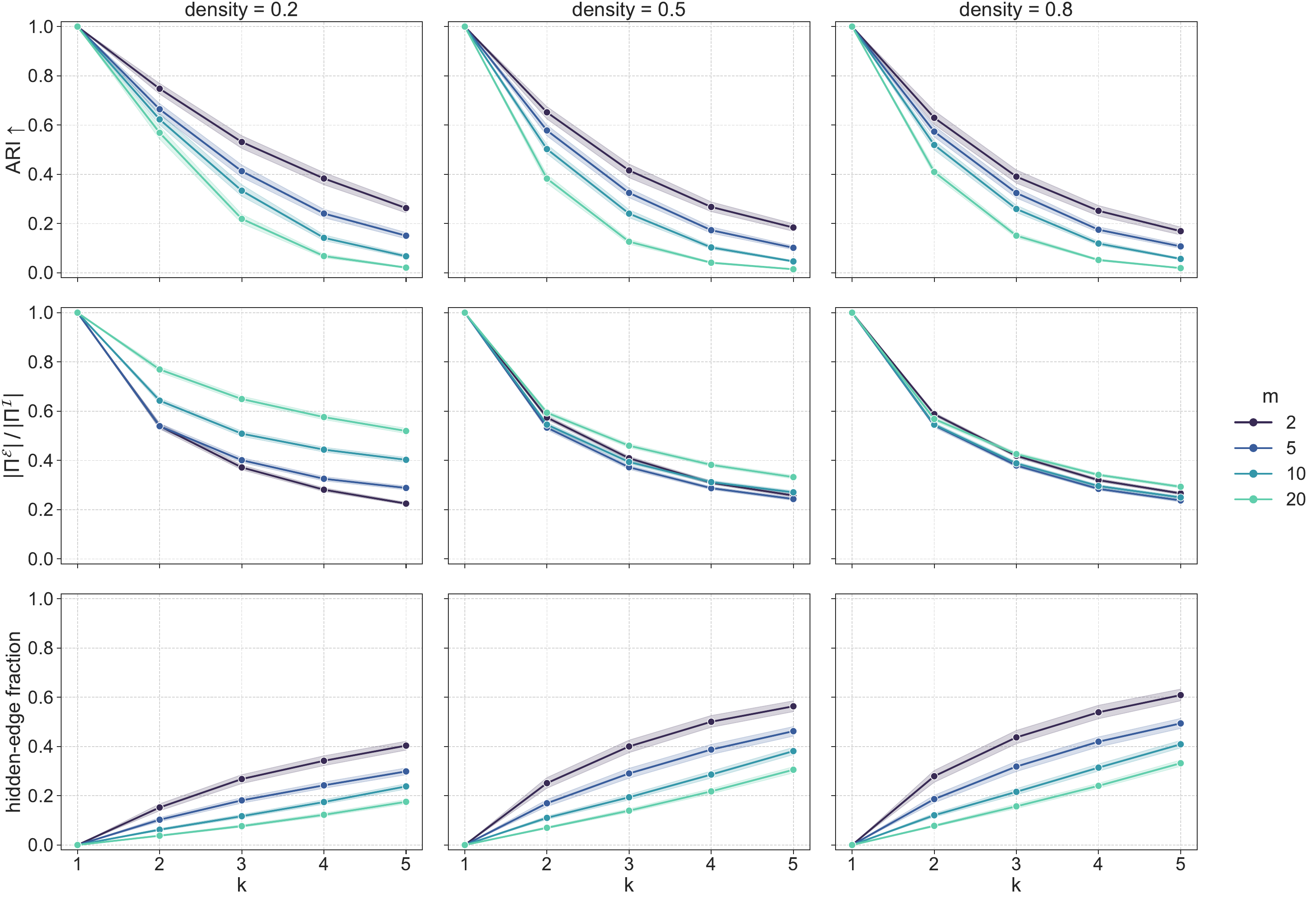}
  \caption{
    $\Pi^\cE$ against $\Pi^\cI$ at the population level, native design, ER graphs, $d = 50$; $k$ targets per environment, one curve per number of environments $m$ (mean and 95\% CI over 200 seeds).
    Resolution is lost as soon as $k > 1$; more environments refine $\Pi^\cI$ faster than $\Pi^\cE$ but hide fewer edges.
  }
  \label{fig:mt-native}
\end{figure*}

\begin{figure*}[t]
  \centering
  \includegraphics[width=\textwidth]{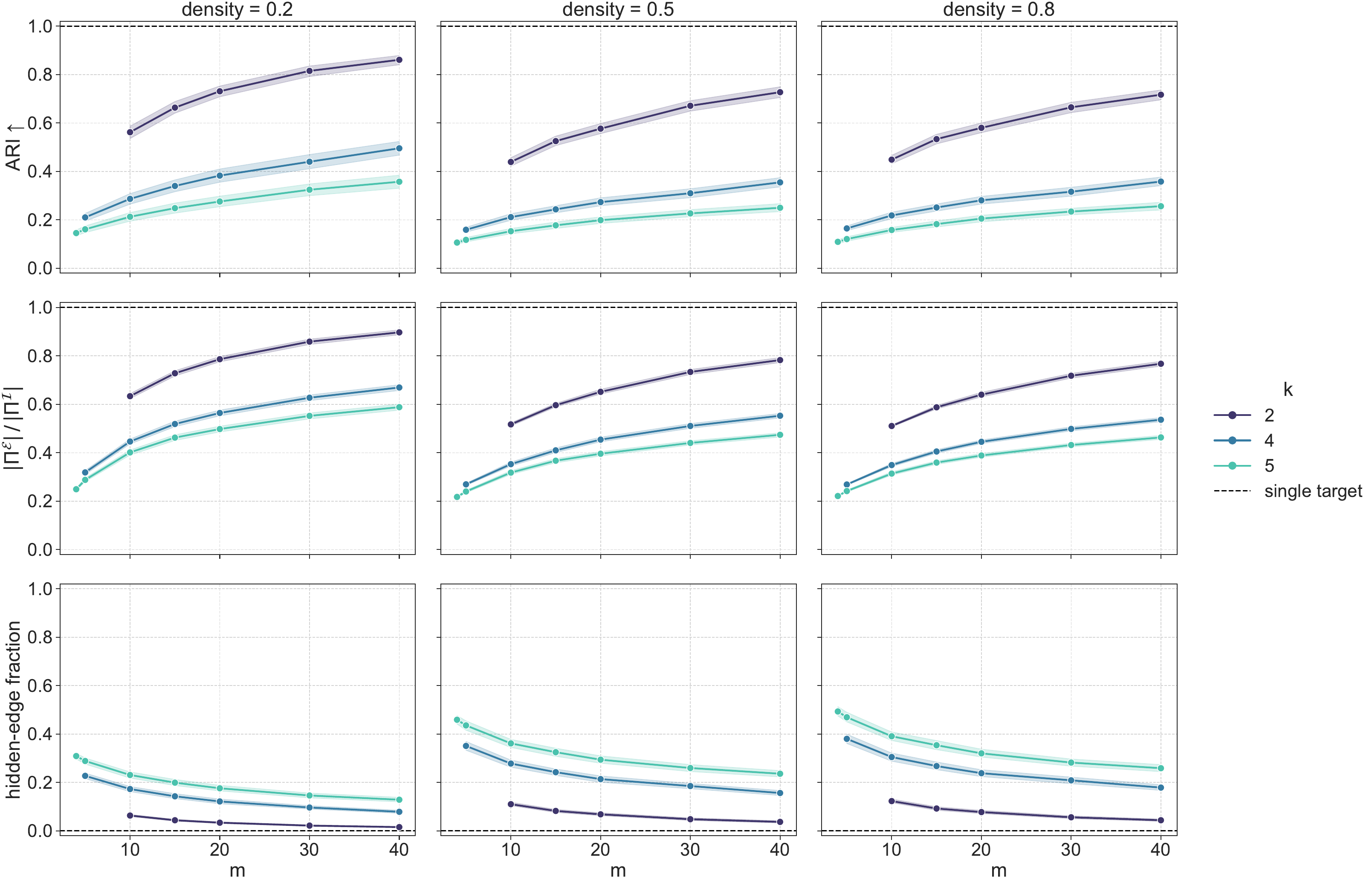}
  \caption{
    As \autoref{fig:mt-native}, overlap design: all targets come from a pool of $20$ nodes and $m$ grows; one curve per $k$, dashed line at the single-target value.
    Overlapping environments recover the resolution lost to multi-target interventions.
  }
  \label{fig:mt-overlap}
\end{figure*}

\begin{figure*}[t]
  \centering
  \includegraphics[width=\textwidth]{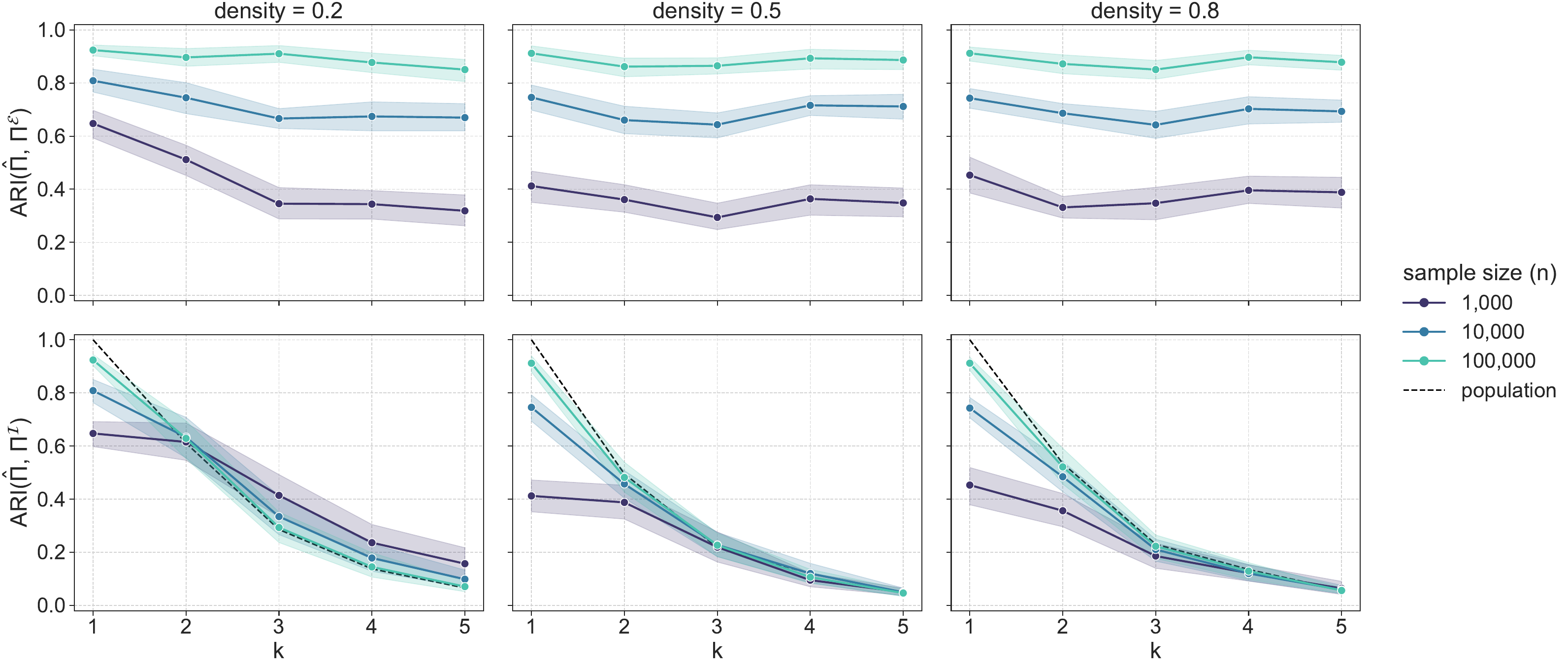}
  \caption{
    Partition estimated by \coarse with cross-validated $\alpha$ against $\Pi^\cE$ (top) and $\Pi^\cI$ (bottom), ER graphs, $d = 50$, $m = 10$; one curve per $n_e$, dashed: population $\mathrm{ARI}(\Pi^\cE, \Pi^\cI)$.
    The estimate tracks $\Pi^\cE$ at every $k$; its distance to $\Pi^\cI$ is the population gap.
  }
  \label{fig:mt-fit}
\end{figure*}

\subsection{High-dimensional regime: reduced-rank variants}
\label{app:exp-highdim}

The block MLE of \autoref{def:blockmodel} exists only when $n_e$ exceeds the family size $|F_j| = r_j + s_j$ (\autoref{app:block}), which fails when blocks are large and samples few.
Two variants reduce the dimension of every block before scoring; both change only the score stage of \autoref{alg:coarse}.
After centering, every environment is divided by the per-variable standard deviation of the observational environment, so that the reduction is the same linear map in every environment.

\emph{kPC-\coarse} replaces each block $X^{(e)}_j \in \R^{r_j}$ by $Z^{(e)}_j := V_j^\top X^{(e)}_j \in \R^{k_j}$, where $V_j \in \R^{r_j \times k_j}$ holds the top $k_j := \min\{k, r_j\}$ right singular vectors of the observational data matrix of block $\pi_j$, that is, its first $k_j$ principal directions, estimated once in the observational environment and applied unchanged to every other one.
\emph{avg-\coarse} is the rank-one case $V_j := \1 / r_j$, the mean of the scaled variables of the block.
The pooled BIC of \autoref{def:bic} and the grow-shrink search are then run on the reduced blocks, with $k_j$ in place of $r_j$.
Consequently, the family size of block $j$ drops from $r_j + s_j$ to $k_j + \sum_{\pi_k \in \Pa_j} k_k$, so the reduced score exists for much smaller $n_e$.

We compare \coarse with 1PC-, 3PC- and avg-\coarse under the oracle partition on ER graphs with $d = 100$, $\rho = 0.1$, $m \in \{5, 10, 20\}$ and $n_e$ from $10$ to $10^3$ (20 seeds).

\autoref{fig:highdim} shows the edge recovery performance of each method as the sample size moves from the high-dimensional to the classical statistical regime.
\autoref{fig:highdim-block} explores a similar setting, but focuses on how many samples are needed to recover the correct parent set.
For every block with at least one parent, it reports whether the parent set is recovered exactly and the fraction of parents recovered, as a function of $n_e / F_j$, where $F_j$ is the number of variables in the block and its true parents.
\autoref{fig:highdim-feasible} reports the share of blocks whose family fits in $n_e$ samples, measured in each method's own units.

For $n_e \le 30$, few families of \coarse are feasible (degenerate fits score $-\infty$) and its F-score is only $0.17$--$0.44$, while the reduced variants already recover a large part of the edges (F-score $0.4$--$0.8$), with 3PC the best variant throughout.
\coarse overtakes 3PC at $n_e \approx 100$ for $m = 20$ and at $n_e \approx 500$ for $m = 5$, reaching an F-score of $0.96$--$0.98$ at $n_e = 10^3$ and precision $1$.
The variants instead plateau or decline, more so with more environments: 1PC/3PC reach $0.87$/$0.93$ at $m = 5$ but only $0.62$/$0.78$ at $m = 20$, and avg reaches $0.59$.
The decline is in precision ($0.46$--$0.66$ at $m = 20$), the discarded within-block directions carry part of the dependence between a block and its parents, so the reduced model is misspecified and, given enough data, admits spurious parents.
Reduction thus trades exactness at large $n_e$ for a score that exists when $n_e$ is below the family size.

Per block (\autoref{fig:highdim-block}), \coarse never recovers a parent set exactly while $n_e \le F_j$, because grow-shrink only admits parents whose family still fits in $n_e$ samples (the MLE does not exist).
Just below $n_e = F_j$, its recall is $0.31$--$0.45$, against $0.60$--$0.72$ for 3PC (exact recovery $0.09$--$0.44$).
Above $n_e = F_j$, \coarse becomes exact ($\ge 0.86$ from $n_e / F_j \approx 20$, ${\approx}\,1$ from ${\approx}\,45$), while 3PC keeps recall at $0.96$--$0.98$ but is exact in only $0.60$--$0.89$ of blocks ($0.31$--$0.94$ for 1PC and avg).
The gap comes from spurious parents, matching the precision loss in \autoref{fig:highdim}.

\begin{figure*}[t]
  \centering
  \includegraphics[width=\textwidth]{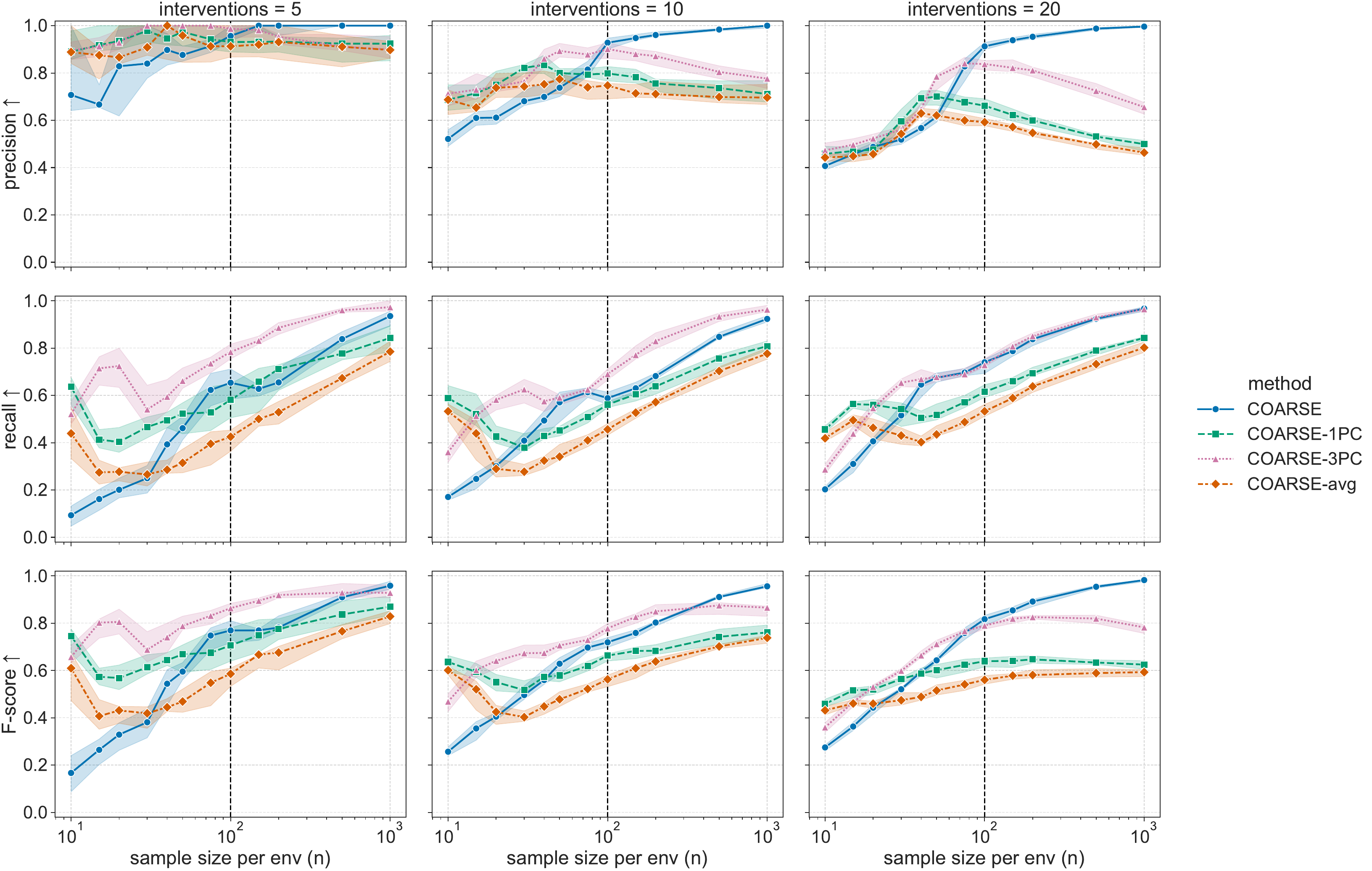}
  \caption{
    Edge recovery under the oracle partition for \coarse and its reduced-rank variants, ER graphs, $d = 100$, $\rho = 0.1$; dashed line at $n_e = d$.
    The variants score where \coarse cannot ($n_e$ below the family size) and lose precision where it can.
  }
  \label{fig:highdim}
\end{figure*}

\begin{figure*}[t]
  \centering
  \includegraphics[width=1\linewidth]{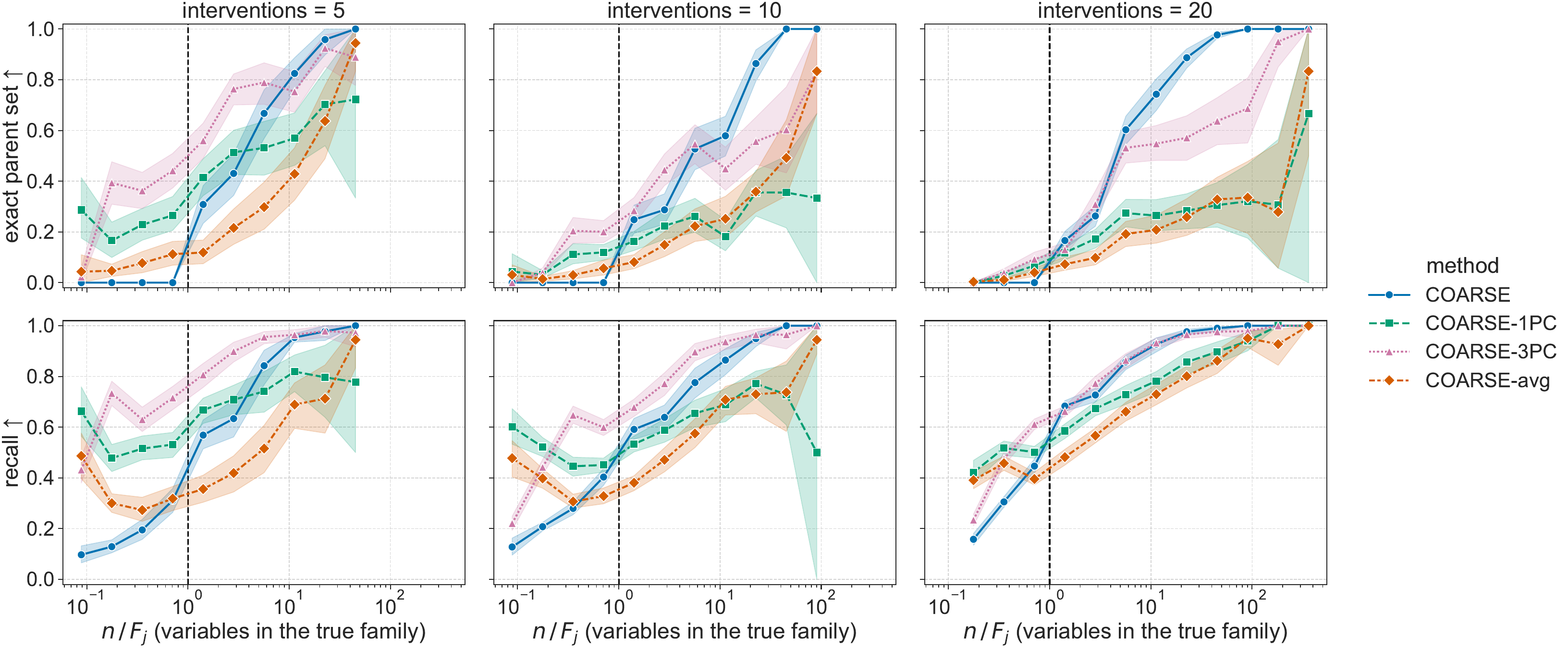}
  \caption{
    Per-block parent recovery under the oracle partition for \coarse and its reduced-rank variants, ER graphs, $d = 100$, $\rho = 0.1$: exact parent set (top) and parent recall (bottom) against $n_e / F_j$, with $F_j$ the number of variables in the true family of $\pi_j$ for every method; blocks binned on a $\log_2$ scale, dashed line at $n_e = F_j$.
    \coarse near 0 recovery below the line stems from the non-existence of the MLE; the variants recover part of the parents below it but add spurious ones above it.
  }

  \label{fig:highdim-block}
\end{figure*}

\begin{figure*}[t]
  \centering
  \includegraphics[width=\textwidth]{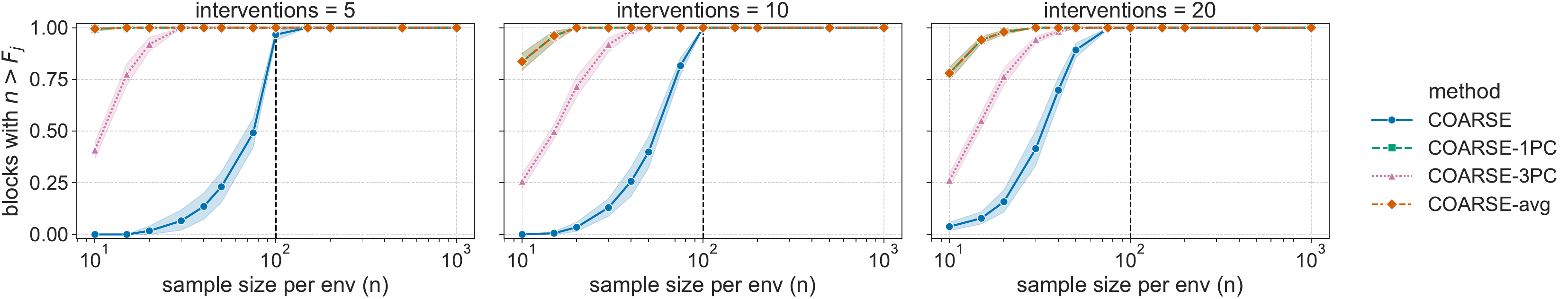}
  \caption{
    Share of blocks whose family size, in each method's own units, is below $n_e$, same setting as \autoref{fig:highdim}.
    The crossover of \autoref{fig:highdim} coincides with the point where every family of \coarse becomes feasible.
  }
  \label{fig:highdim-feasible}
\end{figure*}

\subsection{Noise distributions}
\label{app:exp-noise}

The score of \autoref{def:bic} is a Gaussian likelihood.
\autoref{fig:noise} runs the edge stage under the oracle partition on ER graphs with $\rho = 0.2$ and $m = 5$, for $d \in \{50, 100, 200, 500\}$ and five noise families with equal mean and variance: Gaussian, uniform, Laplace and Student-$t$ with $3$ and $4$ degrees of freedom.
The curves coincide: at every $(d, n_e)$ the mean F-scores of the five families differ by at most $0.01$, and all reach $1$ at $n_e = 5 \times 10^4$.
In a linear model, d-separation implies a vanishing partial covariance whatever the noise, and that is what the increments of \autoref{prop:localcons} detect allowing BIC to correctly rank the models even under misspecification.
The partition is fixed to the oracle here, so this does not test the Welch test under heavy tails, where a rank-based or distribution-free two-sample test could replace it.

\begin{figure*}[t]
  \centering
  \begin{minipage}{0.245\textwidth}
    \centering
    \includegraphics[width=\linewidth]{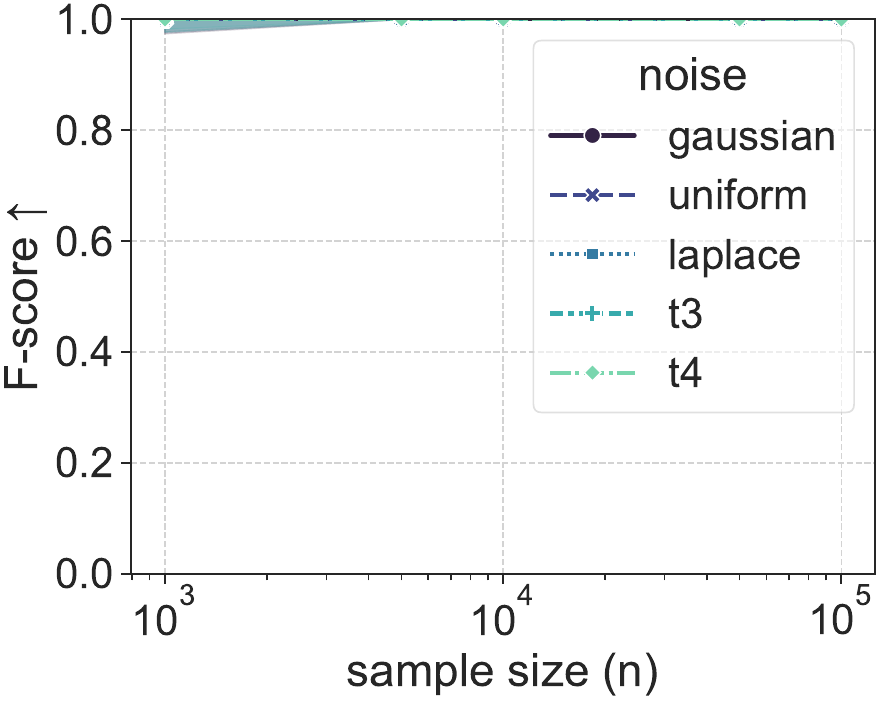}\\
    {\small (a) $d = 50$}
  \end{minipage}
  \hfill
  \begin{minipage}{0.245\textwidth}
    \centering
    \includegraphics[width=\linewidth]{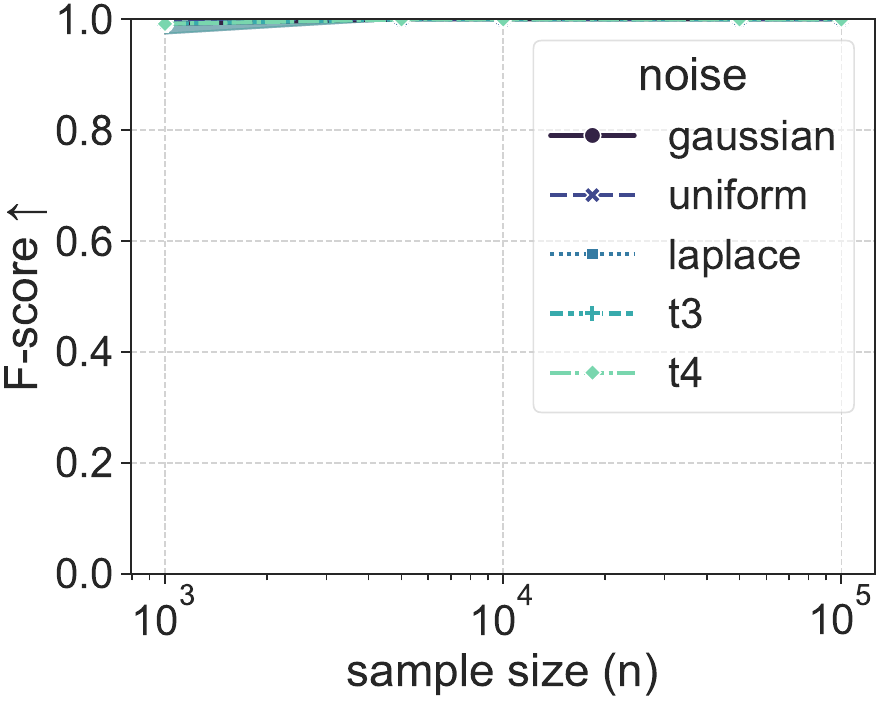}\\
    {\small (b) $d = 100$}
  \end{minipage}
  \hfill
  \begin{minipage}{0.245\textwidth}
    \centering
    \includegraphics[width=\linewidth]{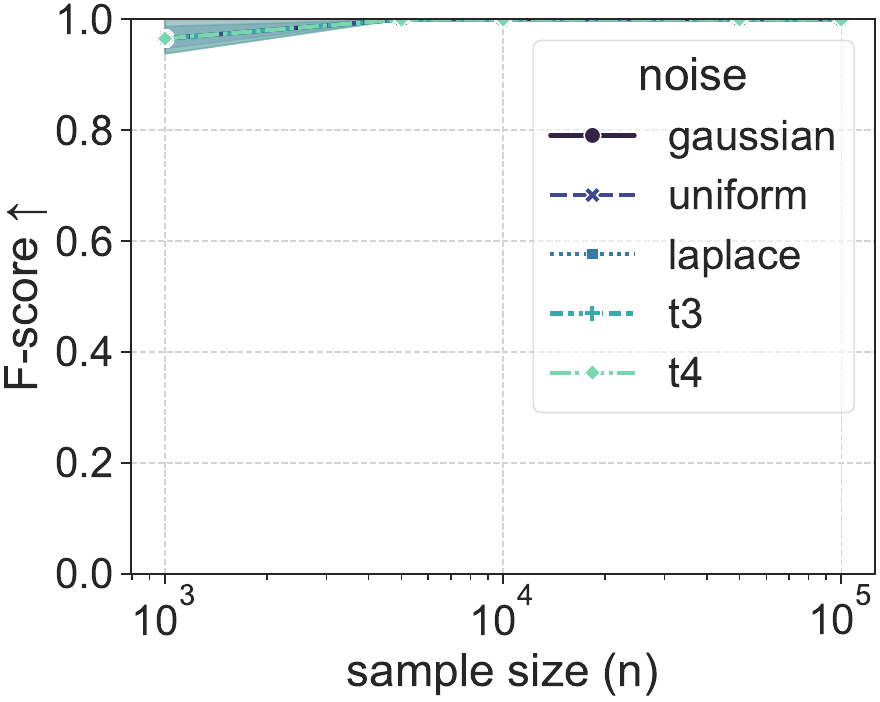}\\
    {\small (c) $d = 200$}
  \end{minipage}
  \hfill
  \begin{minipage}{0.245\textwidth}
    \centering
    \includegraphics[width=\linewidth]{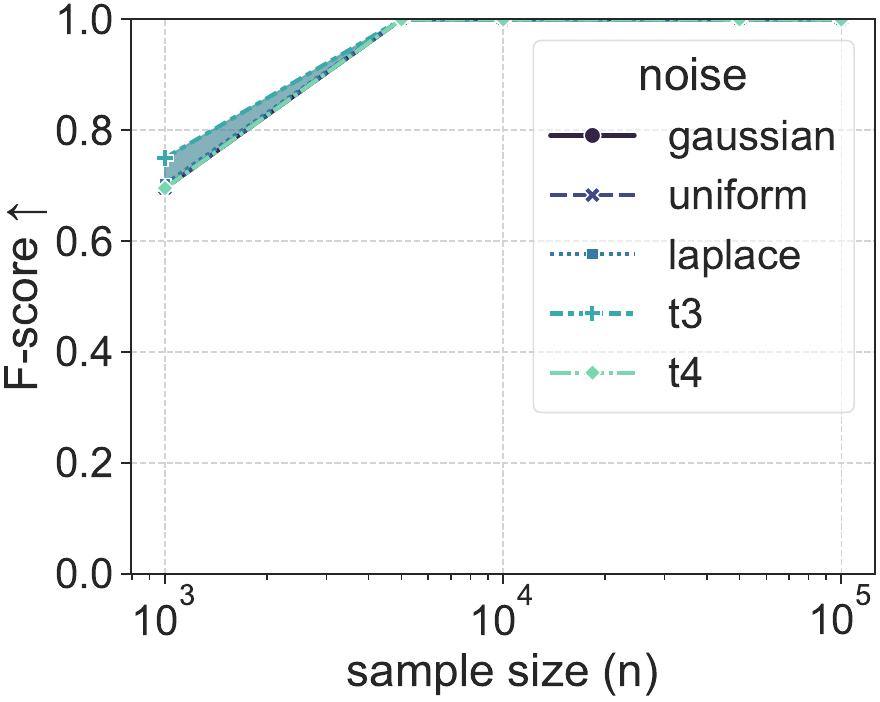}\\
    {\small (d) $d = 500$}
  \end{minipage}
  \caption{
    Edge F-score under the oracle partition for five noise families, ER graphs, $\rho = 0.2$, $m = 5$.
    The Gaussian score is insensitive to the noise family in the linear model.
  }
  \label{fig:noise}
\end{figure*}

\subsection{Causal chamber: full results}
\label{app:exp-chamber}

\autoref{tab:chamber-full} reports precision and recall, oracle selection and the grouped mode.
In the ungrouped mode the hyperparameters selected without ground truth coincide with the oracle ones for \coarse ($\alpha = 10^{-4}$) and RePaRe ($\alpha = 10^{-4}$, $\beta = 10^{-2}$); UT-IGSP gains $0.11$ in F-score under oracle selection.
In the grouped mode ($m = 2$) both methods select a looser partition level than the oracle ($\alpha = 10^{-2}$ for \coarse, $10^{-3}$ for RePaRe) and miss the oracle partition by one variable (ARI $0.93$); the edges are the same under both selections, $2$ of the $4$ true coarse edges with no false edge.
\autoref{tab:chamber-adj} complements \autoref{tab:chamber} with adjacency (undirected) metrics at the variable level and the exact hypergeometric null.
Every method recovers adjacencies above chance, including GnIES, whose directed F-score in \autoref{tab:chamber} is at chance: its errors are in the orientations.
\coarse and RePaRe have the highest precision, $0.86$, and tie with UT-IGSP for the highest recall.

\begin{table}[h]
  \centering
  \caption{
    Light tunnel, all methods and selections.
    \emph{Score} rows select hyperparameters without ground truth (\coarse: cross-validation; RePaRe: GnIES score; UT-IGSP: BIC), \emph{oracle} rows by maximum ARI, then F-score; GIES and GnIES have no hyperparameters.
    Metrics refer to the coarse DAG for \coarse and RePaRe and to the variable-level graph for the others; GnIES is scored on the skeleton. 
    Fit and Search are wall-clock seconds of the final fit and of the hyperparameter search.
  }
  \label{tab:chamber-full}

  \begin{tabular}{llcccccc}
    \toprule
    Method  & Sel.   & ARI $\uparrow$ & Prec. $\uparrow$ & Rec. $\uparrow$ &  F1 $\uparrow$  &   Fit(s)   & Search(s) \\
    \midrule
    \multicolumn{8}{l}{\textbf{Ungrouped} ($m = 5$)}                 \\[2pt]
    \coarse & score  & 1.00 & 1.00  & 0.80 & 0.89 &  0.02   &  0.40  \\
            & oracle & 1.00 & 1.00  & 0.80 & 0.89 &  0.01   &   --   \\
    RePaRe  & score  & 1.00 & 1.00  & 0.80 & 0.89 &  0.10   &  1.54  \\
            & oracle & 1.00 & 1.00  & 0.80 & 0.89 &  0.10   &   --   \\
    GIES    & --     &  --  & 0.63  & 0.62 & 0.62 &  1.83   &   --   \\
    GnIES   & --     &  --  & 0.41  & 0.56 & 0.47 &   376   &   --   \\
    UT-IGSP & score  &  --  & 0.38  & 0.54 & 0.45 &  0.19   &  2.69  \\
            & oracle &  --  & 0.53  & 0.59 & 0.56 &  0.13   &   --   \\
    \midrule
    \multicolumn{8}{l}{\textbf{Grouped} ($m = 2$)}                   \\[2pt]
    \coarse & score  & 0.93 & 1.00  & 0.50 & 0.67 & $<$0.01 &  0.12  \\
            & oracle & 1.00 & 1.00  & 0.50 & 0.67 & $<$0.01 &   --   \\
    RePaRe  & score  & 0.93 & 1.00  & 0.50 & 0.67 &  0.02   &  0.27  \\
            & oracle & 1.00 & 1.00  & 0.50 & 0.67 &  0.02   &   --   \\
    \bottomrule
  \end{tabular}

\end{table}

\begin{table}[h]
  \centering
  \small
  \caption{
    Light tunnel, ungrouped mode: adjacency (undirected) metrics against the exact negative control.
    Under the null hypothesis, the number of true adjacencies among the $\hat m$ estimated ones is hypergeometric.
    Null F1 is its mean and 95\% interval.
  }
  \label{tab:chamber-adj}
  \begin{tabular}{lccccr}
    \toprule
    Method  & Prec. $\uparrow$ & Rec. $\uparrow$ & F1 $\uparrow$ &      Null F1      & $\hat m$ \\
    \midrule
    \multicolumn{6}{l}{\textbf{Variable level}}                                                 \\[2pt]
    \coarse &  \textbf{0.86}   &  \textbf{0.64}  & \textbf{0.74} & 0.18 [0.06, 0.29] &       29 \\
    RePaRe  &  \textbf{0.86}   &  \textbf{0.64}  & \textbf{0.74} & 0.18 [0.06, 0.29] &       29 \\
    GIES    &       0.63       &      0.62       &     0.62      & 0.20 [0.10, 0.31] &       38 \\
    GnIES   &       0.41       &      0.56       &     0.47      & 0.24 [0.13, 0.34] &       54 \\
    UT-IGSP &       0.45       &  \textbf{0.64}  &     0.53      & 0.24 [0.13, 0.34] &       55 \\
    \midrule
    \multicolumn{6}{l}{\textbf{Block level}}                                                    \\[2pt]
    \coarse &  \textbf{1.00}   &      0.80       & \textbf{0.89} & 0.32 [0.11, 0.56] &        8 \\
    RePaRe  &  \textbf{1.00}   &      0.80       & \textbf{0.89} & 0.32 [0.11, 0.56] &        8 \\
    GIES    &       0.80       &      0.80       &     0.80      & 0.36 [0.10, 0.60] &       10 \\
    GnIES   &       0.50       &      0.80       &     0.62      & 0.44 [0.23, 0.62] &       16 \\
    UT-IGSP &       0.50       &  \textbf{0.90}  &     0.64      & 0.46 [0.29, 0.64] &       18 \\
    \bottomrule
  \end{tabular}
\end{table}

\end{document}